\documentclass[twoside,11pt]{article}

\usepackage{amsmath}
\usepackage{amssymb}
\usepackage{booktabs}
\usepackage{xltabular}
\usepackage{makecell}
\usepackage{multirow}
\usepackage{bm}
\usepackage{bbm}
\usepackage{fancyhdr}
\usepackage{subcaption}
\usepackage{placeins}
\usepackage{enumitem}
\usepackage{csquotes}
\usepackage{tikz}
\usepackage{pdflscape}
\usepackage[preprint]{jmlr2e}

\newcommand{\indep}{\perp \!\!\! \perp}
\newcommand{\pearldo}{\mathrm{do}}

\newcommand{\masdtw}{\textrm{MASD}_{\textrm{Weight-}t}(k,a_{0:t}',a_{0:t}'')}
\newcommand{\masdthw}{\textrm{MASD}_{\textrm{Hajek-}t}(k,a_{0:t}',a_{0:t}'')}

\newcommand{\masdtdiffmean}{\textrm{MASD-Diff}_{\textrm{Weight-}t}(k,t)}
\newcommand{\masdtwmean}{\textrm{MASD}_{\textrm{Weight-}t}(k,t)}

\usepackage{lastpage}
\jmlrheading{23}{2022}{1-\pageref{LastPage}}{1/21; Revised 5/22}{9/22}{21-0000}{Spear and Pope and Sebire}

\ShortHeadings{Evaluating covariate balance for long time horizon MDPs}{One and Two}
\firstpageno{1}

\begin{document}

\title{Evaluating covariate balance for long time horizon Markov decision processes}

\author{\name Joshua Spear \email joshua.spear.21@ucl.ac.uk \\
       \addr Post-Doctoral researcher \\
       Institute of Child Health\\
       University College London\\
       London, WC1E 6BT, UK
       \AND
       \name Rebecca Pope \email r.pope@ucl.ac.uk \\
       \addr Institute of Child Health\\
       University College London\\
       London, WC1E 6BT, UK
       \AND
       \name Neil J Sebire \email Neil.Sebire@gosh.nhs.uk \\ 
       \addr NIHR GOSH BRC \\
       Institute of Child Health\\
       University College London\\
       London, WC1E 6BT, UK
       }

\editor{}

\maketitle

\begin{abstract}
This article explores the application of covariate balance diagnostics for detecting the presence of hidden confounding/model miss-specification in studies applying offline reinforcement learning (RL) to deriving optimal treatment recommendations. The results demonstrate that, either there is a high risk of bias within existing offline RL studies for treatment recommendations and/or, existing covariate balance metrics are not sufficient to assess such studies. Regardless, existing offline RL studies cannot be concluded as being statistically robust. The conclusions propose future research directions for obtaining more methodologically robust applications of offline RL to treatment recommendation problems.
\end{abstract}


\section{Introduction}\label{sec:intro}

Offline reinforcement learning (RL) has become increasingly popular within the artificial intelligence (AI) community for various domains, including robotics (), healthcare () and power control. Some of these domains, a notable example being healthcare, assume the analysis is performed entirely offline i.e., there is no ability to obtain on-policy trajectories (outside of behaviour cloning). These domains rely on off-policy evaluation (OPE) to obtain quantitative measures of performance.\\

OPE is a causal inference problem in the sense that the assumptions are equivalent (\cite{} Barenboim work). Under OPE (or offline RL), a Markov decision process (MDP), or variant thereof, is assumed whilst, for causal inference analyses, the assumptions of positivity, consistency and exchangeability are required. Asserting the correctness of an MDP, or of consistency and exchangeability, is equally unverifiable. Within the context of causal inference, to account for this lack of verifiability, various diagnostic checks of robustness of causal inferences have been proposed (\cite{Rubin2001}, \cite{austin_balance_2009}). However, such checks are seldom seen in applications of OPE and offline RL in the AI community.\\

The primary aim of this analysis was to understand the extent to which existing applications of offline RL satisfy the diagnostic checks of assumptions that are common place within the causal inference community. The analysis focused on applying the covariate balance diagnostic to applications of offline RL for learning treatment regimens for sepsis. In conducting this analysis it was evident that there lacked a clear definition and motivation of the covariate balance metric within the literature. As such, this work formally motivates covariate balance along with its relationship to causal inference as a necessary, but not sufficient condition.\\

The contributions of this manuscript include:
\begin{itemize}
	\item An analysis of bias in offline RL studies from model miss-specification/confounding. The results suggest that either:
	\begin{itemize}
		\item Model miss-specification/confounding is present in such studies or;
		\item Existing diagnostic measures for assessing covariance balance are insufficient for evaluating long horizon tasks due to the presence of large variance in the balance estimators.
	\end{itemize}
	\item A proposal of new research directions away from policy learning research.
\end{itemize}

\section{Background}\label{sec:background}

The evaluation workflows of counterfactual analyses appearing in the AI and causal inference community are starkly different. Within AI, outcome based methods, particularly FQE (\cite{Le2019}) have recently fallen in favour. Large simulation studies such as those conducted by \cite{Voloshin2019} and \cite{Fu2021} have demonstrated the efficacy of the approach with respect to obtaining low mean squared error predictions of performance in comparison to other methods. In contrast, within causal inference, the benefit of propensity based methods, used in conjunction with outcome methods, has long been established (\cite{Rubin1973}).\\

Another noticeable difference is the lack of diagnostic evaluations of bias within the AI community and a seeming lack of awareness of robust statistical frameworks such as target trial emulation (TTE) (\cite{Hernan2020}). Within AI, sepsis management with IV fluids and antibiotics, and optimal mechanical extubation are popular research areas however, no studies addressing these domains have been cast within a TTE framework or considered diagnostic measures of balance (\cite{peine_development_2021}, \cite{Prasad2017}, \cite{Kondrup2023}, \cite{Huang2022}, \cite{Raghu2018b}, \cite{Kaushik2022}).\\

Within causal inference diagnostic checks for confounding/model miss-specification have predominantly revolved around covariate balance assessments using propensity scores (\cite{Rubin2001}, \cite{austin_balance_2009}). \cite{Rubin2001} in particular advocated for the use of propensity based methods since they do not require exposure to outcome data, emulating to some extent the process of a clinical trial. Described in section \ref{sec:covariate_balance}, there are different forms of covariate balance that have been used within the literature however, the intuition is that: if the set of covariates, and any balancing score (a function of the covariates, defined in section \ref{sec:covariate_balance}) are correctly specified, then the covariates distributions can be re-weighted to be independent of treatment assignment. If this independence is shown to exist, then a necessary condition for obtaining valid causal inferences is shown to hold.\\

A distinct diagnostic test for the correct specification of propensity scores is also common within the causal inference literature and in particular, in longitudinal studies most similar to the long horizon tasks considered in this work (\cite{cole_constructing_2008}). These works often use stabilised propensity weight estimators and assert that the propensity model is correctly specified if the mean of the stabilised weights equals 1, the validity of this is demonstrated in appendix section \ref{sec:stabilised_weights}. The use of stabilised weights and the associated balancing metric is not considered for the analysis presented since stabilised weights are biased for time dependent and dynamic (condition on the state) action assignments. This is similarly demonstrated in appendix section \ref{sec:stabilised_weights}. Whilst the stabilised weight estimator is biased, strictly speaking, a balancing metric similar based on stabilised weights could be used even if an estimator based on stabilised weights is not used. The potential methodological limitation of this however, is the requirement to define an entirely new model, $p(A_{t}|\bar{A}_{t-1})$ for the sole purpose of assessing balance. This could be approximated by $\mathbb{E}_{S}[p(A_{t}|\bar{A}_{t-1},\bar{S}_{t})]$ however, the downstream performance of this balance metric is unknown and is discussed in section \ref{sec:discussion}.\\

\section{Covariate balance}\label{sec:covariate_balance}
Covariate balance is one of several diagnostics for evaluating model misspecification/hidden confounding. The specific definition can be motivated from two perspectives, herein referred to as conditional and weighted covariate balance, defined in corollaries \ref{corol:cov_bal_as_nec_cond} and \ref{corol:weight_cov_bal_as_nec_cond}, respectively. The corolaries both rely on Theorem 3 of \cite{rosenbaum_central_1983}; assume a contextual bandit density $p(\tau) = p(S)p(A|S)p(R|A,S)$; and the following definitions:
\begin{itemize}
	\item $e(S) = P(A=1|S)$ and;
	\item A balancing score is any $b(S)$ such that, $\forall S: A \indep S | b(S)$.
\end{itemize}
The proofs for the colorraries are provided in appendix section \ref{sec:cov_bal_appendix}.\\

\begin{corollary}\label{corol:cov_bal_as_nec_cond} If the conditions of theorem \ref{ther:rosenbaum_central} hold, then $\forall b(s)$:
	\begin{align*}
		(R_{1},R_{0}) \indep A|b(S) \textrm{ and } 0<P(A=1|b(S))<1 \implies A \indep S | b(S)
	\end{align*}
\end{corollary}

\begin{corollary}\label{corol:weight_cov_bal_as_nec_cond} If the conditions of theorem \ref{ther:rosenbaum_central} hold, and:
\begin{equation}\label{equ:weighted_cov_bal}
	\mathbb{E}\Bigg[\frac{\mathbbm{1}(A=1)S}{\hat{e}(S)}-\frac{\mathbbm{1}(A=0)S}{1-\hat{e}(S)}\Bigg] = 0
\end{equation}
then, $\forall \hat{e}(s)$:
\begin{align*}
	(R_{1},R_{0}) \indep A|\hat{e}(S) \textrm{ and } 0<P(A=1|\hat{e}(S))<1
\end{align*}
\end{corollary}

Under each of these corollaries, authors have derived covariate balance metrics. The weighted covariate balance metric is defined in equation \ref{equ:weighted_cov_bal}. The conditional balance metric is defined by asserting the validity of the independence statement $A \indep S | b(S)$. Note that $A \indep S | b(S) \iff p(S|A=1,b(S)) - p(S|A=0,b(S)) = 0, \forall S$ almost everywhere. This is often approximated using only the first moments i.e., $\mathbb{E}[S|A=a,b(S)] - \mathbb{E}[S|A=a',b(S)] = 0$ (\cite{} rosenbaum book proof). However, authors have noted the necessity to check higher order moments (\cite{austin_balance_2009}). This lower order moment approximation is notably not required for the weighted covariate balance metric since the entire measure $\hat{e}(S)$ is asserted to be equivalent to $e(S)$.\\

\subsection{Implications of weighted covariate balance}\label{sec:cov_bal_implications}
\sloppy
The corollaries state that satisfying the respective balance metrics are necessary but not sufficient conditions for drawing valid causal inferences in the sense that $(R_{1},R_{0}) \indep A|f(S) \textrm{ and } 0<P(A=1|f(S))<1$ where $f$ is either a balancing score or the propensity score. It is beneficial however to establish \emph{why} valid causal inferences could not be drawn. This will be explored for the weighted estimator since it is the subject of analysis in the rest of the paper.\\
 
 Under the weighted covariate balance estimator, the test fails if $\hat{e}(S) \neq e(S)$. Assuming that there are no hidden confounders i.e., $\mathcal{S}$ is such that $(R_{1},R_{0}) \indep A|S$, then the test will fail if the \emph{functional form} of $\hat{e}(S)$ and $e(S)$ differ. This is referred to as model miss-specification. Alternatively, if there is hidden confounding then, by the uniqueness of measures, the test will fail. Consider the true propensity model, $e(S)$ and $\tilde{e}(S) = \tilde{P}(A=1|S_{\setminus k})$ where the domain of $S_{\setminus k}$ is $\mathcal{S}_{\setminus k} = \prod_{i=1,i\neq k}^{n} \mathcal{S}_{i}$ and $\mathcal{S} = \prod_{i=1}^{n} \mathcal{S}_{i}$. Valid causal inferences hold under $\tilde{e}(S)$ if $A\indep S_{\setminus k}$ which further means $e(S) = \tilde{e}(S)$. Thus, ignoring approximation error of the measure $\tilde{e}(S)$, the balancing metric will never be satisfied under missing confounding unless:
 \begin{align*}
	\mathbb{E}_{P}\Bigg[&\frac{\mathbbm{1}(A=1)S}{\tilde{e}(S)}-\frac{\mathbbm{1}(A=0)S}{1-\tilde{e}(S)}\Bigg] = 0
 \end{align*}
 $\iff A\indep S_{\setminus k}$. Note, the balancing metric is still defined with respect to the true measure $P$ however, the confounded measures $\tilde{e}$ are being evaluated. This reflects the real world scenario where the data is generated according to the true unknown measure, $P$. In this sense, the diagnostic test is a test for hidden confounding. Discussed in section \ref{sec:discussion} is observation that, if the propensity model is approximated non-parametrically, then covariate balance is entirely a test of hidden confounding.\\
 
 \subsection{Time dependence, finite sample approximation, and variance reduction}
 
Of interest for this analysis is the time dependent setting with non-binary action spaces. Assuming now, an MDP density (equation \ref{equ:mdp_density}):
\begin{equation}\label{equ:mdp_density}
	p_{\pi}(\tau) = p_{S_{0}}(s_{0})\prod_{t=0}^{H-1}p_{\pi}(a_{t}|s_{t})p_{T}(s_{t+1}|a_{t},s_{t})p_{R}(r_{t+1}|a_{t},s_{t})
\end{equation}
Then:
\begin{align*}
	\mathbb{E}\Bigg[&\frac{\mathbbm{1}(A_{0:t}=a_{0:t}')g(S_{0:t})}{\prod_{t'=0}^{t}p(A_{t'}|S_{t'})}-\frac{\mathbbm{1}(A_{0:t}=a_{0:t}'')g(S_{0:t})}{\prod_{t'=0}^{t}p(A_{t'}|S_{t'})}\Bigg] = \mathbb{E}[g(S_{0:t})] - \mathbb{E}[g(S_{0:t})]
\end{align*}
where $S_{0:t} = \{S_{0},...,S_{t}\}$. Define $g = \sum_{t'=1}^{t}\mathbbm{1}(t'=t)S_{t}$ to obtain the time $t$ balance.
This is a result of:
\begin{align*}
	\mathbb{E}\Bigg[&\frac{\mathbbm{1}(A_{0:t}=a_{0:t}')g(S_{0:t})}{\prod_{t'=0}^{t}p(A_{t'}|S_{t'})}-\frac{\mathbbm{1}(A_{0:t}=a_{0:t}'')g(S_{0:t})}{\prod_{t'=0}^{t}p(A_{t'}|S_{t'})}\Bigg] \\
	=&\int\frac{\mathbbm{1}(A_{0:t}=a_{0:t}')g(S_{0:t})}{\prod_{t'=0}^{t}p(A_{t'}|S_{t'})}p(S_{0:t},A_{0:t})ds_{0:t}a_{0:t} \\
	&-\int\frac{\mathbbm{1}(A_{0:t}=a_{0:t}'')g(S_{0:t})}{\prod_{t'=0}^{t}p(A_{t'}|S_{t'})}p(S_{0:t},A_{0:t})ds_{0:t}a_{0:t} \\
	=&\int\frac{g(S_{0:t})}{p(A_{0:t}=a_{0:t}'|S_{0:t})}p(A_{0:t}=a_{0:t}'|S_{0:t})p(S_{0:t})ds_{0:t}\\
	&-\int\frac{g(S_{0:t})}{p(A_{0:t}=a_{0:t}''|S_{0:t})}p(A_{0:t}=a_{0:t}''|S_{0:t})p(S_{0:t})ds_{0:t}\\
	=&\mathbb{E}[g(S_{0:t})] - \mathbb{E}[g(S_{0:t})] = 0
\end{align*}

The empirical balancing metric can thus be defined as:
\begin{equation}\label{equ:masd_weight_t}
	\masdtw{} = \frac{|\hat{\mu}_{a_{0:t}'}(k) - \hat{\mu}_{a_{0:t}''}(k)|}{\sqrt{n^{-1}\hat{\sigma}^{2}_{a_{0:t}'}(k)+n^{-1}\hat{\sigma}^{2}_{a_{0:t}''}(k)}}
\end{equation}
where: 
\begin{align*}
	\hat{\mu}_{a_{0:t}'}(k) =& \frac{1}{n}\sum_{i=1}^{n}\frac{\mathbbm{1}(a_{i,0:t}=a_{0:t}')g(s_{i,0:t,k})}{\prod_{t'=0}^{t}p(a_{i,t'}|s_{i,t'})}\\
	\hat{\sigma}^{2}_{a_{0:t}'}(k) =& \frac{1}{n-1}\sum_{i=1}^{n}\Bigg(\frac{\mathbbm{1}(a_{i,0:t}=a_{0:t}')g(s_{i,0:t,k})}{\prod_{t'=0}^{t}p(a_{i,t'}|s_{i,t'})} - \hat{\mu}_{a_{0:t}'}\Bigg)^{2}
\end{align*}

\subsubsection{Variance reduction}
The weighted covariate balance metric is an importance sampling (IS) estimator. Such estimators are understood to suffer from large variance. The most popular variance reduction methods for IS estimators include: clipping and Hajek. The clipped estimator is identical to $\masdtw{}$ (equation \ref{equ:masd_weight_t}) except the propensity weights, $\prod_{t'=0}^{t}p(A_{t'}|S_{t'})$ are replaced by:
\begin{align*}
	\max\Bigg(\prod_{t'=0}^{t}p(A_{t'}|S_{t'}), \tau^{-1}\Bigg)
\end{align*} 
The Hajek estimator is defined as:
\begin{equation}\label{equ:masd_hajek_weight_t}
	\masdthw{} = \frac{|\hat{\mu}_{a_{0:t}'}(k) - \hat{\mu}_{a_{0:t}''}(k)|}{\sqrt{n^{-1}\hat{\sigma}^{2}_{a_{0:t}'}(k)+n^{-1}\hat{\sigma}^{2}_{a_{0:t}''}(k)}}
\end{equation}
where: 
\begin{align*}
	\hat{\mu}_{a_{0:t}'}(k) =& \frac{\frac{1}{n}\sum_{i=1}^{n}\frac{\mathbbm{1}(a_{i,0:t}=a_{0:t}')g(s_{i,0:t,k})}{\prod_{t'=0}^{t}p(a_{i,t'}|s_{i,t'})}}{\frac{1}{n}\sum_{i=1}^{n}\frac{\mathbbm{1}(a_{i,0:t}=a_{0:t}')}{\prod_{t'=0}^{t}p(a_{i,t'}|s_{i,t'})}}\\
	\hat{\sigma}^{2}_{a_{0:t}'}(k) =& \frac{1}{n-1}\sum_{i=1}^{n}\Bigg(\frac{\frac{\mathbbm{1}(a_{i,0:t}=a_{0:t}')g(s_{i,0:t,k})}{\prod_{t'=0}^{t}p(a_{i,t'}|s_{i,t'})}}{\frac{1}{n}\sum_{i=1}^{n}\frac{\mathbbm{1}(a_{i,0:t}=a_{0:t}')}{\prod_{t'=0}^{t}p(a_{i,t'}|s_{i,t'})}} - \hat{\mu}_{a_{0:t}'}\Bigg)^{2}
\end{align*}

The derivation for the Hajek estimator is almost identical to $\masdtw{}$.
\begin{align*}
	\mathbb{E}\Bigg[&\frac{\frac{\mathbbm{1}(A_{0:t}=a_{0:t}')g(S_{0:t})}{\prod_{t'=0}^{t}p(A_{t'}|S_{t'})}}{\mathbb{E}\big[\frac{\mathbbm{1}(A_{0:t}=a_{0:t}')}{\prod_{t'=0}^{t}p(A_{t'}|S_{t'})}\big]}-\frac{\frac{\mathbbm{1}(A_{0:t}=a_{0:t}'')g(S_{0:t})}{\prod_{t'=0}^{t}p(A_{t'}|S_{t'})}}{\mathbb{E}\big[\frac{\mathbbm{1}(A_{0:t}=a_{0:t}'')}{\prod_{t'=0}^{t}p(A_{t'}|S_{t'})}\big]}\Bigg] \\
	=& \frac{\mathbb{E}[g(S_{0:t})]}{\mathbb{E}\big[\frac{\mathbbm{1}(A_{0:t}=a_{0:t}')}{\prod_{t'=0}^{t}p(A_{t'}|S_{t'})}\big]} - \frac{\mathbb{E}[g(S_{0:t})]}{\mathbb{E}\big[\frac{\mathbbm{1}(A_{0:t}=a_{0:t}'')}{\prod_{t'=0}^{t}p(A_{t'}|S_{t'})}\big]} \\
	=& \frac{\mathbb{E}[g(S_{0:t})]}{1} - \frac{\mathbb{E}[g(S_{0:t})]}{1} \\
	=& \mathbb{E}[g(S_{0:t})] - \mathbb{E}[g(S_{0:t})] \\
\end{align*}
 
 \section{Analysis of covariate balance in sepsis management}
 The proceeding section presents the core experimental results and analysis of the paper. The weighted covariate balance metric was applied to the problem of sepsis management which has been widly studied and critically, work has consistently used the same data and MDP definition. For this line of work, the causal question analysed has been, the causal effect of different dosages of intravenous (IV) fluids and vasopressors on 90 day mortality of adult sepsis patients, with exclusions being applied to palliative patients. Patient trajectories have been defined by 24 hours of observations prior to sepsis onset and 48 hours proceeding. These trajectories have been discretised into 4 hour buckets, resulting in a maximum 18 step algorithmic horizon. The action space has been defined by discretising IV fluid and vasopressor dosages into 5 buckets each, resulting in a discrete action space of dimension 25. For a more detailed breakdown of the cohort, refer to \cite{Komorowski2018}.\\

The original application of offline RL to the sepsis domain by \cite{Komorowski2018} used a discrete state space. In recent years, continuous state spaces have generally been considered (\cite{Raghu2018b}, \cite{Huang2022}, \cite{Kaushik2022}). Given this, a continuous state spaces was be used for the analysis presented. The preprocessing code for the analysis was copied from \cite{subramanian_mimc_sepsis_2023} with modifications only to explicitely export a reduced action space for the analysis in section \ref{sec:fluid_action_space}.\\

\subsection{Experimental protocol}\label{sec:experimental_protocol}

For the analysis, preprocessing steps and MDP assumptions were identical to previous works. However, of the previous analyses which used continuous state representations, non performed evaluations with propensity models and thus custom propensity models were developed.\\

Throughout, analysis statements such as "aggregation across X" will be made e.g., "median (across features) and mean (across action trajectories) ratio of variances" or "mean $\masdtw{}$ across features". As an example, if the ratio of variances is defined as: 
\begin{align*}
	\frac{\sqrt{n^{-1}\hat{\sigma}^{2}_{a_{0:t}'}(k)+n^{-1}\hat{\sigma}^{2}_{a_{0:t}''}(k)}}{\sqrt{n^{-1}\hat{\sigma}^{\prime2}_{a_{0:t}'}(k)+n^{-1}\hat{\sigma}^{\prime2}_{a_{0:t}''}(k)}}
\end{align*}
then the "median (across features) and mean (across action trajectories)" refers the following aggregation:
\begin{align*}
	\textrm{med}_{k}\Bigg(\frac{1}{|\mathcal{B}|}\sum_{a_{0:t}',a_{0:t}''} \frac{\sqrt{n^{-1}\hat{\sigma}^{2}_{a_{0:t}'}(k)+n^{-1}\hat{\sigma}^{2}_{a_{0:t}''}(k)}}{\sqrt{n^{-1}\hat{\sigma}^{2'}_{a_{0:t}'}(k)+n^{-1}\hat{\sigma}^{2'}_{a_{0:t}''}(k)}}\Bigg)
\end{align*}
where $\mathcal{B}$ defines the set of intervention sequences of length $t$ (i.e., values of $a_{0:t}',a_{0:t}''$) and $\textrm{med}$ defines the median operator with an ordering over $k$.\\

\subsubsection{MASD denominator justification}
Figures \ref{fig:mimic_sepsis_post_mean_var_ratio_cov_bal_by_time} and \ref{fig:mimic_sepsis_fluid_post_mean_var_ratio_cov_bal_by_time} display the median (across features) and mean (across action trajectories) ratio of variances for the full action space analysed in section \ref{sec:full_action_space} and the fluid action space analysed in section \ref{sec:fluid_action_space}, respectively. Observe that the ratio of variances is divergent by the fourth timestep for the full action space and by the fifth timestep for the fluid action space. Whilst heuristic, it is clear that the Welsh type denominator (derived for samples with unequal variances) should be used, motivating the selection in equations \ref{equ:masd_cond} and \ref{equ:masd_cond}.\\

\begin{figure}[!h]
	\centering
	\begin{subfigure}{0.49\linewidth}
		\includegraphics[width=\linewidth]{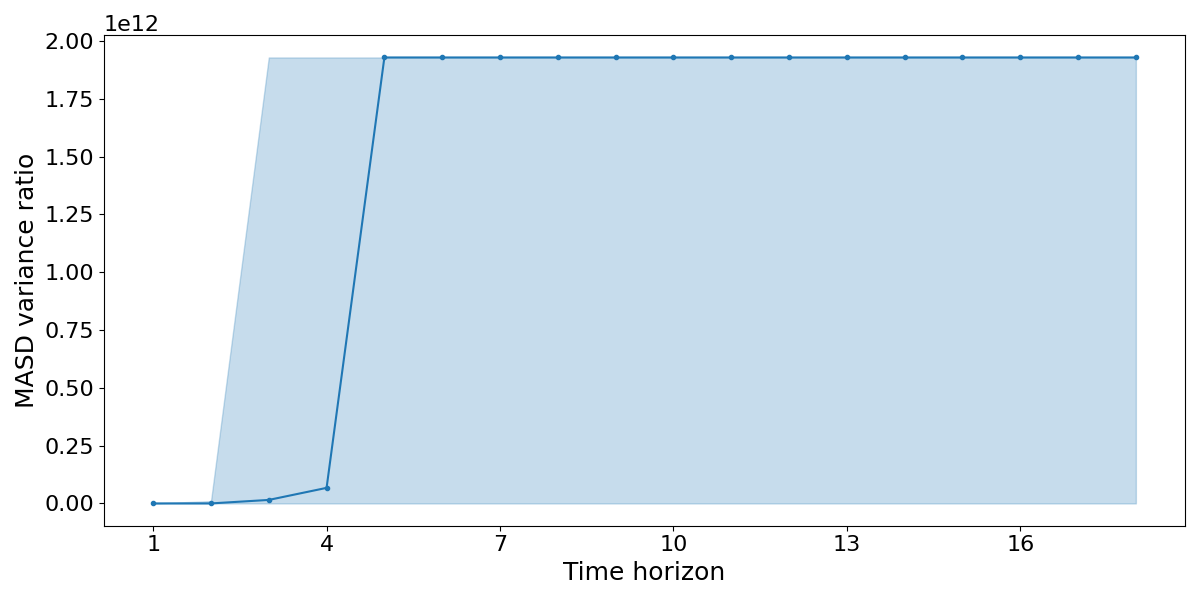}
		\caption{Training dataset}
	\end{subfigure}
	\begin{subfigure}{0.49\linewidth}
		\includegraphics[width=\linewidth]{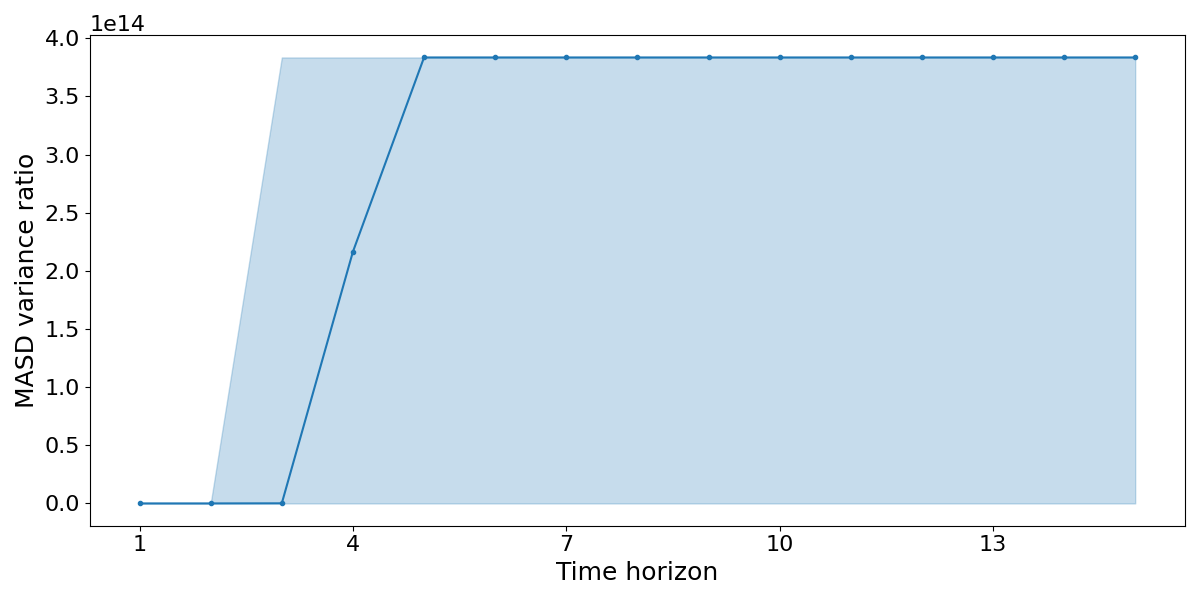}	
		\caption{Testing dataset}
	\end{subfigure}
	\caption{The figure describes aggregations of the ratio of variance values used in the denominator of the vanilla MASD metric on the full action space dataset. The values on the plot have been clipped to a maximum of 25\% of the original aggregated values. The solid line describes the median (across features) and mean (across action trajectories). The shaded region defines the min and max across features.}
	\label{fig:mimic_sepsis_post_mean_var_ratio_cov_bal_by_time}
\end{figure}

\begin{figure}[!h]
	\centering
	\begin{subfigure}{0.49\linewidth}
		\includegraphics[width=\linewidth]{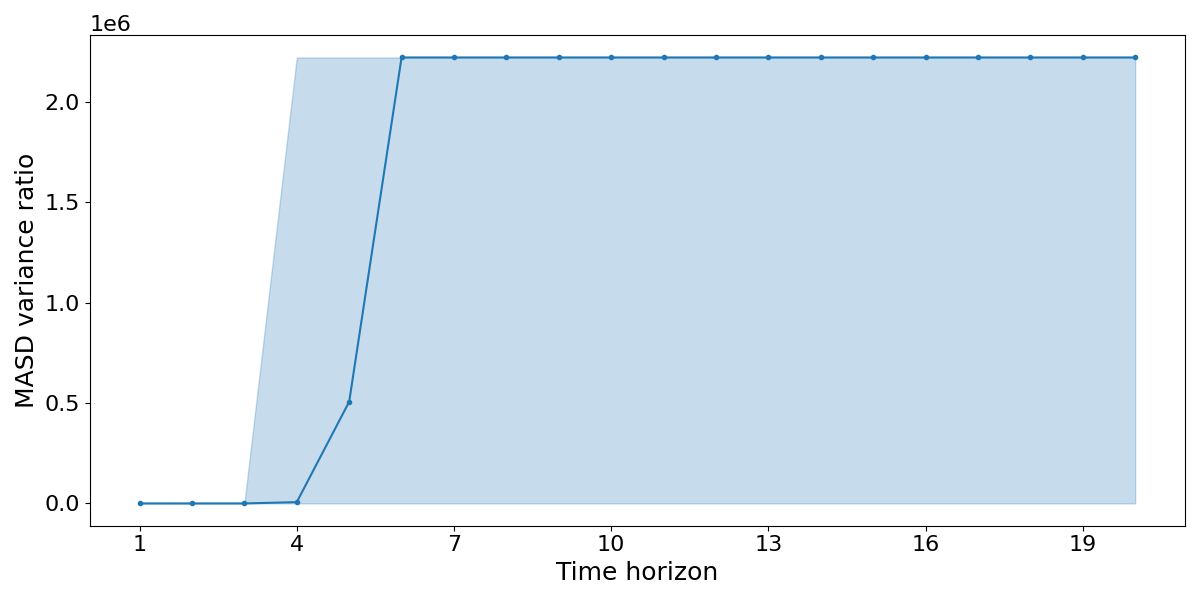}
		\caption{Training dataset}
	\end{subfigure}
	\begin{subfigure}{0.49\linewidth}
		\includegraphics[width=\linewidth]{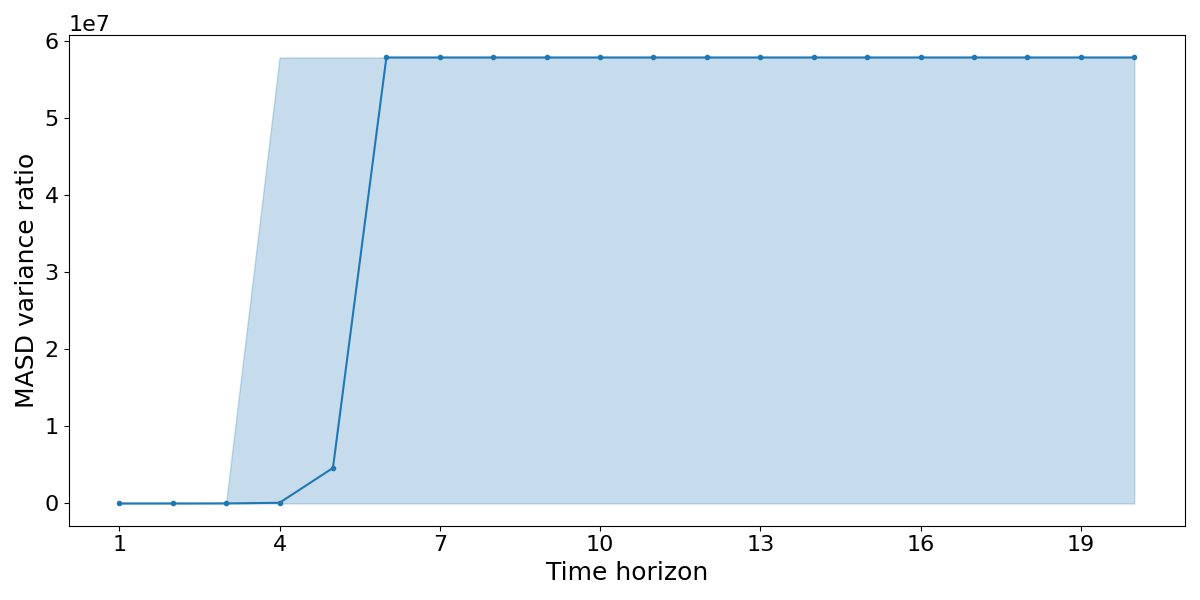}	
		\caption{Testing dataset}
	\end{subfigure}
	\caption{The figure describes aggregations of the ratio of variance values used in the denominator of the vanilla MASD metric on the fluid action space dataset. The values on the plot have been clipped to a maximum of 25\% of the original aggregated values. The solid line describes the median (across features) and mean (across action trajectories). The shaded region defines the min and max across features.}
	\label{fig:mimic_sepsis_fluid_post_mean_var_ratio_cov_bal_by_time}
\end{figure}

\subsubsection{Covariate balance assessment criteria}
Since $\sqrt{\frac{1}{2}\hat{\sigma}^{2}_{a_{0:t}'}(k)+\frac{1}{2}\hat{\sigma}^{2}_{a_{0:t}''}(k)} > \sqrt{n^{-1}\hat{\sigma}^{2}_{a_{0:t}'}(k)+n^{-1}\hat{\sigma}^{2}_{a_{0:t}''}(k)}$ for $n>2$, the implication of using the Welsh style denominator is that the resulting covariate balance will, by definition, be larger than under the Wald style denominator. Arguably the 0.25 (or 0.10) thresholds used in the literature should still be valid since the use of the Welsh style denominator should be viewed merely as an adjustment for the unequal variances. However, since these thresholds have been established experimentally rather than theoretically, this implication is not necessarily rigorous. As such, covariate balance was assessed under the following conditions:
\begin{itemize}
	\item The value of $\masdtwmean{}$ was less than 25\% for all timesteps  (i.e., the existing condition posed by \cite{cannas_comparison_2019}) and/or;
	\item The value of $\textrm{MASD-Diff}_{\textrm{Weight-}t}(k,t)$ was positive for all timesteps;
\end{itemize}
where:
\begin{align*}
	\textrm{MASD-Diff}_{\textrm{Weight-}t}(k,a_{0:t}',a_{0:t}'') = & \textrm{MASD}_{\textrm{UnWeight-}t}(k,a_{0:t},a_{0:t}'') \\
	&-\masdtw{} \\
	\textrm{MASD-Diff}_{\textrm{Weight-}t}(k,t) = & \frac{1}{|\mathcal{W}|}\sum_{a_{0:t}',a_{0:t}''} \textrm{MASD-Diff}_{\textrm{Weight-}t}(k,a_{0:t}',a_{0:t}'') \\
	\masdtwmean{} = & \frac{1}{|\mathcal{W}|}\sum_{a_{0:t}',a_{0:t}''} \masdtw{}
\end{align*}
\sloppy
and $\mathcal{W} = \{(a_{0:t}',a_{0:t}''):a_{0:t}',a_{0:t}''\in \mathcal{A}_{0:t}\times\mathcal{A}_{0:t}, a_{0:t}' \neq a_{0:t}''\}$ and $\textrm{MASD}_{\textrm{UnWeight-}t}(k,a_{0:t}',a_{0:t}'')$ is defined as $\masdtw{}$ with $p(a|s) = 1, \forall s,a \in \mathcal{S}\times\mathcal{A}$. Observe that a positive value for $\masdtdiffmean{}$ defines an improvement in covariate balance, on average across intervention combinations, under the propensity score and feature $k$, hence the desire for the metric to be positive.\\

Existing statistical learning theory necessitates the use of separate validation sets to obtain unbiased predictions of model performance (\cite{efron_how_1986}, \cite{hastie_elements_2009}). For example, \cite{efron_how_1986} demonstrated that the bias of the training mean squared error, $\frac{1}{n}\sum_{i=1}^{n}(y_{i},\hat{f}(x_{i}))^{2}$ with respect to the expected error (with a fixed function $\hat{f}$) is equivalent to, $\frac{2}{n}\sum_{i=1}^{n}\textrm{Cov}(y_{i},\hat{f}(x_{i}))$ i.e., the training error under estimates the expected mean squared error in proportion to the covariance of the training labels and the predicted training labels. Clearly, this under estimation is correlated with the learning process for $\hat{f}$.\\

It is intuitive to see that, when used as a training objective (\cite{Imai2014}), the validation covariate balance should be used to perform evaluations. It is similarly reasonable to assume that the validation covariate balance statistics should be used even when covariate balance is not used as a training objective. No formal derivation was provided for either of these conjectures and as such, where relevant, all of the results presented in this analysis are shown for both the training and a separate testing set.\\

\subsection{Results}

The results are presented in three sections where section \ref{sec:full_action_space} performs the analysis as described in section \ref{sec:experimental_protocol}, and section \ref{sec:fluid_action_space} and \ref{sec:reduced_time_horizon} are supplementary analyses using an identical setup except:
\begin{itemize}
	\item Section \ref{sec:fluid_action_space} uses an action space that is restricted to only fluid dosages, and thus is of dimension 5. 
	\item Section \ref{sec:reduced_time_horizon} considers both the full action space from section \ref{sec:full_action_space} and reduced space from section \ref{sec:fluid_action_space} but considers truncated trajectory lengths.
\end{itemize}
The dataset corresponding to the full action space is herein referred to as the "sepsis" dataset, whilst the dataset corresponding to the fluid only action space is herein referred to as the "sepsis-fluid" dataset.\\

To conduct all analyses, two seeds were used which randomised the dataset splitting and propensity model training. To produce the covariate balance figures and data presented, unless otherwise specified, the average balance statistics of both seeds were taken.\\

\subsubsection{Propensity score models}
Since the existing studies considering treatment regiemens for sepsis (\cite{Raghu2018b}, \cite{Huang2022}, \cite{Kaushik2022}) had not developed propensity models, these had to be developed as part of the analysis. For the full action space and fluid action space three xgboost (\cite{Chen2016}) models were developed. The first two were identical in hyperparameters, using defaults with minimal regularisation (the full list is detailed in appendix section \ref{sec:propens_models_appendix}) however, xgboost\_1\_w used a weighted objective, where weights were defined per label as the reciprocal of the number of occurrences in the training set. The final model, xgboost\_1\_w\_optuna was an xgboost model with hyperparameters tuned with optuna (\cite{optuna_2019}) and with a weighted objective. Whilst the unweighted (xgboost\_1) and wighted (xgboost\_1\_w) models performed identically for both datasets (table \ref{tbl:propense_results}), the weighted regression was deemed to make more sense from a model development standpoint and thus it was retained. For the experiments presented in section \ref{sec:reduced_time_horizon}, assessing the reduced time horizon, only the xgboost\_1\_w and xgboost\_1\_w\_optuna were trained for the reason previously stated.\\

The aim of the propensity score model development was to obtain a sufficiently strong model, resembling that which a practitioner might reasonably develop. As such, for each dataset, the strongest performing model according to a held-out validation was selected as the model to use for the scenario. For example, as demonstrated by table \ref{tbl:propense_results}, under the sepsis dataset without trajectory truncation, the xgboost\_1\_w\_optuna model performed the best and thus it was selected. In contrast, under the sepsis-fluid with maximum trajectory length 8, the xgboost\_1\_w model was selected.\\

\begin{table}[!h]
	\centering
	\begin{tabular}{ccccc}
		\toprule
		\makecell{Dataset\\(Max trajectory length)} & Model & \multicolumn{3}{c}{Macro F1 Score} \\
		& & Training & Validation & Test \\
		\midrule
		\midrule
		Sepsis & xgboost\_1 & 0.94 & 0.455 & 0.467 \\
		Sepsis & xgboost\_1\_w & 0.94 & 0.455 & 0.467\\
		Sepsis & xgboost\_1\_w\_optuna & 0.622 & \textbf{0.457} & 0.481\\
		\midrule
		\midrule
		Sepsis-fluid & xgboost\_1 & 0.756 & 0.657 & 0.656 \\
		Sepsis-fluid & xgboost\_1\_w & 0.756 & 0.657 & 0.656 \\
		Sepsis-fluid & xgboost\_1\_w\_optuna & 0.711 & \textbf{0.662} & 0.657 \\
		\midrule
		\midrule
		Sepsis (2) & xgboost\_1\_w & 0.999 & \textbf{0.356} & 0.394 \\
		Sepsis (2) & xgboost\_1\_w\_optuna & 0.616 & 0.348 & 0.377 \\
		\midrule
		Sepsis (5) & xgboost\_1\_w & 0.977 & 0.352 & 0.382 \\
		Sepsis (5) & xgboost\_1\_w\_optuna & 0.431 & \textbf{0.379} & 0.397 \\
		\midrule
		Sepsis (8) & xgboost\_1\_w & 0.961 & 0.393 & 0.415 \\
		Sepsis (8) & xgboost\_1\_w\_optuna & 0.480 & \textbf{0.399} & 0.424 \\
		\midrule
		Sepsis (13) & xgboost\_1\_w & 0.948 & 0.429 & 0.449 \\
		Sepsis (13) & xgboost\_1\_w\_optuna & 0.458 & \textbf{0.442} & 0.462 \\
		\midrule
		\midrule
		Sepsis-fluid (2) & xgboost\_1\_w & 0.994 & \textbf{0.660} & 0.651 \\
		Sepsis-fluid (2) & xgboost\_1\_w\_optuna & 0.690 & 0.659 & 0.656 \\
		\midrule
		Sepsis-fluid (5) & xgboost\_1\_w & 0.869 & 0.632 & 0.632 \\
		Sepsis-fluid (5) & xgboost\_1\_w\_optuna & 0.716 & \textbf{0.633} & 0.630 \\
		\midrule
		Sepsis-fluid (8) & xgboost\_1\_w & 0.802 & \textbf{0.637} & 0.636 \\
		Sepsis-fluid (8) & xgboost\_1\_w\_optuna & 0.661 & 0.636 & 0.635 \\
		\midrule
		Sepsis-fluid (13) & xgboost\_1\_w & 0.765 & \textbf{0.647} & 0.648 \\
		Sepsis-fluid (13) & xgboost\_1\_w\_optuna & 0.660 & \textbf{0.647} & 0.648 \\
		\bottomrule
	\end{tabular}
	\caption{The table displays the performance of propensity models evaluated on each of the sepsis, sepsis-fluid and truncated versions of the datasets. The bracketed numbers in the Dataset (Max trajectory length) column define the maximum trajectory length for the truncated scenarios.}
	\label{tbl:propense_results}
\end{table}

\FloatBarrier

\subsubsection{Importance sampling variance and floating point approximations}\label{sec:is_var_fp_approx}
The large variance of the importance sampling estimator is exacerbated over longer time horizons. This issue is further exacerbated by computations required for variance calculations (i.e., squaring already large numbers) and the floating point (FP) approximation of computers. For the vanilla importance sampling metric, values often overflowed the floating point approximation and became infinite, particularly when calculating the sample variance required for the denominators of the various MASD metrics. This was initially handled by imputing the overflowed variance values however, this lead to miss-leading results. For example, imputing the variance values obscured the observation that: the denominator of the MASD metric prevents a fair comparison across time horizons (figure \ref{fig:mimic_sepsis_xgboost_4_w_prop_finite_denom} and discussed in section \ref{sec:full_action_space}).\\

Whether a more sophisticated approach to handling FP approximations existed was not explored. However, the analysis suggests that, at the point that FP approximations become problematic, the variance of the estimator is so large that: covariate balance is highly unlikely to be achieved (figure \ref{fig:mimic_sepsis_xgboost_4_w_cov_bal_by_time_numer_finite_denom} and discussed in section \ref{sec:full_action_space}) and; any conclusions drawn should be treated with a high level of skepticism. As such, whilst flawed, the approach used for handling FP approximations is unlikely to have altered the conclusions drawn.\\

\subsubsection{Full action space (Sepsis dataset)}\label{sec:full_action_space}

Figure \ref{fig:mimic_sepsis_xgboost_4_w_post_masd_avg_across_feat} displays the median $\masdtwmean{}$ value, across features. In expectation, any deviation from 0 suggests conditional exchangeability does not hold however, to account for sampling variability, the figure should be interpreted as: the larger the deviation from 0, the greater the likelihood that conditional exchangeability does not hold. As described in section \ref{sec:experimental_protocol}, an MASD value of less than 0.25 is colloquially accepted as a reasonable level balance. Under this metric, covariate balance is not present in the study and there is a reasonable chance of obtaining biased causal effect estimations.\\

Despite this, the covariate balance seemingly improves with horizon. On average the median MASD reduces with horizon length. This however, is demonstrated to be an artefact of the metric. Figure \ref{fig:mimic_sepsis_xgboost_4_w_cov_bal_by_time_numer_finite_denom} displays the mean (across action trajectories) and median (across features) numerator and denominator values of the $\masdtw{}$ metric for the model. The aggregations for the plot are defined over those action trajectory pairs with finite denominator values, the proportion of which is displayed in figure \ref{fig:mimic_sepsis_xgboost_4_w_prop_finite_denom}. In calculating the $\masdtwmean{}$ value, displayed in figure \ref{fig:mimic_sepsis_xgboost_4_w_post_masd_avg_across_feat}, those action trajectory combinations with overfload variance values were excluded from the calculation. This resulted in a selection bias, giving the impression that covariate balance was around 1 or marginally below, accounting for the rounding error (as demonstrated by \ref{fig:mimic_sepsis_xgboost_4_w_cov_bal_by_time_numer_finite_denom}). Had denominator values not overflown, the $\masdtwmean{}$ metric values (displayed in \ref{fig:mimic_sepsis_xgboost_4_w_post_masd_avg_across_feat}) would have decreased even further with trajectory horizon however, this is clearly an artefact of the estimator. Examining the unormalised numerator in figure \ref{fig:mimic_sepsis_xgboost_4_w_cov_bal_by_time_numer_finite_denom}, clearly demonstrates a divergence in covariate balance as a result of uncontrolled, weights. As such, the variance of the importance sampling estimator remains problematic.\\

\begin{figure}[!h]
	\centering
	\begin{subfigure}{0.49\linewidth}
		\includegraphics[width=\linewidth]{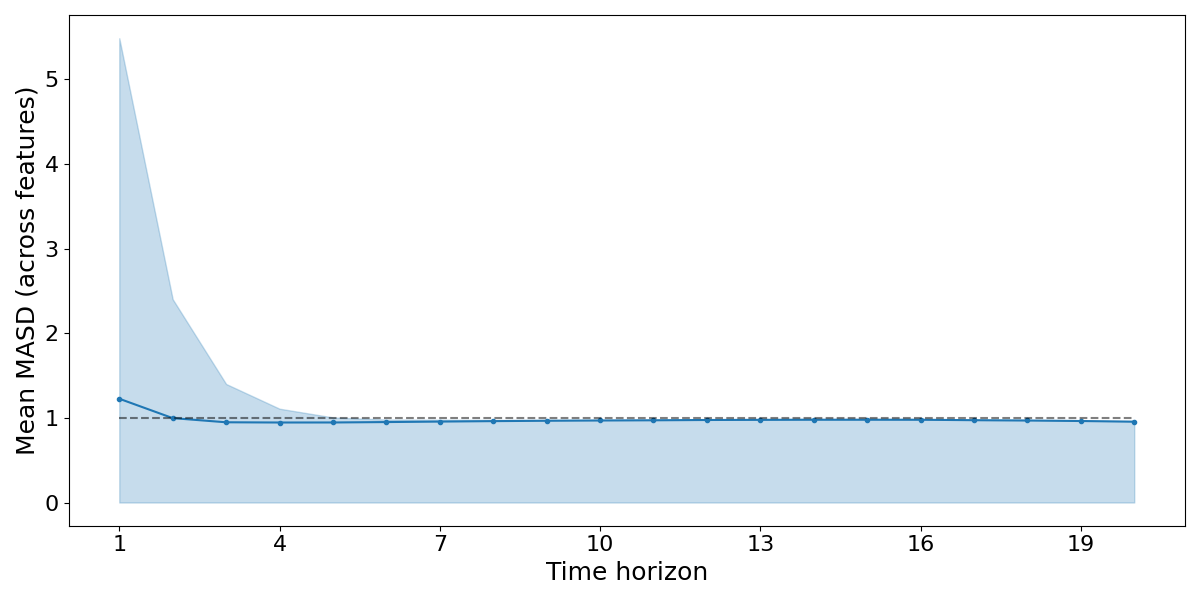}
		\caption{Training dataset}	
	\end{subfigure}
	\begin{subfigure}{0.49\linewidth}
		\includegraphics[width=\linewidth]{images/mimic_sepsis/no_clip_cov_bal_post_only/cov_bal_by_time_Train.png}
		\caption{Testing dataset}
	\end{subfigure}
	\caption[The figure describes the median mean MASD value, across features on the sepsis dataset. The x-axis displays the time horizon and the y-axis displays the median mean MASD value as the solid line with the min and max values described by the shaded region.]{The figure describes the median $\masdtwmean{}$ value, across features on the sepsis dataset. The x-axis displays the time horizon and the y-axis displays the median $\masdtwmean{}$ value as the solid line with the min and max values described by the shaded region. The dotted line is a reference for the value 1.}
	\label{fig:mimic_sepsis_xgboost_4_w_post_masd_avg_across_feat}
\end{figure}

\begin{figure}[!h]
	\centering
	\begin{subfigure}{0.49\linewidth}
		\includegraphics[width=\linewidth]{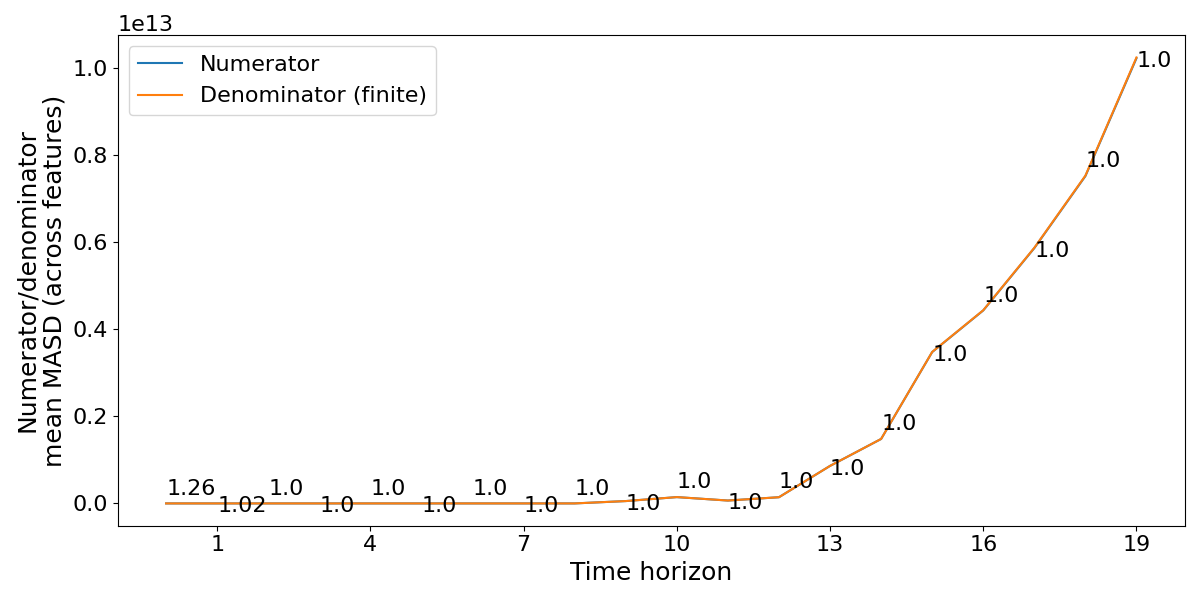}
		\caption{Training dataset}	
	\end{subfigure}
	\begin{subfigure}{0.49\linewidth}
		\includegraphics[width=\linewidth]{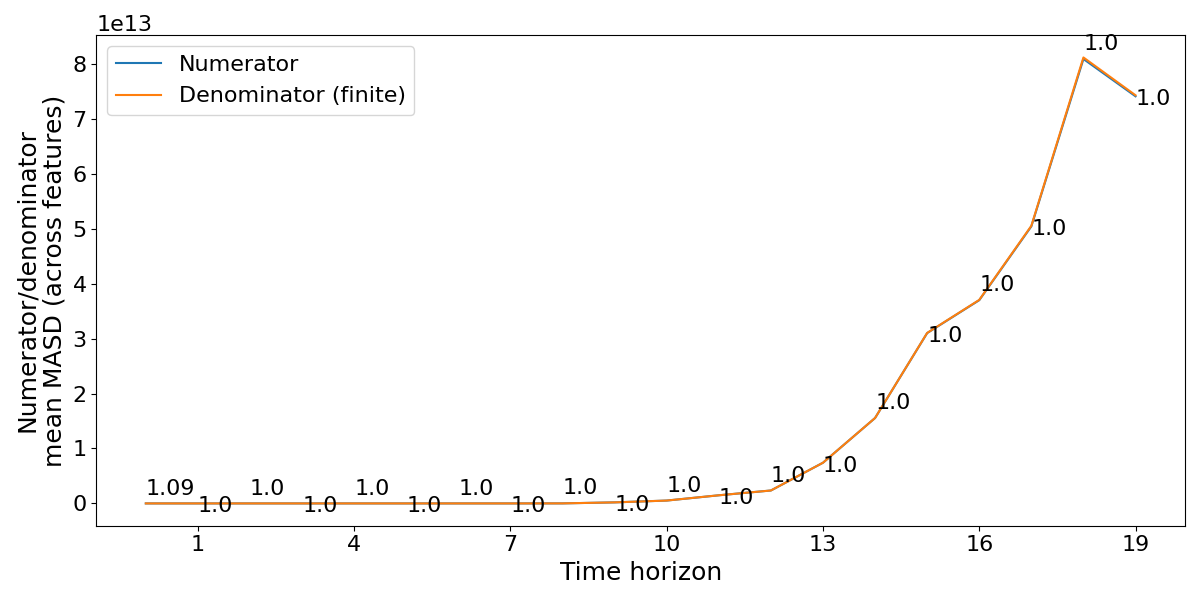}
		\caption{Testing dataset}	
	\end{subfigure}
	\caption[The figure describes the median numerator and denominator of the MASD metric, across features on the sepsis dataset, under the vanilla estimator. The x-axis displays the time horizon. Aggregations are taken over action trajectory pairs with finite denominators. The text on the figure defines the ratio of the numerator and denominator values, rounded to 2 decimal places.]{The figure describes the median numerator and denominator of the $\masdtw{}$ metric, across features on the sepsis dataset, under the vanilla estimator. The x-axis displays the time horizon. Aggregations are taken over action trajectory pairs with finite denominators. The text on the figure defines the ratio of the numerator and denominator values, rounded to 2 decimal places.}
	\label{fig:mimic_sepsis_xgboost_4_w_cov_bal_by_time_numer_finite_denom}
\end{figure}

\begin{figure}[!h]
	\centering
	\begin{subfigure}{0.49\linewidth}
		\includegraphics[width=\linewidth]{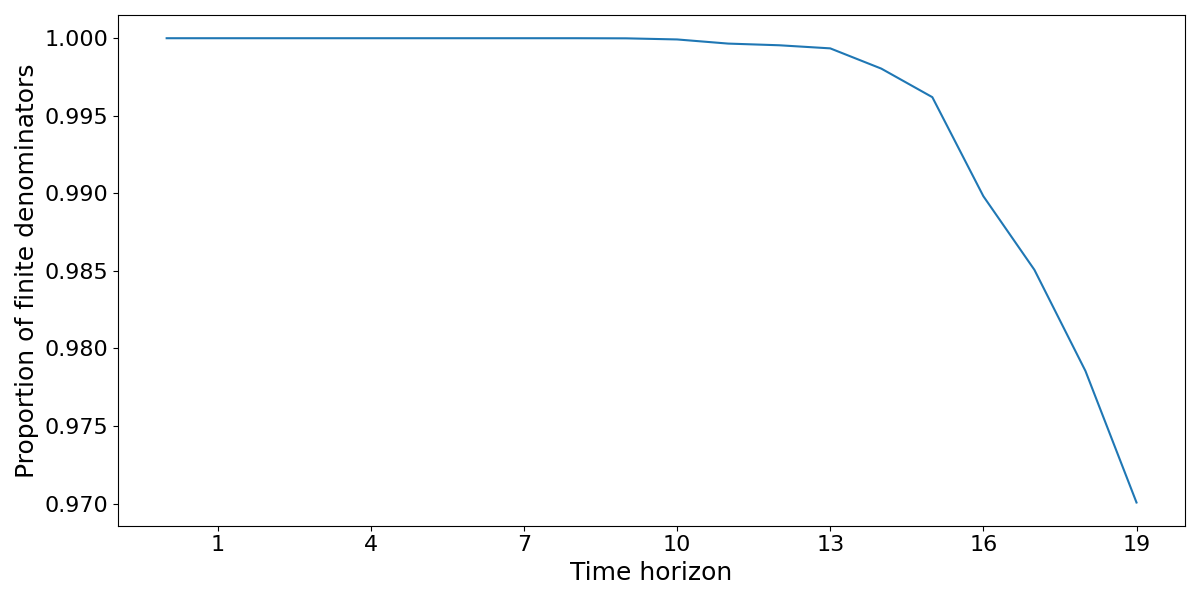}
		\caption{Training dataset}	
	\end{subfigure}
	\begin{subfigure}{0.49\linewidth}
		\includegraphics[width=\linewidth]{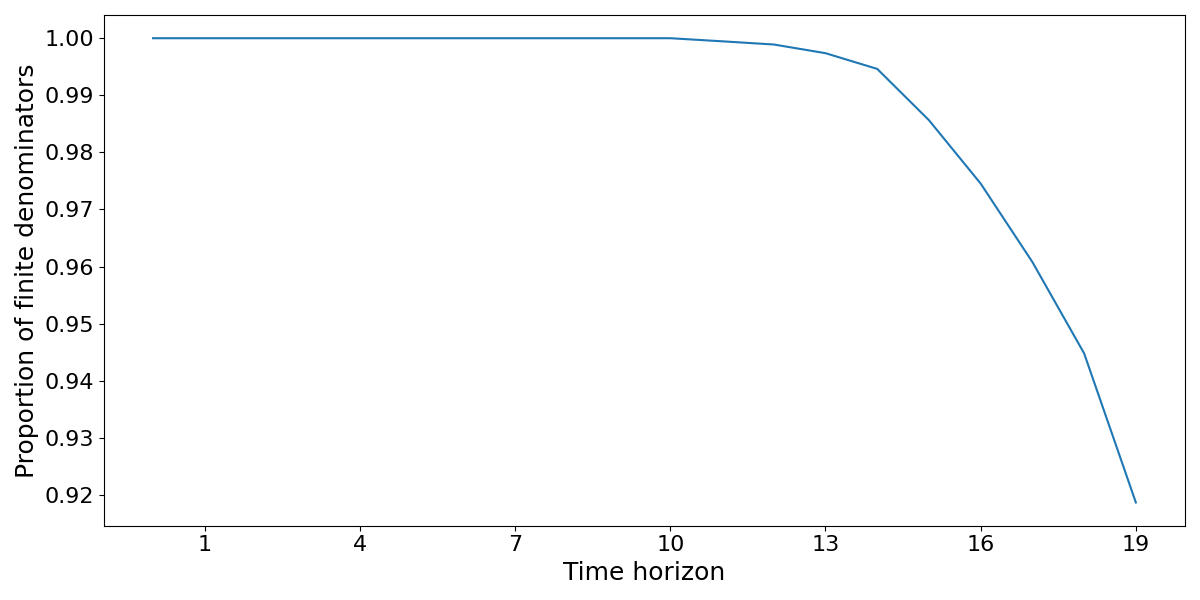}
		\caption{Testing dataset}	
	\end{subfigure}
	\caption[The figure displays the proportion of action trajectory pairs where the associated denominator value for MASD is finite. The MASD values are defined with respect to the sepsis dataset, under the vanilla estimator, with the x-axis displaying the time horizon and y-axis displaying the proportion of finite denominators.]{The figure displays the proportion of action trajectory pairs where the associated denominator value for $\masdtw{}$ is finite. The $\masdtw{}$ values are defined with respect to the sepsis dataset, under the vanilla estimator, with the x-axis displaying the time horizon and y-axis displaying the proportion of finite denominators.}
	\label{fig:mimic_sepsis_xgboost_4_w_prop_finite_denom}
\end{figure}

\FloatBarrier

Figure \ref{fig:mimic_sepsis_xgboost_4_w_post_masd_avg_across_feat_clip} displays the median $\masdtwmean{}$ but for the clipped importance sampling estimator (\cite{ionides_truncated_2008}) and figure \ref{fig:mimic_sepsis_xgboost_4_w_post_masd_avg_across_feat_hajek}  displays the median $\masdtwmean{}$ for the Hajek estimator. Both the clipped and Hajek estimators present a similar story to the vanilla estimator. Under the clipped estimator covariate balance cannot be concluded to be achieved, since the metric values are greater than 0.25. Without a formal analysis, it is unclear whether the 0.25 cut-off proposed by \cite{cannas_comparison_2019} can be applied to the Hajek estimator. Under both the clipped and Hajek estimators, similarly to the vanilla estimator, balance appears to improve with horizon. Figures \ref{fig:mimic_sepsis_xgboost_4_w_cov_bal_by_time_numer_finite_denom_clip} and \ref{fig:mimic_sepsis_xgboost_4_w_cov_bal_by_time_numer_finite_denom_hajek} display the decomposition of $\masdtw{}$ by numerator and denominator (similarly to figure \ref{fig:mimic_sepsis_xgboost_4_w_cov_bal_by_time_numer_finite_denom}). In comparison to the vanilla estimator, the unormalised numerator of the $\masdtwmean{}$ metric does not diverge, suggesting that the apparent improvement in covariate balance under the estimator at longer time horizons can be trusted. However, subsequent analysis (section \ref{sec:bias_invest_clip_hajek}) will demonstrate that this apparent improvement is an artefact of the estimators.\\
 
 \begin{figure}[!h]
	\centering
	\begin{subfigure}{0.49\linewidth}
		\includegraphics[width=\linewidth]{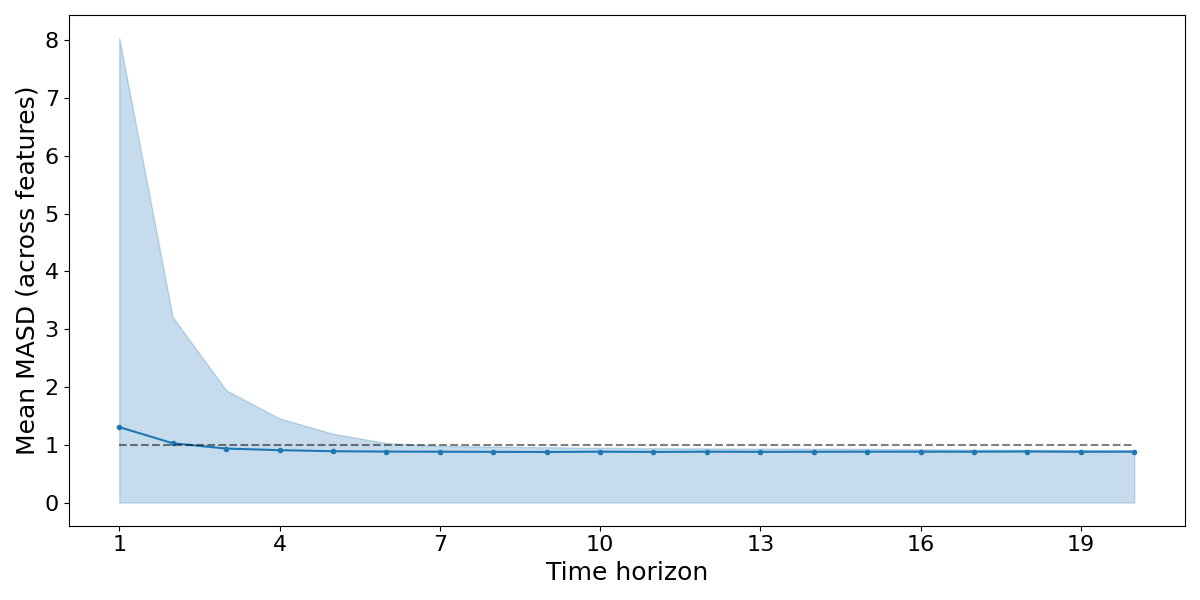}
		\caption{Training dataset}	
	\end{subfigure}
	\begin{subfigure}{0.49\linewidth}
		\includegraphics[width=\linewidth]{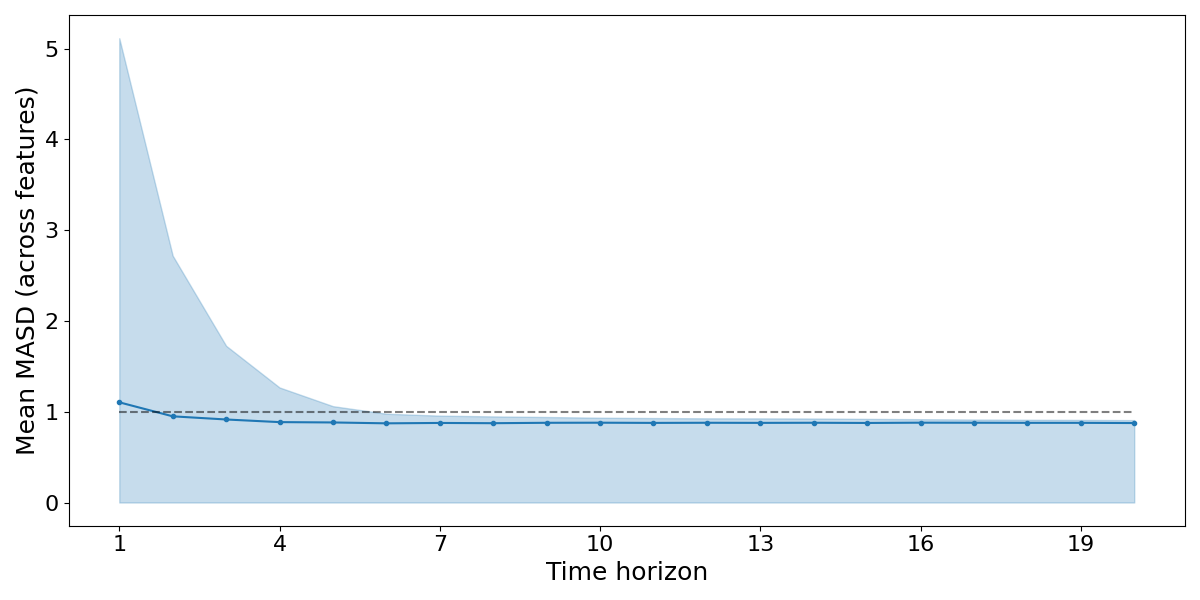}
		\caption{Testing dataset}
	\end{subfigure}
	\caption[The figure describes the median mean MASD value using clipped importance sampling, across features on the sepsis dataset. The x-axis displays the time horizon and the y-axis displays the median mean MASD value as the solid line with the min and max values described by the shaded region.]{The figure describes the median $\masdtwmean{}$ value using clipped importance sampling, across features on the sepsis dataset. The x-axis displays the time horizon and the y-axis displays the median $\masdtwmean{}$ value as the solid line with the min and max values described by the shaded region. The dotted line is a reference for the value 1.}
	\label{fig:mimic_sepsis_xgboost_4_w_post_masd_avg_across_feat_clip}
\end{figure}

\begin{figure}[!h]
	\centering
	\begin{subfigure}{0.49\linewidth}
		\includegraphics[width=\linewidth]{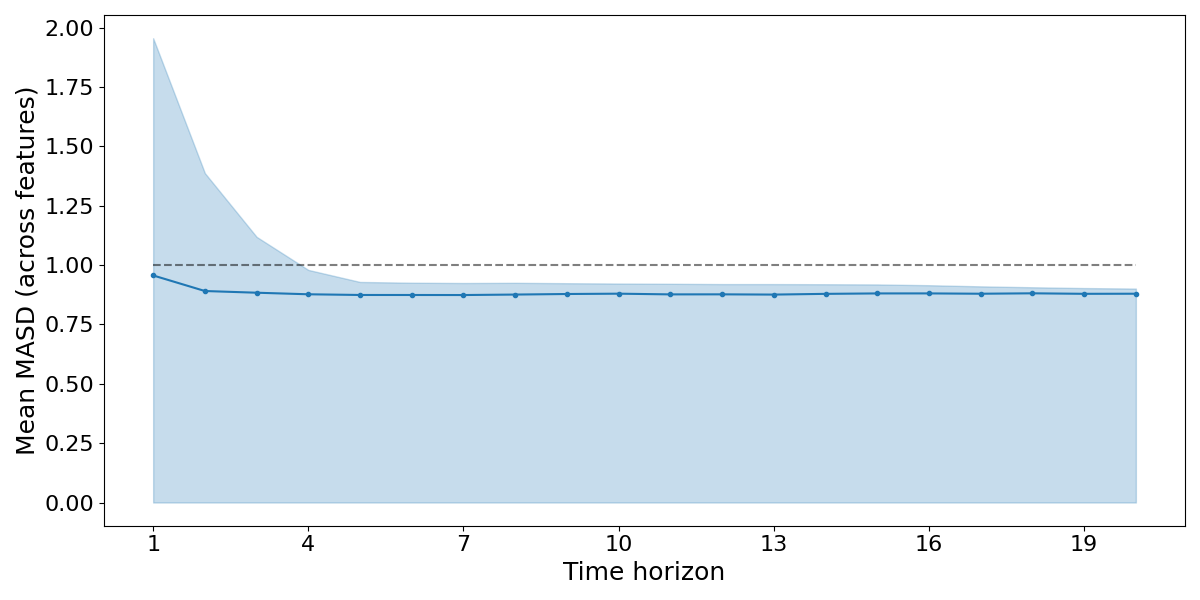}
		\caption{Training dataset}	
	\end{subfigure}
	\begin{subfigure}{0.49\linewidth}
		\includegraphics[width=\linewidth]{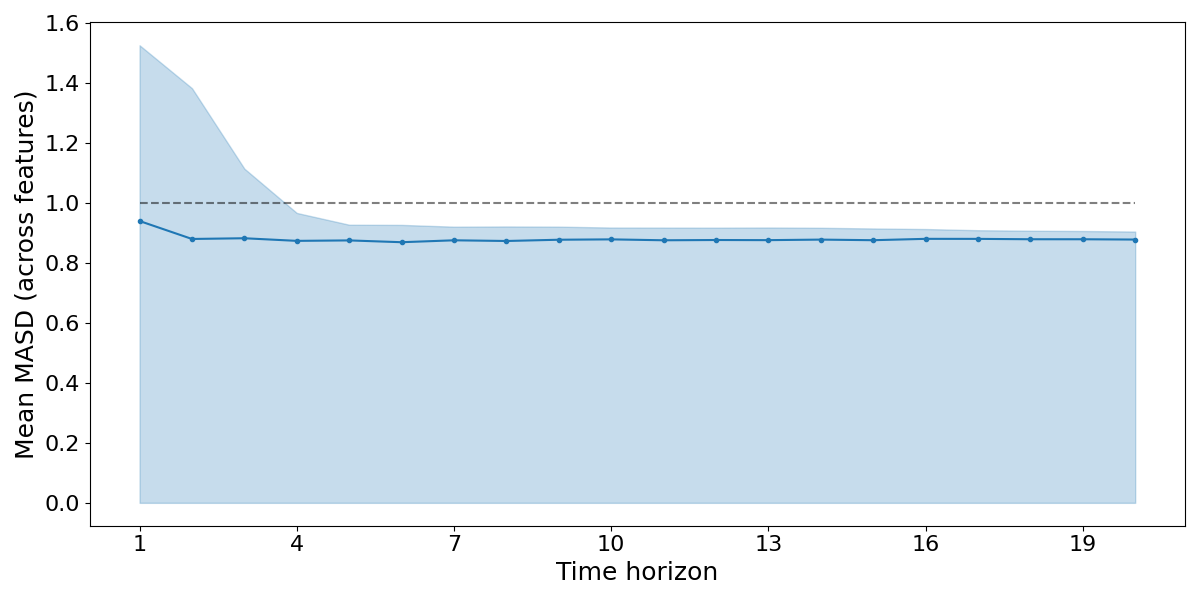}
		\caption{Testing dataset}
	\end{subfigure}
	\caption[The figure describes the median mean MASD value using Hajek importance sampling, across features  on the sepsis dataset. The x-axis displays the time horizon and the y-axis displays the median mean MASD value as the solid line with the min and max values described by the shaded region.]{The figure describes the median $\masdtwmean{}$ value using Hajek importance sampling, across features on the sepsis dataset. The x-axis displays the time horizon and the y-axis displays the median $\masdtwmean{}$ value as the solid line with the min and max values described by the shaded region. The dotted line is a reference for the value 1.}
	\label{fig:mimic_sepsis_xgboost_4_w_post_masd_avg_across_feat_hajek}
\end{figure}
 
\begin{figure}[!h]
	\centering
	\begin{subfigure}{0.49\linewidth}
		\includegraphics[width=\linewidth]{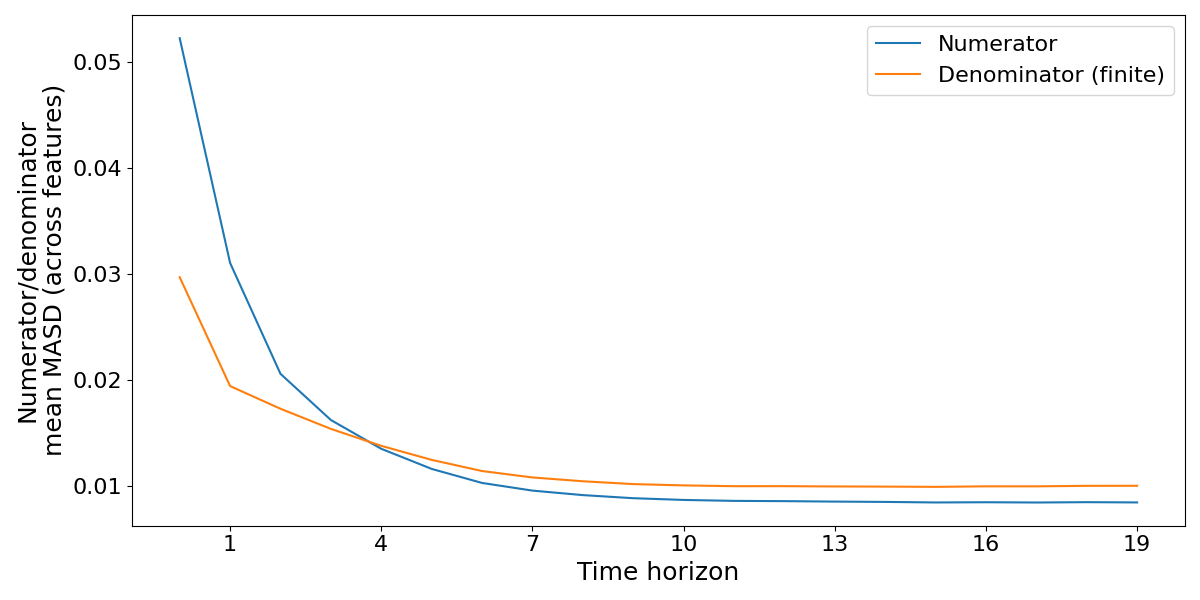}
		\caption{Training dataset}
	\end{subfigure}
	\begin{subfigure}{0.49\linewidth}
		\includegraphics[width=\linewidth]{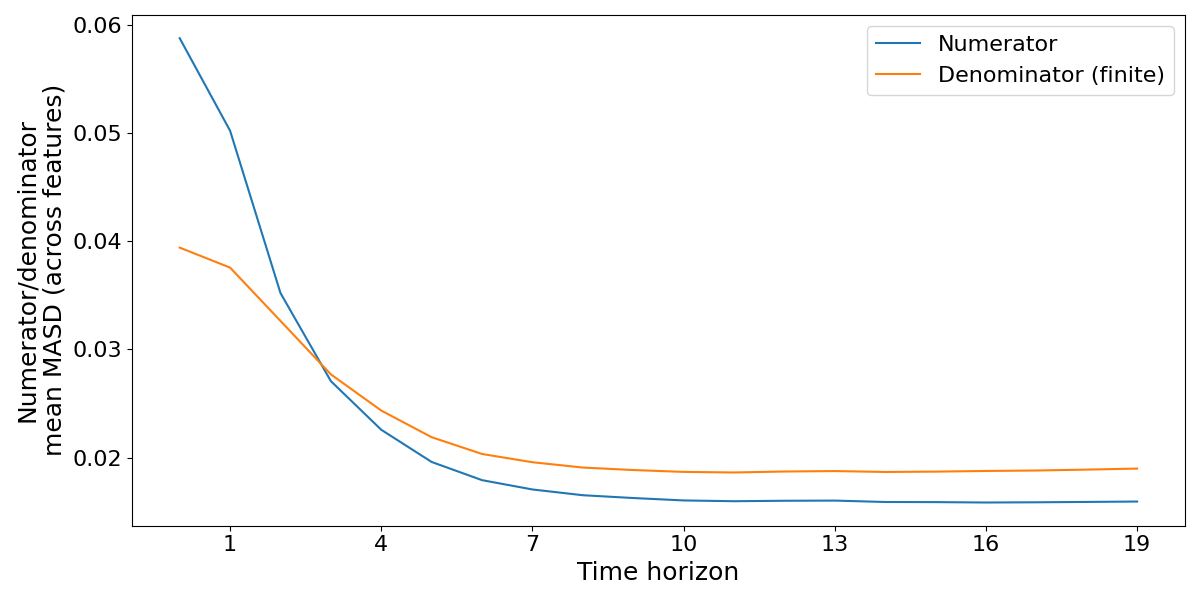}
		\caption{Testing dataset}
	\end{subfigure}
	\caption[The figure describes the median numerator and denominator of the MASD metric, across features on the sepsis dataset, under the clipped estimator. The x-axis displays the time horizon. Aggregations are taken over action trajectory pairs with finite denominators.]{The figure describes the median numerator and denominator of the $\masdtw{}$ metric, across features on the sepsis dataset, under the clipped estimator. The x-axis displays the time horizon. Aggregations are taken over action trajectory pairs with finite denominators.}
	\label{fig:mimic_sepsis_xgboost_4_w_cov_bal_by_time_numer_finite_denom_clip}
\end{figure}

\begin{figure}[!h]
	\centering
	\begin{subfigure}{0.49\linewidth}
		\includegraphics[width=\linewidth]{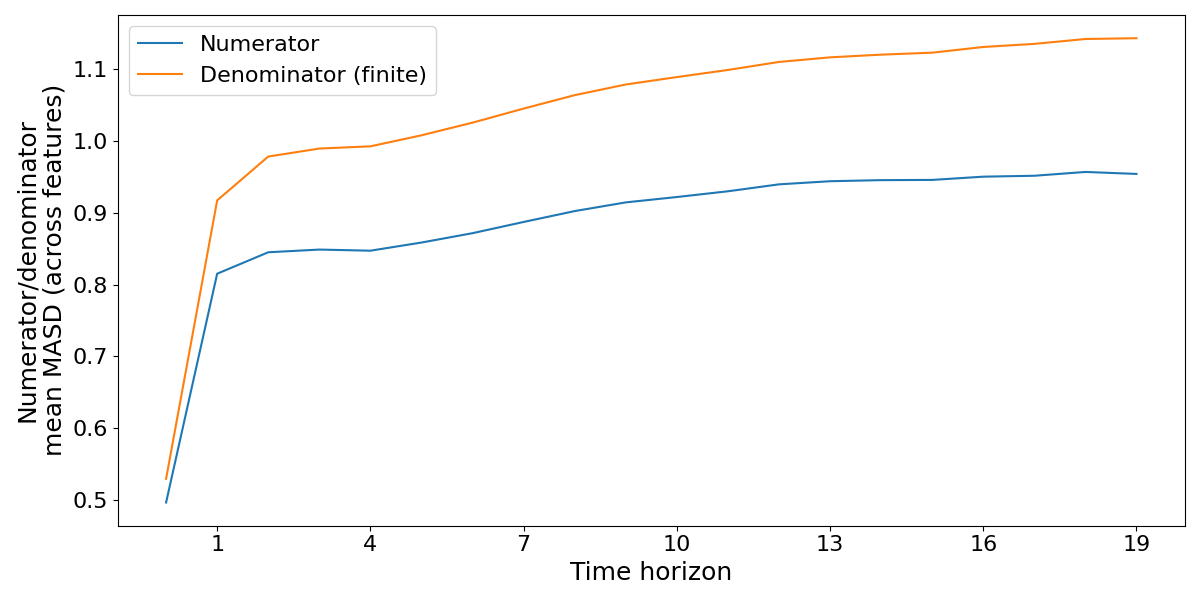}
		\caption{Training dataset}
	\end{subfigure}
	\begin{subfigure}{0.49\linewidth}
		\includegraphics[width=\linewidth]{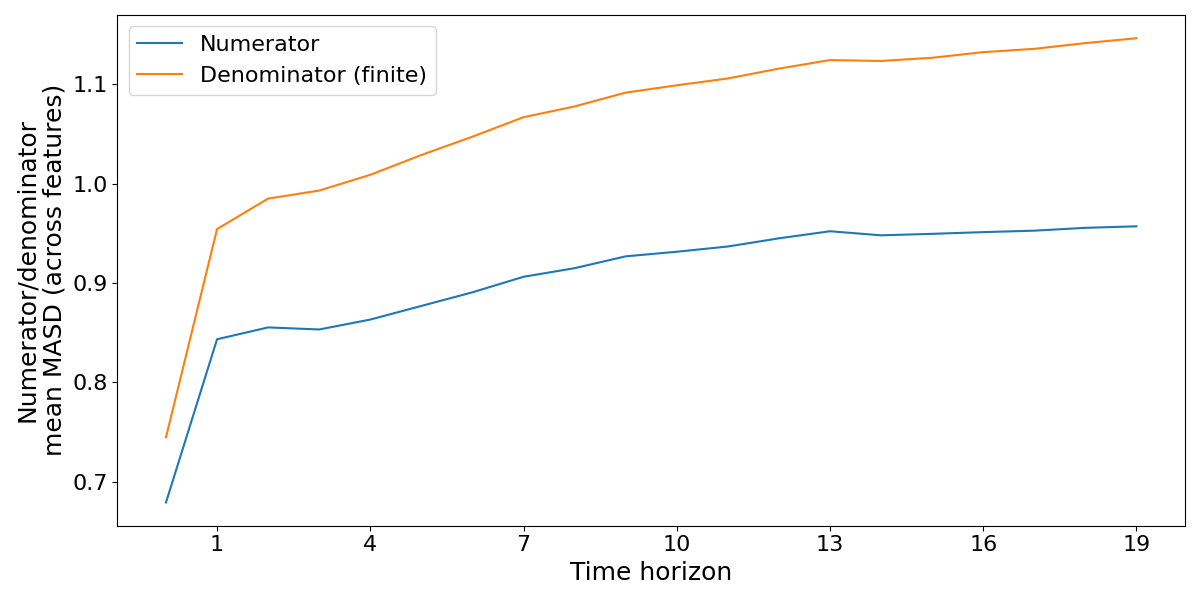}
		\caption{Testing dataset}
	\end{subfigure}
	\caption[The figure describes the median numerator and denominator of the MASD metric, across features on the sepsis dataset, under the Hajek estimator. The x-axis displays the time horizon. Aggregations are taken over action trajectory pairs with finite denominators.]{The figure describes the median numerator and denominator of the $\masdtw{}$ metric, across features on the sepsis dataset, under the Hajek estimator. The x-axis displays the time horizon. Aggregations are taken over action trajectory pairs with finite denominators.}
	\label{fig:mimic_sepsis_xgboost_4_w_cov_bal_by_time_numer_finite_denom_hajek}
\end{figure}
 
\FloatBarrier

Figures \ref{fig:xgboost_4_w_masd_avg_across_feat}, \ref{fig:xgboost_4_w_masd_avg_across_feat_clip} and \ref{fig:xgboost_4_w_masd_avg_across_feat_hajek} describe the median (across features) $\masdtdiffmean{}$ values on the sepsis dataset. The blue describes the range of values where covariate balance is improved under the propensity model and orange describes the opposite. For the vanilla estimator, the result supports the deduction that covariate balance degrades with horizon. In particular, after timepoint 3 on the training set and timepoint 2 on the testing set, the median $\masdtdiffmean{}$ across features is not greater than 0. Figure \ref{fig:xgboost_4_w_masd_avg_across_feat_clip}, displaying balance under the clipped estimator, suggests the propensity model improves covariate balance for the entire trajectory. This is similarly the case for the Hajek estimator, presented in figure \ref{fig:xgboost_4_w_masd_avg_across_feat_hajek}. A conclusion regarding the validity of this claim is reserved until the bias of the estimator is explored, as already alluded to.\\

\begin{figure}[!h]
	\centering
	\begin{subfigure}{0.49\linewidth}
		\includegraphics[width=\linewidth]{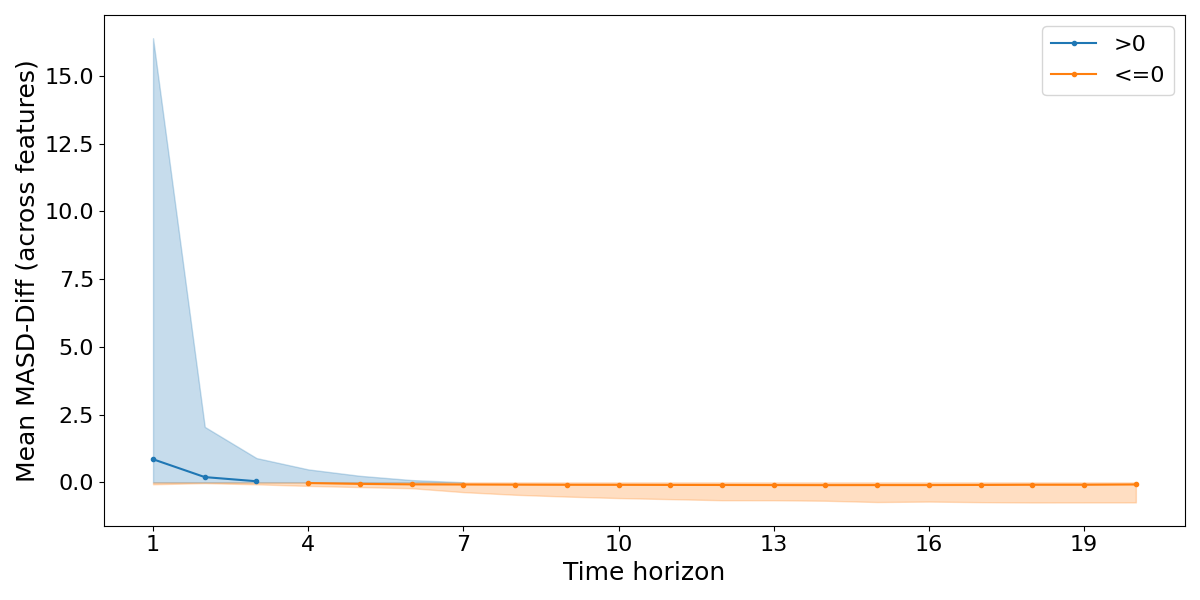}
		\caption{Training dataset}
	\end{subfigure}
	\begin{subfigure}{0.49\linewidth}
		\includegraphics[width=\linewidth]{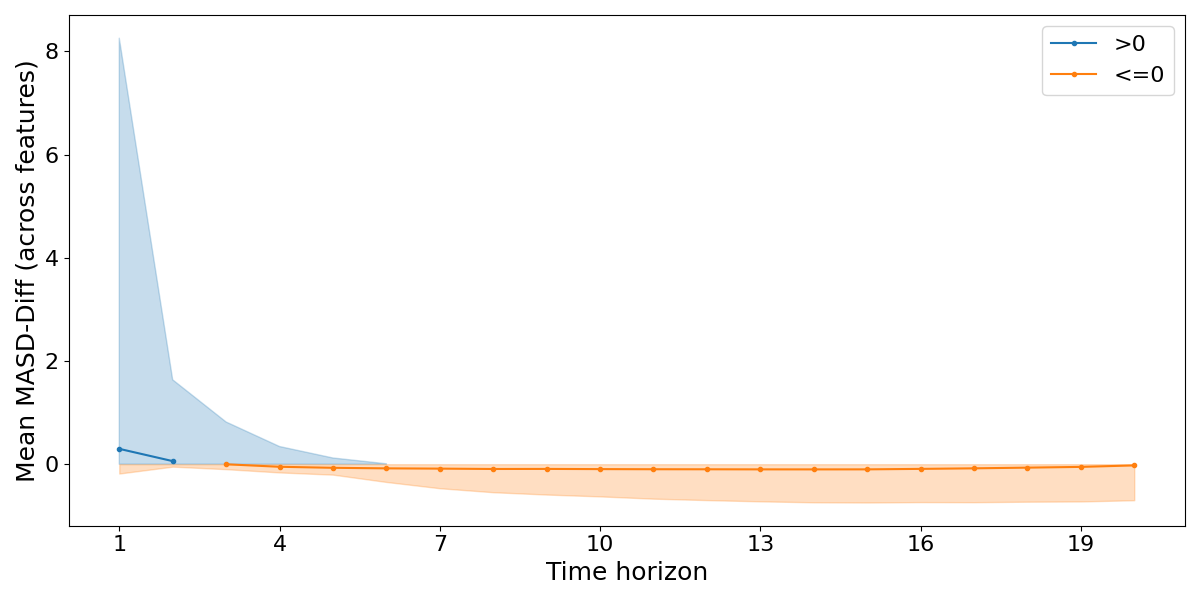}
		\caption{Testing dataset}
	\end{subfigure}
	\caption[The figure describes the median (across features) mean MASD-Diff values on the sepsis dataset, under the vanilla estimator. The shaded region describes the min and max mean MASD-Diff values by timepoint.]{The figure describes the median (across features) $\masdtdiffmean{}$ values on the sepsis dataset, under the vanilla estimator. The shaded region describes the min and max $\masdtdiffmean{}$ values by timepoint.}
	\label{fig:xgboost_4_w_masd_avg_across_feat}
\end{figure}

\begin{figure}[!h]
	\centering
	\begin{subfigure}{0.49\linewidth}
		\includegraphics[width=\linewidth]{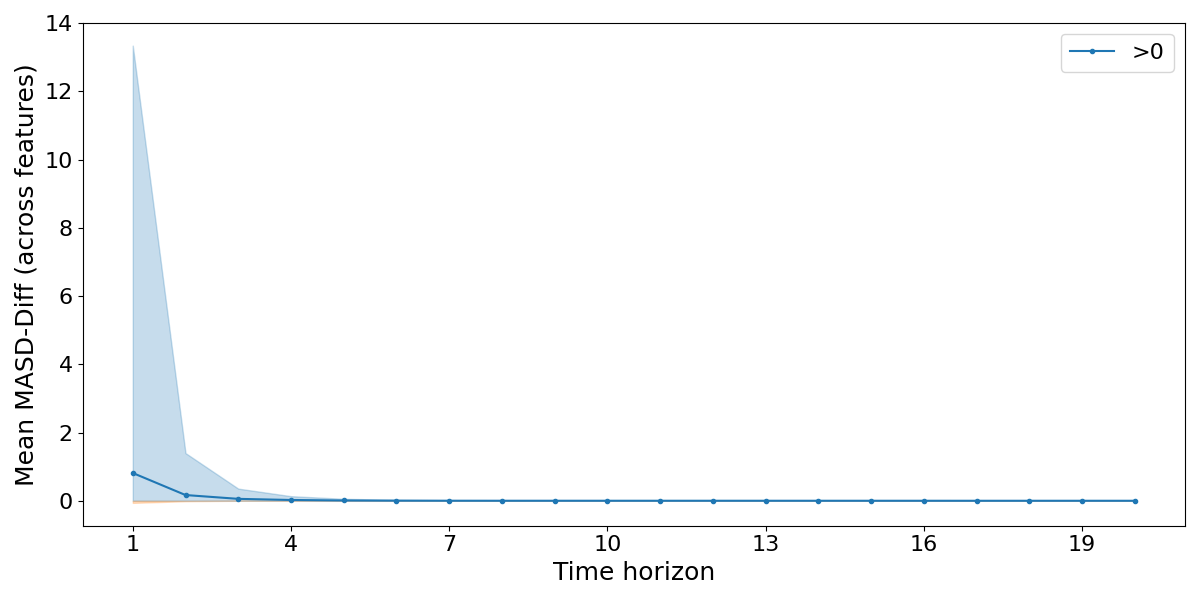}
		\caption{Training dataset}
	\end{subfigure}
	\begin{subfigure}{0.49\linewidth}
		\includegraphics[width=\linewidth]{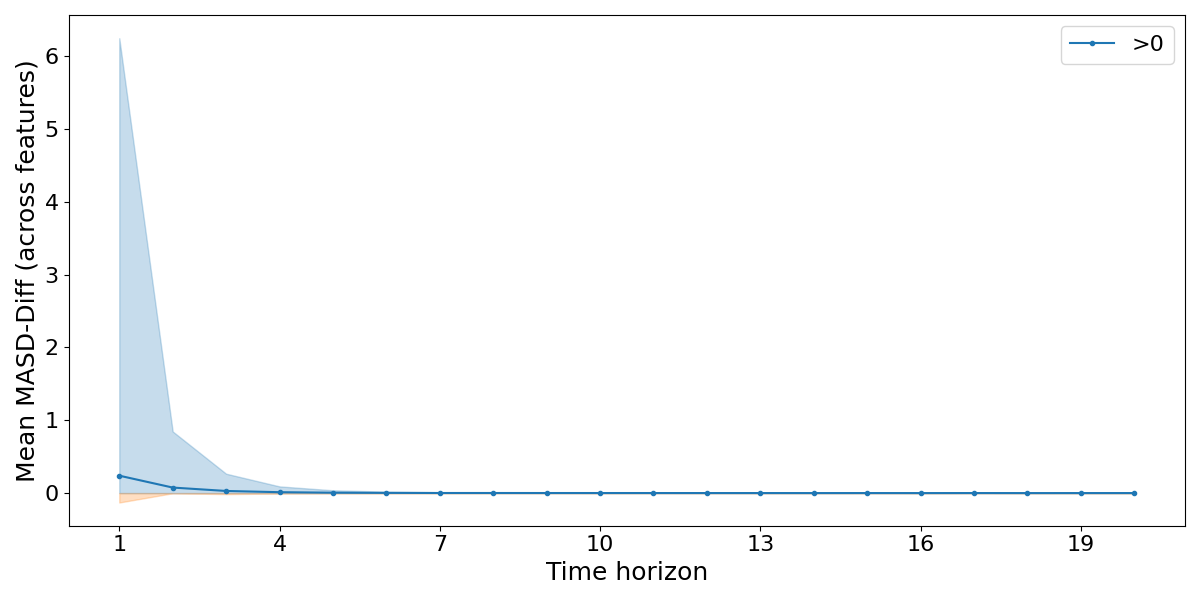}
		\caption{Testing dataset}
	\end{subfigure}
	\caption[The figure describes the median (across features) mean MASD-Diff values on the sepsis dataset, under the clipped estimator. The shaded region describes the min and max mean MASD-Diff values by timepoint.]{The figure describes the median (across features) $\masdtdiffmean{}$ values on the sepsis dataset, under the clipped estimator. The shaded region describes the min and max $\masdtdiffmean{}$ values by timepoint.}
	\label{fig:xgboost_4_w_masd_avg_across_feat_clip}
\end{figure}

\begin{figure}[!h]
	\centering
	\begin{subfigure}{0.49\linewidth}
		\includegraphics[width=\linewidth]{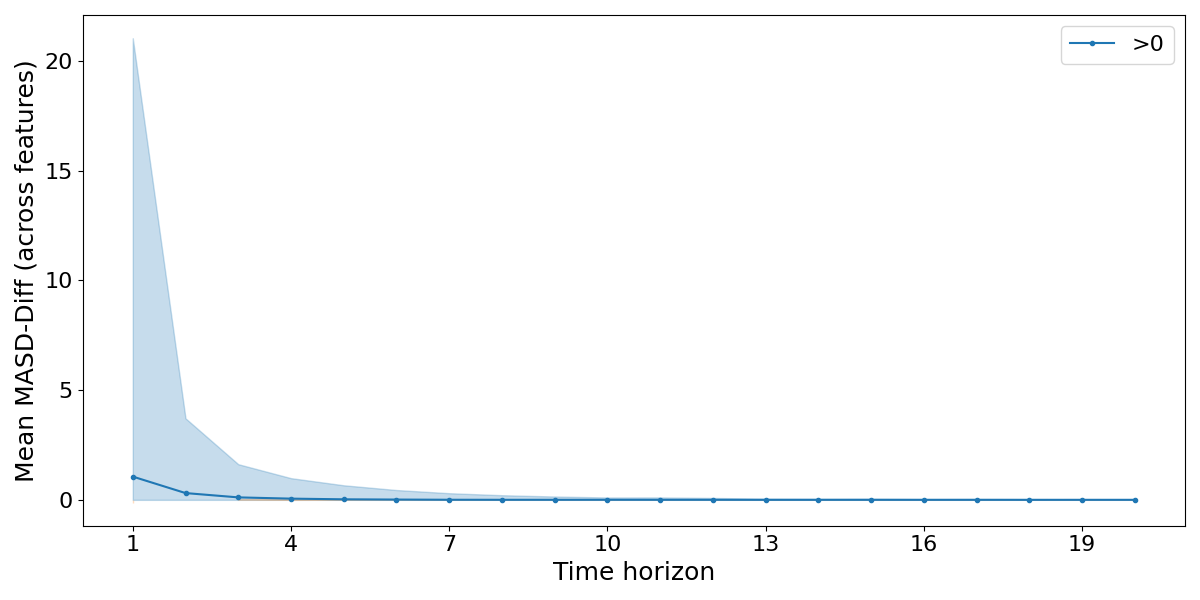}
		\caption{Training dataset}
	\end{subfigure}
	\begin{subfigure}{0.49\linewidth}
		\includegraphics[width=\linewidth]{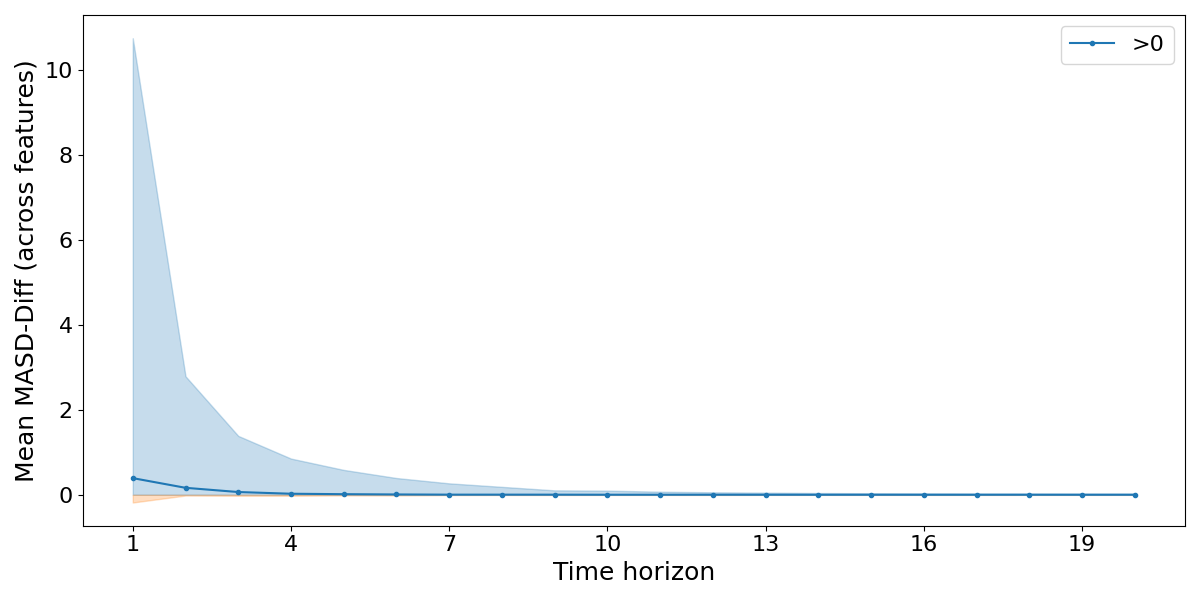}
		\caption{Testing dataset}
	\end{subfigure}
	\caption[The figure describes the median (across features) mean MASD-Diff values on the sepsis dataset, under the Hajek estimator. The shaded region describes the min and max mean MASD-Diff values by timepoint.]{The figure describes the median (across features) $\masdtdiffmean{}$ values on the sepsis dataset, under the Hajek estimator. The shaded region describes the min and max $\masdtdiffmean{}$ values by timepoint.}
	\label{fig:xgboost_4_w_masd_avg_across_feat_hajek}
\end{figure}

\FloatBarrier

\paragraph{Bias investigation for clipped and Hajek importance sampling}\label{sec:bias_invest_clip_hajek}
Figure \ref{fig:xgboost_4_w_avg_state_times_weight_scatter} displays scatter plots of the average (across intervention level) original covariate on the x-axis (i.e., $g(s_{i,0:t,k})$) against the weighted covariate on the y-axis (i.e., $g(s_{i,0:t,k})(\prod_{t'=0}^{t}p(a_{i,t'}|s_{i,t'}))^{-1}$ from equation \ref{equ:masd_weight_t}). The fewer points observed at timestep 1 is a function of the fact that the number of potential interventions increases as $25^{t+1}$. Across the timesteps, a non-zero amount of variance is retained by the map from the unweighted to the weighted covariate (i.e., between the x and y axes).\\

Figures \ref{fig:xgboost_4_w_avg_state_times_clipped_weight_scatter} and \ref{fig:xgboost_4_w_avg_state_times_hajek_weight_scatter} display the same information as figure \ref{fig:xgboost_4_w_avg_state_times_weight_scatter} but for the clipped and Hajek estimators, respectively. Visually, in comparison to the vanilla estimator, the variance in the importance weighted projection of the covariates vanishes at longer time horizons. Let $h(s_{i,t,k})$ define the projected covariate under an importance weighting regimen i.e.:
\begin{align*}
	h_{\textrm{Vanilla}}(s_{i,t,k},a_{0:t}) =& \frac{g(s_{i,0:t,k})}{\prod_{t'=0}^{t}p(a_{i,t'}|s_{i,t'})}\\
	h_{\textrm{Clipped}}(s_{i,t,k},a_{0:t}) =& \frac{g(s_{i,0:t,k})}{\zeta \vee \prod_{t'=0}^{t}p(a_{i,t'}|s_{i,t'})}\\
	h_{\textrm{Hajek}}(s_{i,t,k},a_{0:t}) =& \frac{g(s_{i,0:t,k})}{\prod_{t'=0}^{t}p(a_{i,t'}|s_{i,t'})(\sum \prod_{t'=0}^{t}p(a_{i,t'}|s_{i,t'}))^{-1}}
\end{align*}
The lack of variance for a given value of $g(s_{i,0:t,k})$ under the $h_{\textrm{Clipped}}$ and $h_{\textrm{Hajek}}$ projections (demonstrated by figures \ref{fig:xgboost_4_w_avg_state_times_clipped_weight_scatter} and \ref{fig:xgboost_4_w_avg_state_times_hajek_weight_scatter}) implies $h_{\textrm{Clipped}}(s_{i,t,k},a_{0:t}) = h_{\textrm{Clipped}}(s_{i,t,k},a_{0:t}')$ and $h_{\textrm{Hajek}}(s_{i,t,k},a_{0:t}) = h_{\textrm{Hajek}}(s_{i,t,k},a_{0:t}')$ i.e., the bias introduced by the clipped and Hajek estimators results in the covariate projections collapsing to a single value for a given covariate value, $g(s_{i,0:t,k})$ and for all intervention levels, $a_{0:t}$. The implication is that, under high variance regimens (e.g., long time horizons), the covariate weight under the clipped and Hajek estimators will be equal for all actions trajectories, giving the impression of balance that is equivalent to the unweighted state i.e., does not degrade but does not improve. Recall, figures \ref{fig:xgboost_4_w_masd_avg_across_feat_clip} and \ref{fig:xgboost_4_w_masd_avg_across_feat_hajek} demonstrated that the $\masdtdiffmean{}$ values converged to 0 with time and thus maybe a function of the afore described phenomena.\\

\begin{figure}[htbp]
	\centering
	\begin{subfigure}{\linewidth}
		\includegraphics[width=\linewidth]{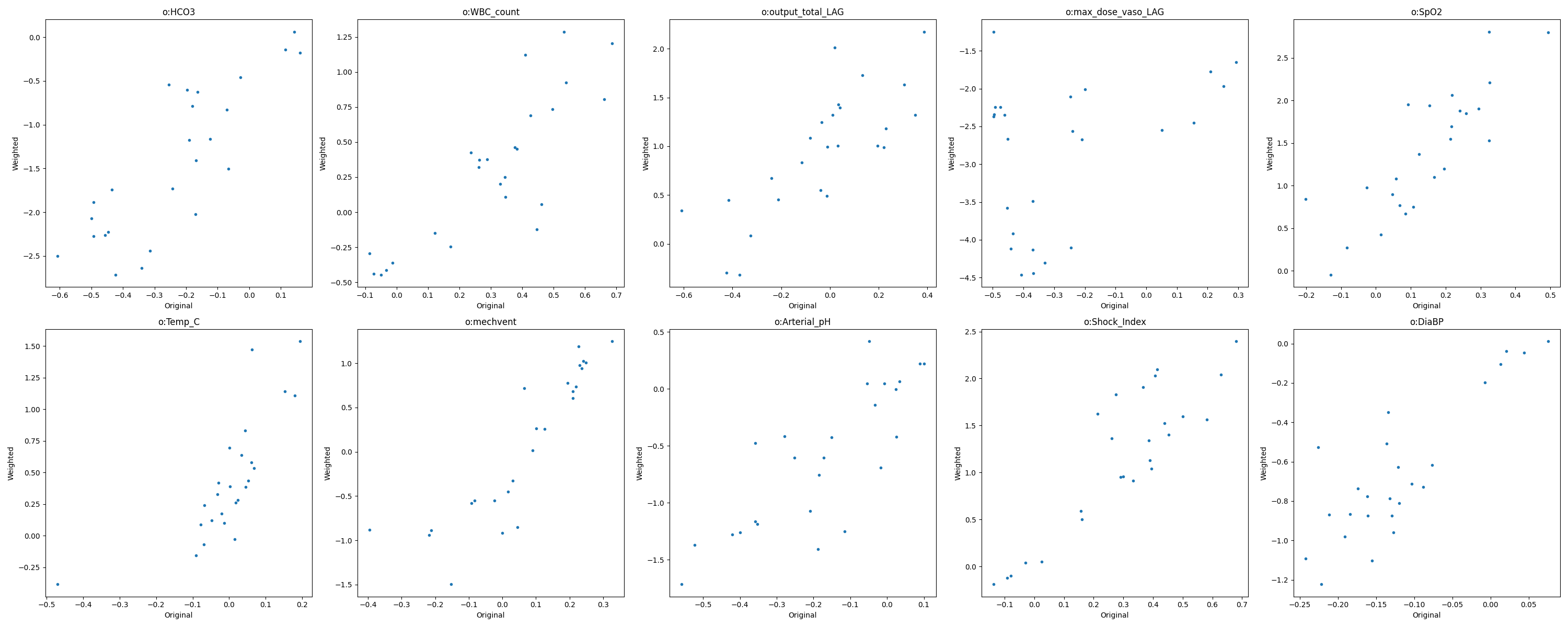}
		\caption{Timestep 1}
	\end{subfigure}
	\begin{subfigure}{\linewidth}
		\includegraphics[width=\linewidth]{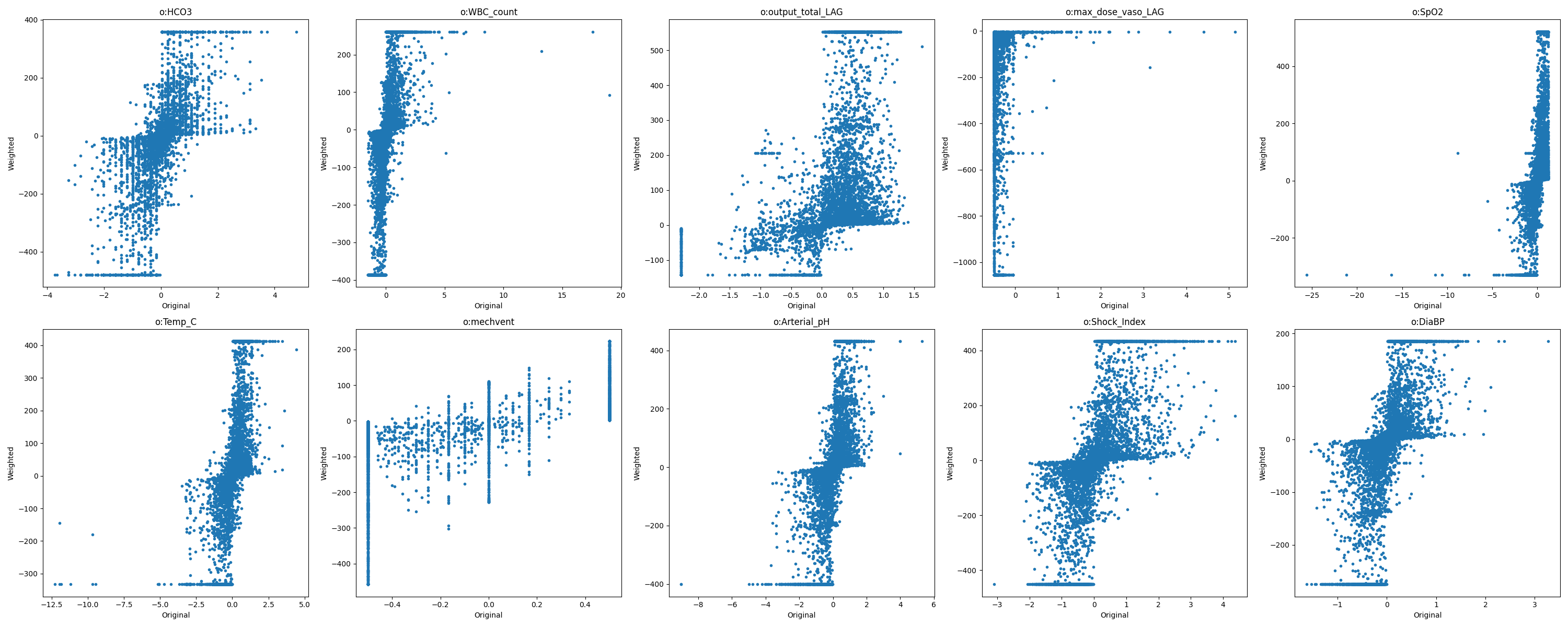}
		\caption{Timestep 4}
	\end{subfigure}
	\begin{subfigure}{\linewidth}
		\includegraphics[width=\linewidth]{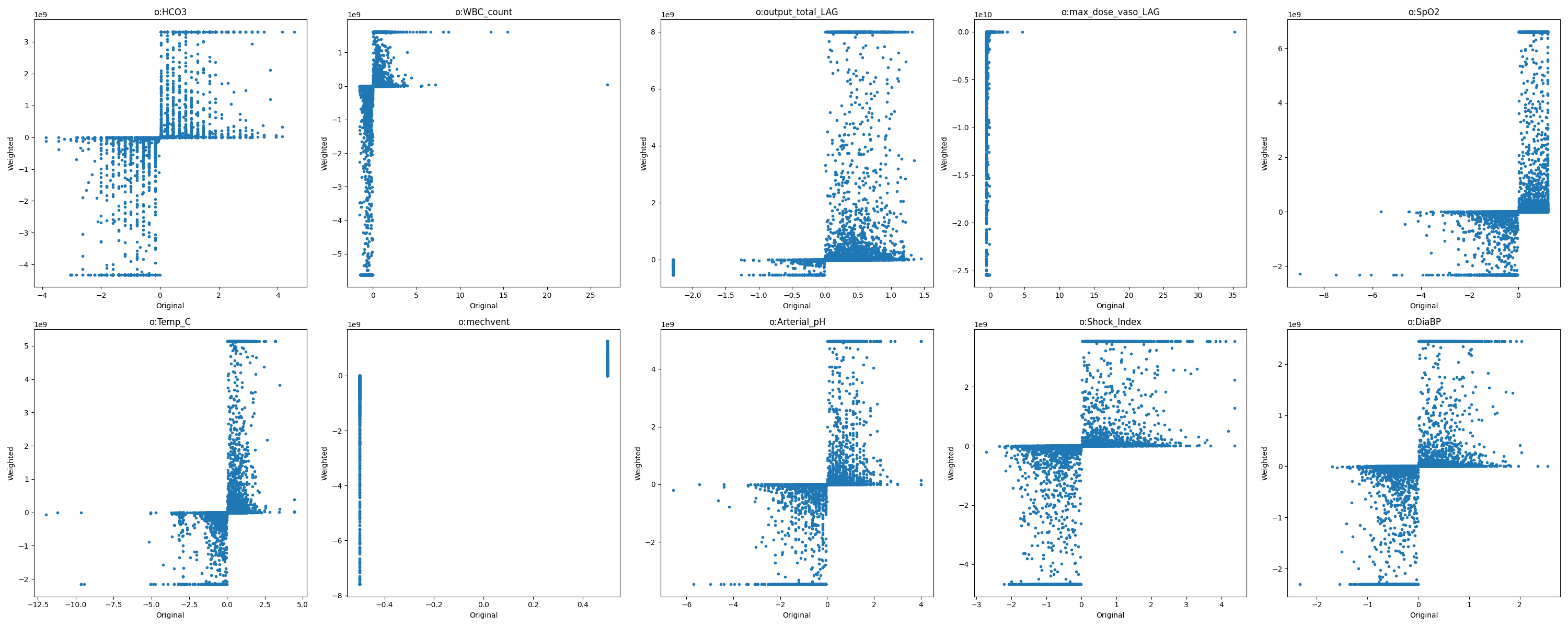}
		\caption{Timestep 16}
	\end{subfigure}
	\caption[The figure displays scatter plots of the average (across intervention level) covariate value on the x-axis against the weighted covariate on the y-axis. The weighted covariate values are clipped between the 10th and 90th percentile of the original distribution. The figure displays timestep 1 at the top, timestep 4 in the middle and timestep 16 at the bottom, with propensity score predictions made over seed 0 the training dataset.]{The figure displays scatter plots of the average (across intervention level) covariate value on the x-axis against the weighted covariate (i.e., $g(s_{i,0:t,k})(\prod_{t'=0}^{t}p(a_{i,t'}|s_{i,t'}))^{-1}$ from equation \ref{equ:masd_weight_t}) on the y-axis. The weighted covariate values are clipped between the 10th and 90th percentile of the original distribution. The figure displays timestep 1 at the top, timestep 4 in the middle and timestep 16 at the bottom, with propensity score predictions made over seed 0 of the training dataset.}
	\label{fig:xgboost_4_w_avg_state_times_weight_scatter}
\end{figure}

\begin{figure}[htbp]
	\centering
	\begin{subfigure}{\linewidth}
		\includegraphics[width=\linewidth]{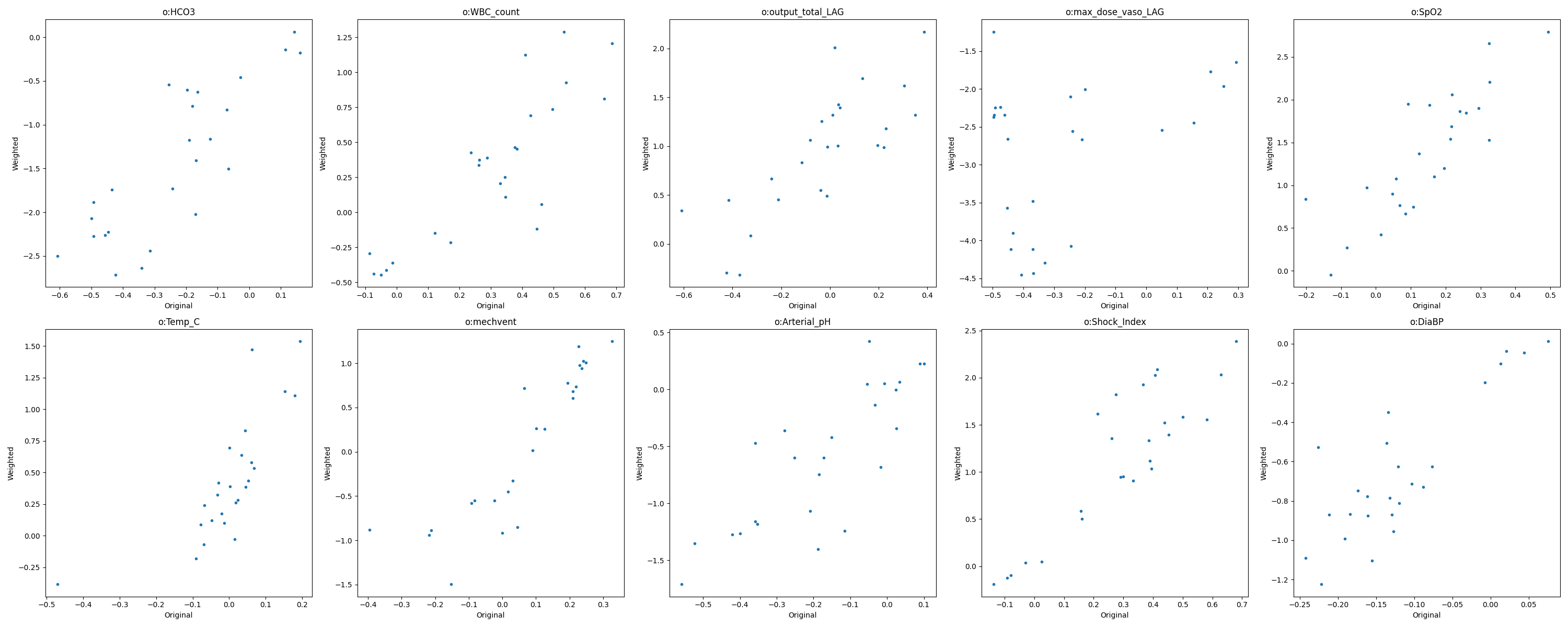}
		\caption{Timestep 1}
	\end{subfigure}
	\begin{subfigure}{\linewidth}
		\includegraphics[width=\linewidth]{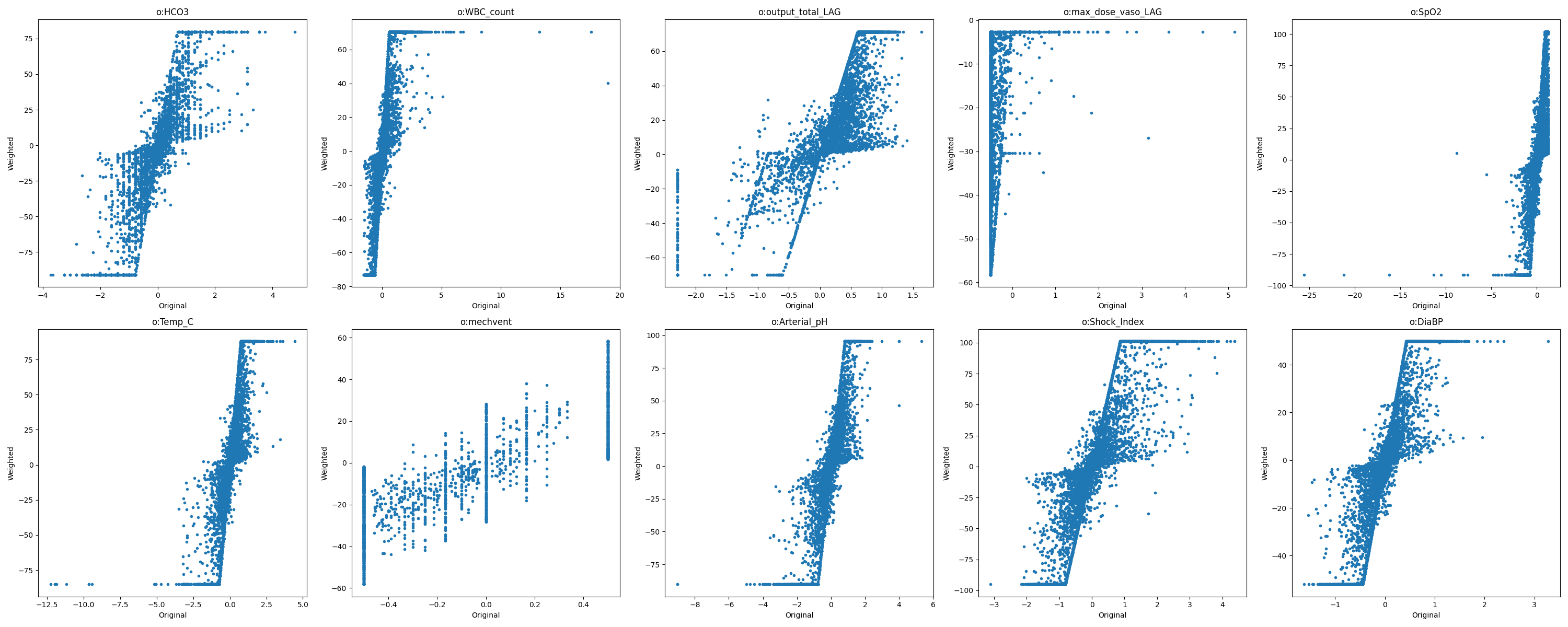}
		\caption{Timestep 4}
	\end{subfigure}
	\begin{subfigure}{\linewidth}
		\includegraphics[width=\linewidth]{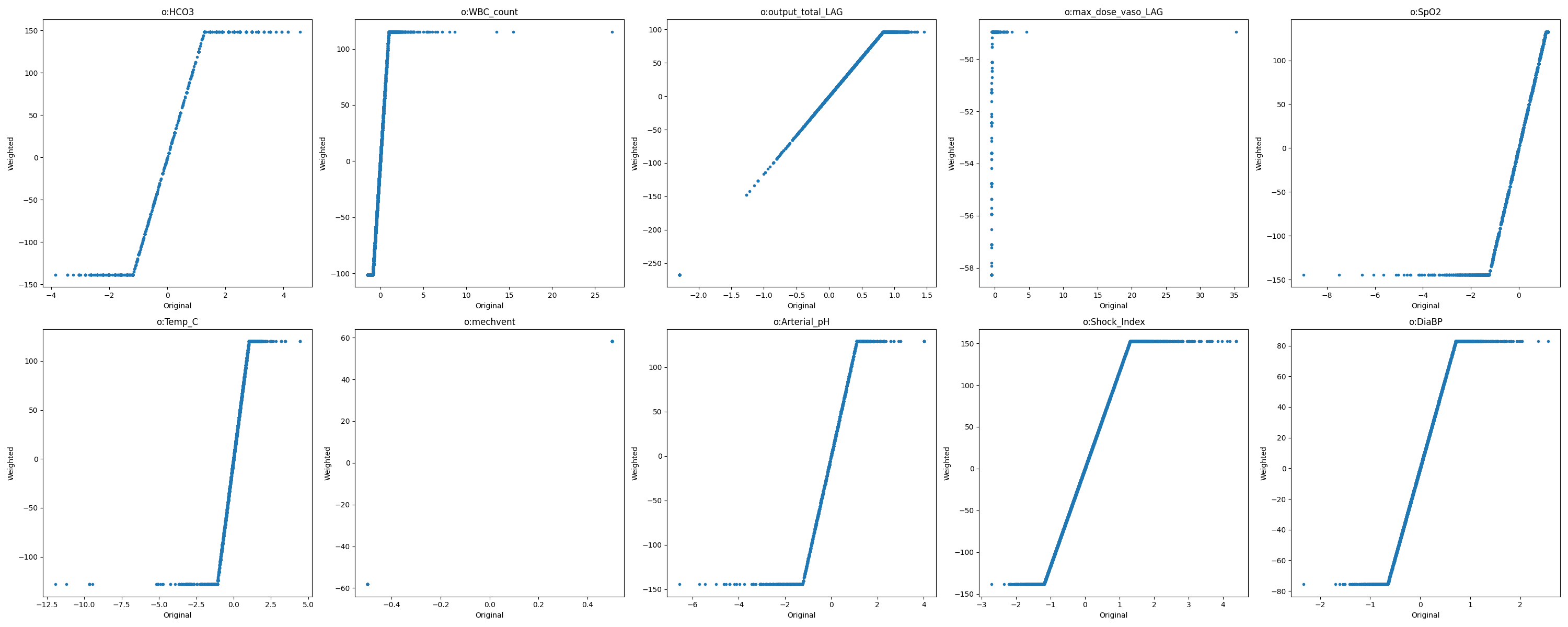}
		\caption{Timestep 16}
	\end{subfigure}
	\caption[The figure displays scatter plots of the average (across intervention level) covariate value on the x-axis against the clipped importance weighted covariate on the y-axis. The weighted covariate values are further clipped between the 10th and 90th percentile of the original distribution. The figure displays timestep 1 at the top, timestep 4 in the middle and timestep 16 at the bottom, with propensity score predictions made over seed 0 of the training dataset.]{The figure displays scatter plots of the average (across intervention level) covariate value on the x-axis against the clipped importance weighted covariate (i.e., $g(s_{i,0:t,k})(\prod_{t'=0}^{t}p(a_{i,t'}|s_{i,t'}))^{-1}$ from equation \ref{equ:masd_weight_t}) on the y-axis. The weighted covariate values are further clipped between the 10th and 90th percentile of the original distribution. The figure displays timestep 1 at the top, timestep 4 in the middle and timestep 16 at the bottom, with propensity score predictions made over seed 0 of the training dataset.}
	\label{fig:xgboost_4_w_avg_state_times_clipped_weight_scatter}
\end{figure}

\begin{figure}[htbp]
	\centering
	\begin{subfigure}{\linewidth}
		\includegraphics[width=\linewidth]{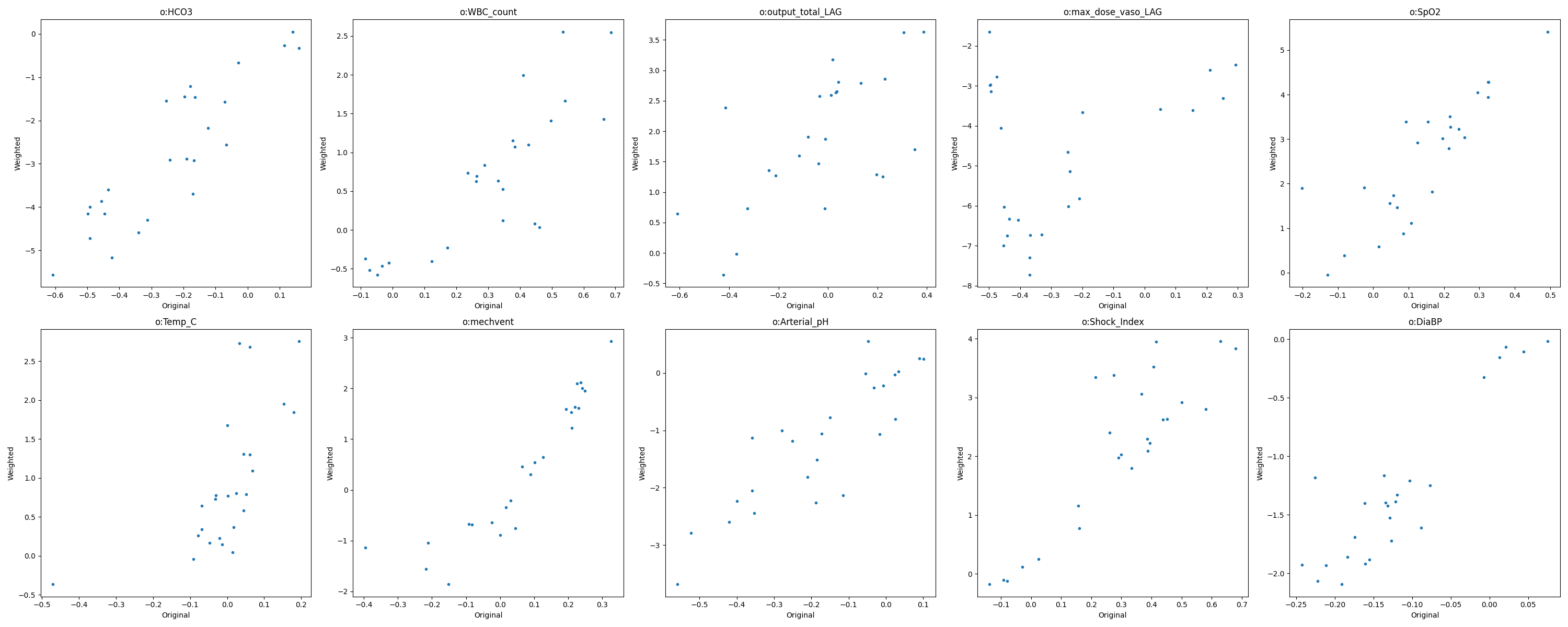}
		\caption{Timestep 1}
	\end{subfigure}
	\begin{subfigure}{\linewidth}
		\includegraphics[width=\linewidth]{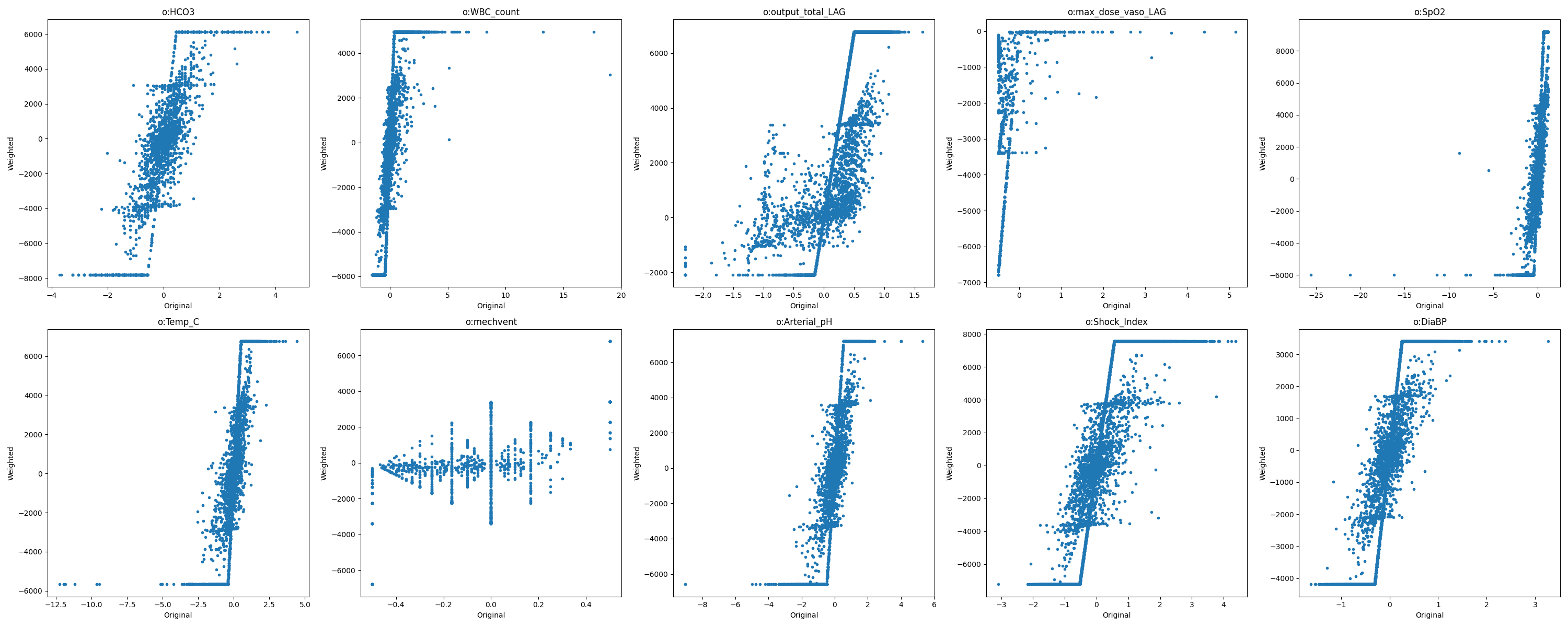}
		\caption{Timestep 4}
	\end{subfigure}
	\begin{subfigure}{\linewidth}
		\includegraphics[width=\linewidth]{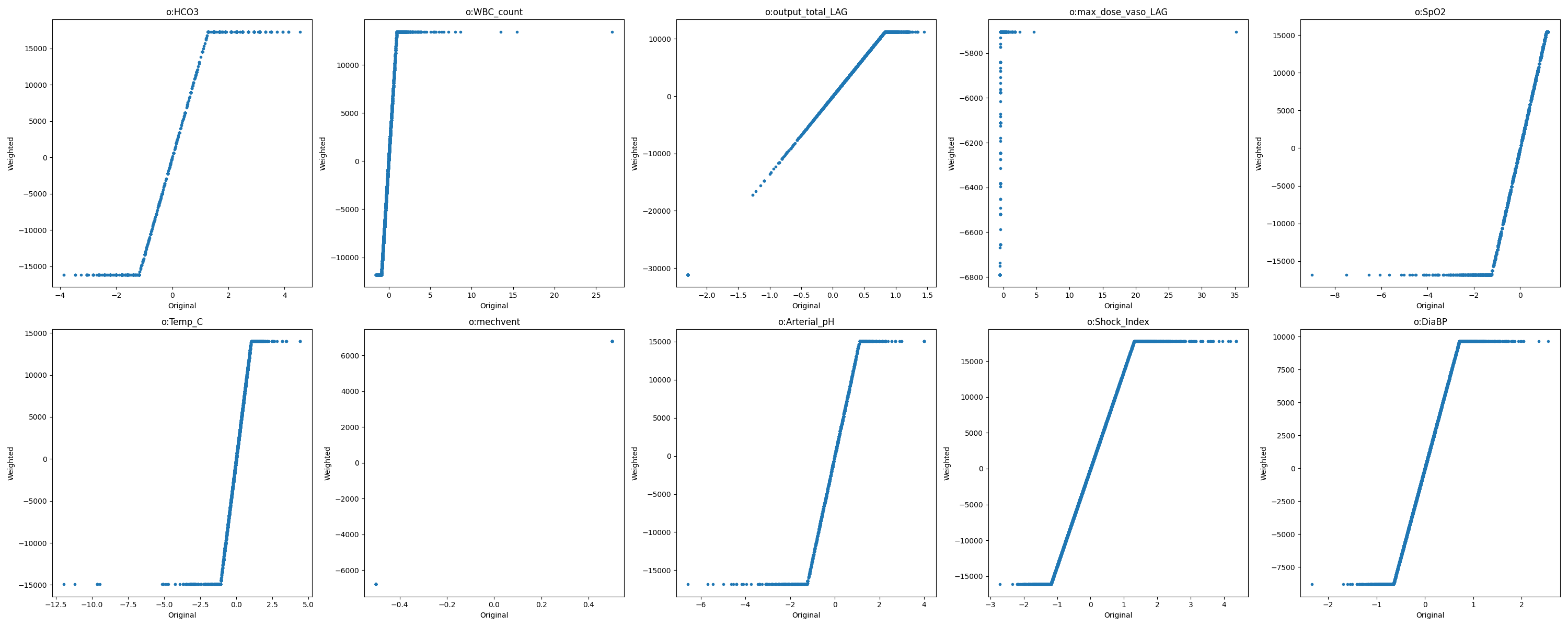}
		\caption{Timestep 16}
	\end{subfigure}
	\caption[The figure displays scatter plots of the average (across intervention level) covariate value on the x-axis against the Hajek importance weighted covariate on the y-axis. The weighted covariate values are clipped between the 10th and 90th percentile of the original distribution. The figure displays timestep 1 at the top, timestep 4 in the middle and timestep 16 at the bottom, with propensity score predictions made over seed 0 of the training dataset.]{The figure displays scatter plots of the average (across intervention level) covariate value on the x-axis against the Hajek importance weighted covariate (i.e., $g(s_{i,0:t,k})(\prod_{t'=0}^{t}p(a_{i,t'}|s_{i,t'}))^{-1}$ from equation \ref{equ:masd_weight_t}) on the y-axis. The weighted covariate values are clipped between the 10th and 90th percentile of the original distribution. The figure displays timestep 1 at the top, timestep 4 in the middle and timestep 16 at the bottom, with propensity score predictions made over seed 0 of the training dataset.}
	\label{fig:xgboost_4_w_avg_state_times_hajek_weight_scatter}
\end{figure}

\paragraph{Bias analysis of clipped importance sampling} Let: 
\begin{align*}
	d=\{\tau_{i}\}_{i=1}^{n}=\{s_{i,0},a_{i,0},...,s_{i,H-2},a_{i,H-2},s_{i,H-1}\}_{i=1}^{n}
\end{align*}
with density defined by equation \ref{equ:mdp_density}. For a sufficiently large time horizon, $t$, since $p(a_{i,t'}|s_{i,t'})$ is a normalised probability density, observe that $\forall \tau\in d, \prod_{t'=0}^{H-2}p(a_{i,t'}|s_{i,t'}) \vee \zeta = \zeta$. Then:
\begin{align*}
	\hat{\mu}_{a_{0:t}'}(k) = & \frac{1}{n}\sum_{i=1}^{n}\frac{\mathbbm{1}(a_{i,0:t}=a_{0:t}')g(s_{i,0:t,k})}{\zeta} \\
	=& \zeta^{-1}\Bigg(\frac{1}{n}\sum_{i=1}^{n}\mathbbm{1}(a_{i,0:t}=a_{0:t}')g(s_{i,0:t,k})\Bigg)
\end{align*}
Additionally, for simplicity, considering the biased estimate of $\sigma^{2}_{a_{0:t}'}(k)$:
\begin{align*}
	\bar{\sigma}^{2}_{a_{0:t}'}(k) =& \frac{1}{n}\sum_{i=1}^{n}\Bigg(\frac{\mathbbm{1}(a_{i,0:t}=a_{0:t}')g(s_{i,0:t,k})}{\prod_{t'=0}^{t}p(a_{i,t'}|s_{i,t'})} - \hat{\mu}_{a_{0:t}'}\Bigg)^{2} \\
	=& \frac{1}{n}\sum_{i=1}^{n}\Bigg(\frac{\mathbbm{1}(a_{i,0:t}=a_{0:t}')g(s_{i,0:t,k})}{\prod_{t'=0}^{t}p(a_{i,t'}|s_{i,t'})}\Bigg)^{2}\\
	& - \Bigg(\frac{1}{n}\sum_{i=1}^{n}\frac{\mathbbm{1}(a_{i,0:t}=a_{0:t}')g(s_{i,0:t,k})}{\prod_{t'=0}^{t}p(a_{i,t'}|s_{i,t'})}\Bigg)^{2} \\
	=& \frac{1}{n}\sum_{i=1}^{n}\Bigg(\frac{\mathbbm{1}(a_{i,0:t}=a_{0:t}')g(s_{i,0:t,k})}{\zeta}\Bigg)^{2}\\
	& - \Bigg(\frac{1}{n}\sum_{i=1}^{n}\frac{\mathbbm{1}(a_{i,0:t}=a_{0:t}')g(s_{i,0:t,k})}{\zeta}\Bigg)^{2} \\
	=& \zeta^{-2}\Bigg(\frac{1}{n}\sum_{i=1}^{n}\bigg(\mathbbm{1}(a_{i,0:t}=a_{0:t}')g(s_{i,0:t,k})\bigg)^{2}\\
	&- \Bigg(\frac{1}{n}\sum_{i=1}^{n}\mathbbm{1}(a_{i,0:t}=a_{0:t}')g(s_{i,0:t,k})\Bigg)^{2}\Bigg)
\end{align*}

Combining the numerator and denominator, observe that:
\begin{align*}
	\frac{|\hat{\mu}_{a_{0:t}'}(k)-\hat{\mu}_{a_{0:t}''}(k)|}{\sqrt{n^{-1}\bar{\sigma}^{2}_{a_{0:t}'}(k) + n^{-1}\bar{\sigma}^{2}_{a_{0:t}''}(k)}} = 
	\frac{\zeta^{-1}|\hat{\mu\prime}_{a_{0:t}'}(k)-\hat{\mu\prime}_{a_{0:t}''}(k)|}{\sqrt{\zeta^{-2}(n^{-1}\bar{\sigma\prime}^{2}_{a_{0:t}'}(k) + n^{-1}\bar{\sigma\prime}^{2}_{a_{0:t}''}(k))}} \\
\end{align*}
where $\hat{\mu\prime}_{a_{0:t}'}(k)$ and $\bar{\sigma\prime}^{2}_{a_{0:t}'}(k)$ are the unweighted mean and standard deviation. Which implies, for a sufficiently large $t$:
\begin{equation}\label{equ:masd_propto_large_t}
	\masdtw{} \approx \textrm{MASD}_{\textrm{UnWeight-}t}(k,a_{0:t},a_{0:t}'') 
\end{equation}
Equation \ref{equ:masd_propto_large_t} explains the apparent positive covariate balance observed when considering the clipped importance sampling estimator, demonstrated in figures \ref{fig:mimic_sepsis_xgboost_4_w_post_masd_avg_across_feat_clip} and \ref{fig:xgboost_4_w_masd_avg_across_feat_clip}.\\

\paragraph{Bias analysis of Hajek importance sampling} A similar deduction regarding the bias of the Hajek estimator can be made. However, unlike the clipping estimator, the resulting bias is driven by the exponential reduction in effective sample size. This is demonstrated in figure \ref{fig:obs_per_int_level_by_time}, displaying the number of observations per intervention level, split by time for the training and testing datasets.\\

\begin{figure}[!h]
	\centering
	\begin{subfigure}{0.49\linewidth}
		\includegraphics[width=\linewidth]{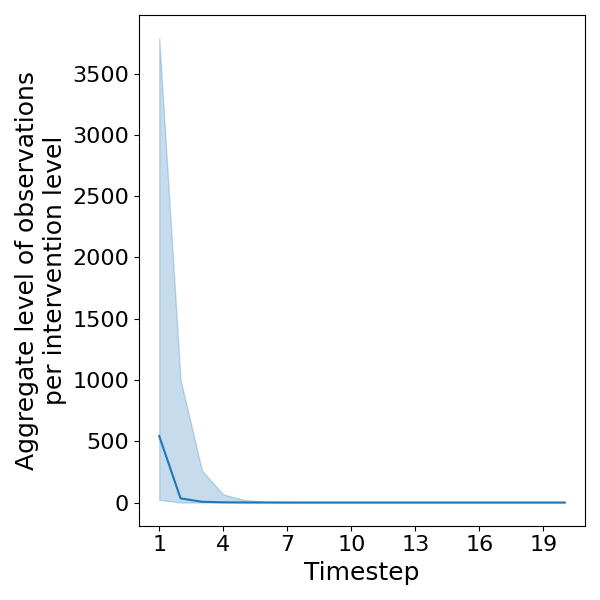}	
	\caption{Training set}
	\end{subfigure}
	\begin{subfigure}{0.49\linewidth}
		\includegraphics[width=\linewidth]{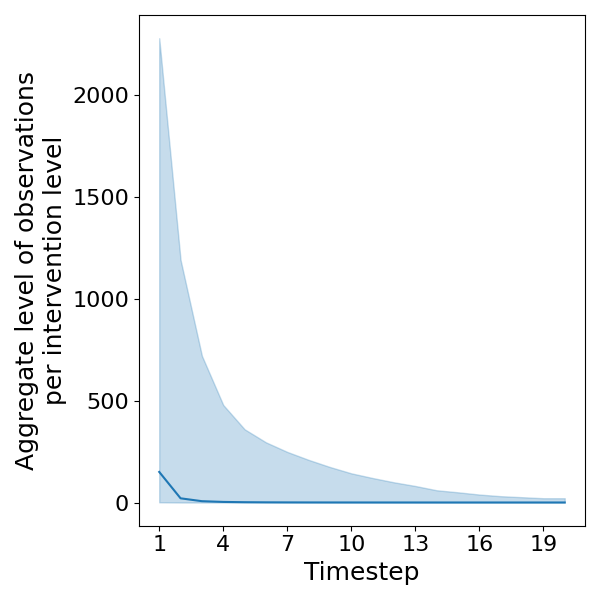}
	\caption{Testing set}
	\end{subfigure}
	\caption{The figure displays the aggregate number of observations per intervention level, across time. The thick blue line defines the mean number of observations whilst the shaded area defines and min and max.}
	\label{fig:obs_per_int_level_by_time}
\end{figure}

Returning to the Hajek estimator, $\masdthw{}$, for a sufficiently large $t$:
\begin{align*}
	\frac{\frac{\mathbbm{1}(a_{i,0:t}=a_{0:t}')g(s_{i,0:t,k})}{\prod_{t'=0}^{t}p(a_{i,t'}|s_{i,t'})}}{\frac{1}{n}\sum_{i=1}^{n}\frac{\mathbbm{1}(a_{i,0:t}=a_{0:t}')}{\prod_{t'=0}^{t}p(a_{i,t'}|s_{i,t'})}} = \mathbbm{1}(a_{i,0:t}=a_{0:t}')g(s_{i,0:t,k})\\
\end{align*}
As such:
\begin{align*}
	\frac{\frac{1}{n}\sum_{i=1}^{n}\frac{\mathbbm{1}(a_{i,0:t}=a_{0:t}')g(s_{i,0:t,k})}{\prod_{t'=0}^{t}p(a_{i,t'}|s_{i,t'})}}{\frac{1}{n}\sum_{i=1}^{n}\frac{\mathbbm{1}(a_{i,0:t}=a_{0:t}')}{\prod_{t'=0}^{t}p(a_{i,t'}|s_{i,t'})}} = \frac{1}{n}\sum_{i=1}^{n}\mathbbm{1}(a_{i,0:t}=a_{0:t}')g(s_{i,0:t,k})
\end{align*}
and:
\begin{align*}
	& \frac{1}{n-1}\sum_{i=1}^{n}\Bigg(\frac{\frac{\mathbbm{1}(a_{i,0:t}=a_{0:t}')g(s_{i,0:t,k})}{\prod_{t'=0}^{t}p(a_{i,t'}|s_{i,t'})}}{\frac{1}{n}\sum_{i=1}^{n}\frac{\mathbbm{1}(a_{i,0:t}=a_{0:t}')}{\prod_{t'=0}^{t}p(a_{i,t'}|s_{i,t'})}} - \hat{\mu}_{a_{0:t}'}\Bigg)^{2} \\
	& = \frac{1}{n-1}\sum_{i=1}^{n}\Bigg(\mathbbm{1}(a_{i,0:t}=a_{0:t}')g(s_{i,0:t,k}) - \frac{1}{n}\sum_{i=1}^{n}\mathbbm{1}(a_{i,0:t}=a_{0:t}')g(s_{i,0:t,k})\Bigg)^{2}
\end{align*}
Thus, similarly to the clipped estimator, for a sufficiently large $t$:
\begin{align}
	\masdthw{} \approx \textrm{MASD}_{\textrm{UnWeight-}t}(k,a_{0:t},a_{0:t}'') 
\end{align}


\FloatBarrier

\paragraph{Summary} To conclude, the high variance of the importance sampling estimator precludes drawing a conclusion regarding the presence of covariate balance at later timesteps. Furthermore, controlling this variance using the conventional clipping and weighting arguably worsens the issue since both methods bias towards perfect covariate balance in high variance settings. More encouragingly, both the clipped and Hajek estimators did suggest covariate balance was present beyond the first three timesteps where, from figures \ref{fig:xgboost_4_w_avg_state_times_clipped_weight_scatter} and \ref{fig:xgboost_4_w_avg_state_times_hajek_weight_scatter}, there is still variation in the weighted states. This argument however, is entirely heuristic and a formal framework describing the trade-off between accuracy and variance reduction would be extremely useful. Pessimistic PAC based estimators for clipped importance sampling have been widely established (\cite{sakhi2024logarithmic}) and thus would provide a solid starting point for developing such a framework. An immediate and potentially significant challenge in adapting these estimators however, is the desirability of pessimism in the context of covariate balance.\\

\FloatBarrier

\subsubsection{Fluid action space (Sepsis-fluid dataset)}\label{sec:fluid_action_space}

Under the full action space used in the previous section, a number of actions were drastically underrepresented even before accounting for the cumulative effect over time. Figures \ref{fig:prop_actions_ovr_time_sepsis} and \ref{fig:prop_actions_ovr_time_fluid} describe the proportion of observations split by action over time. At a high level, for the sepsis dataset, the minimum representation for a given action was only 3\% of the total observations (calculated using the seed 0 training set). It was hypothesised that this lead to higher variance in the covariate balance metrics due to the reduced effective sample size, as demonstrated by figure \ref{fig:obs_per_int_level_by_time_seed_0_sepsis}. This is in contrast to figure \ref{fig:obs_per_int_level_by_time_seed_0_fluid} where the effective sample size of a given intervention reduces to 0 at a later time horizon (judging by the minimum value).\\

\begin{figure}[!h]
	\centering
	\begin{subfigure}{0.49\linewidth}
		\includegraphics[width=\linewidth]{images/mimic_sepsis/clip_analysis_seed_0/train/obs_per_int_level_by_time.png}
		\caption{Sepsis dataset}
		\label{fig:obs_per_int_level_by_time_seed_0_sepsis}
	\end{subfigure}
	\begin{subfigure}{0.49\linewidth}
		\includegraphics[width=\linewidth]{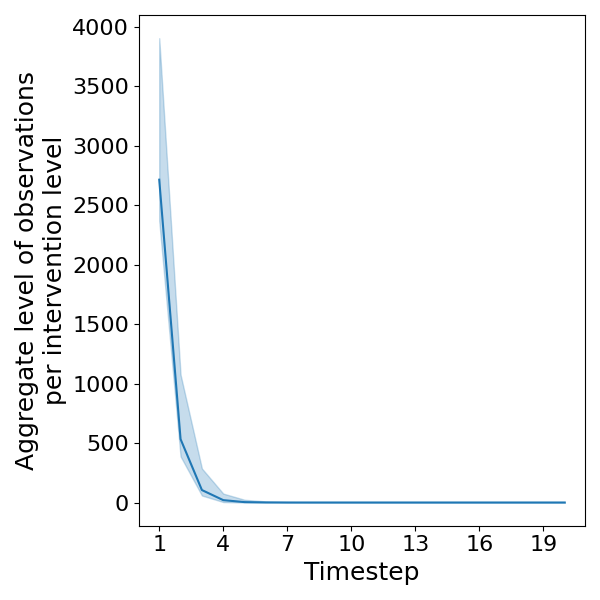}
		\caption{Sepsis-fluid dataset}
		\label{fig:obs_per_int_level_by_time_seed_0_fluid}
	\end{subfigure}
	\caption{The figures describe the mean (blue solid line) and min/max (shaded area) number of observations per intervention level, across time. The seed 0 training set was used to produce the figures.}
	\label{fig:obs_per_int_level_by_time_seed_0}
\end{figure}

Figures \ref{fig:mimic_sepsis_fluid_xgboost_1_post_masd_avg_across_feat} and \ref{fig:mimic_sepsis_fluid_xgboost_1_w_masd_avg_across_feat} display the median $\masdtwmean{}$ and $\masdtdiffmean{}$ values for the sepsis-fluid dataset, respectively. In comparison to the sepsis datasets, the propensity model induces better covariate balance for the first four timesteps (on the testing dataset) instead of the first two (figure \ref{fig:mimic_sepsis_fluid_xgboost_1_w_masd_avg_across_feat}). However, using the restricted action space far from resolves the lack of covariate balance: balance worsens from timestep 4 onwards in comparison to the unweighted covariates.\\

The predictive performance of the propensity score also improved under the sepsis-fluid dataset (table \ref{tbl:propense_results}) which was almost certainly a result of the more uniform action density. Since the propensity model needs to be correctly specified to achieve covariate balance (corollary \ref{corol:weight_cov_bal_as_nec_cond}), this improvement in performance is a confounder for understanding the impact of the reduced action space. Subsequent analysis in section \ref{sec:training_validation_set} demonstrates that sample size does indeed impact covariate balance. As such, all that is unknown is the extent to which the relative balance between the sepsis and the sepsis-fluid datasets is mediated by the improvement in the propensity model performance.\\

\begin{figure}[!h]
	\centering
	\begin{subfigure}{0.49\linewidth}
		\includegraphics[width=\linewidth]{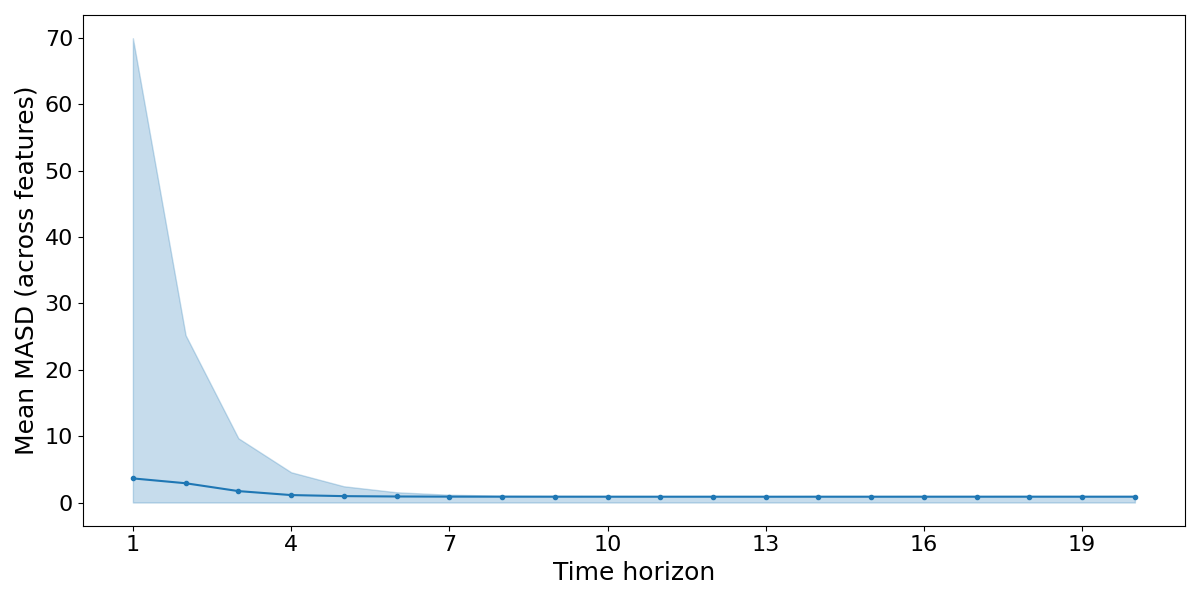}
		\caption{Training dataset}	
	\end{subfigure}
	\begin{subfigure}{0.49\linewidth}
		\includegraphics[width=\linewidth]{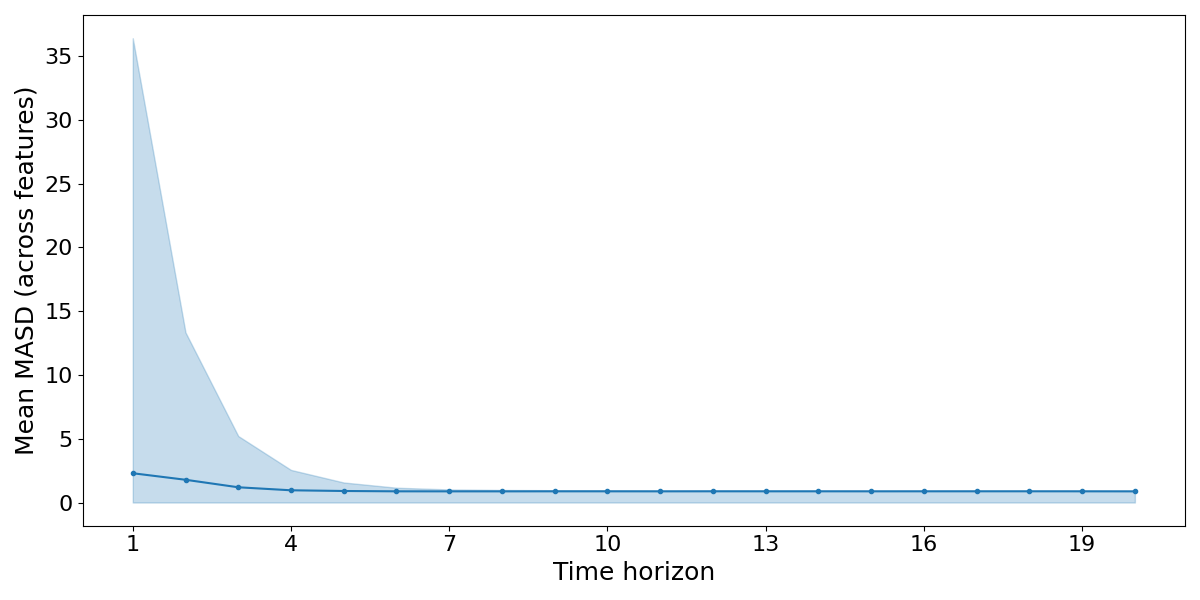}
		\caption{Testing dataset}
	\end{subfigure}
	\caption[The figure describes the median mean MASD value, across features, for the sepsis-fluid dataset, under the vanilla estimator. The x-axis displays the time horizon and the y-axis displays the median mean MASD value as the solid line with the min and max values described by the shaded region.]{The figure describes the median $\masdtwmean{}$ value, across features, for the sepsis-fluid dataset, under the vanilla estimator. The x-axis displays the time horizon and the y-axis displays the median $\masdtwmean{}$ value as the solid line with the min and max values described by the shaded region.}
	\label{fig:mimic_sepsis_fluid_xgboost_1_post_masd_avg_across_feat}
\end{figure}

\begin{figure}[!h]
	\centering
	\begin{subfigure}{0.49\linewidth}
		\includegraphics[width=\linewidth]{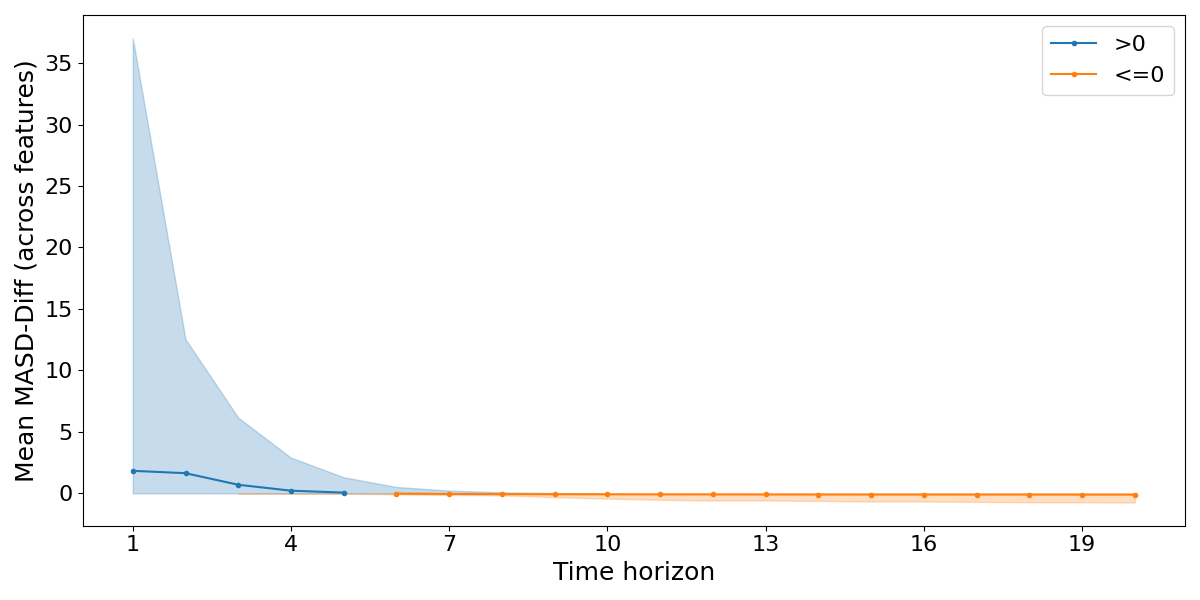}
		\caption{Training dataset}	
	\end{subfigure}
	\begin{subfigure}{0.49\linewidth}
		\includegraphics[width=\linewidth]{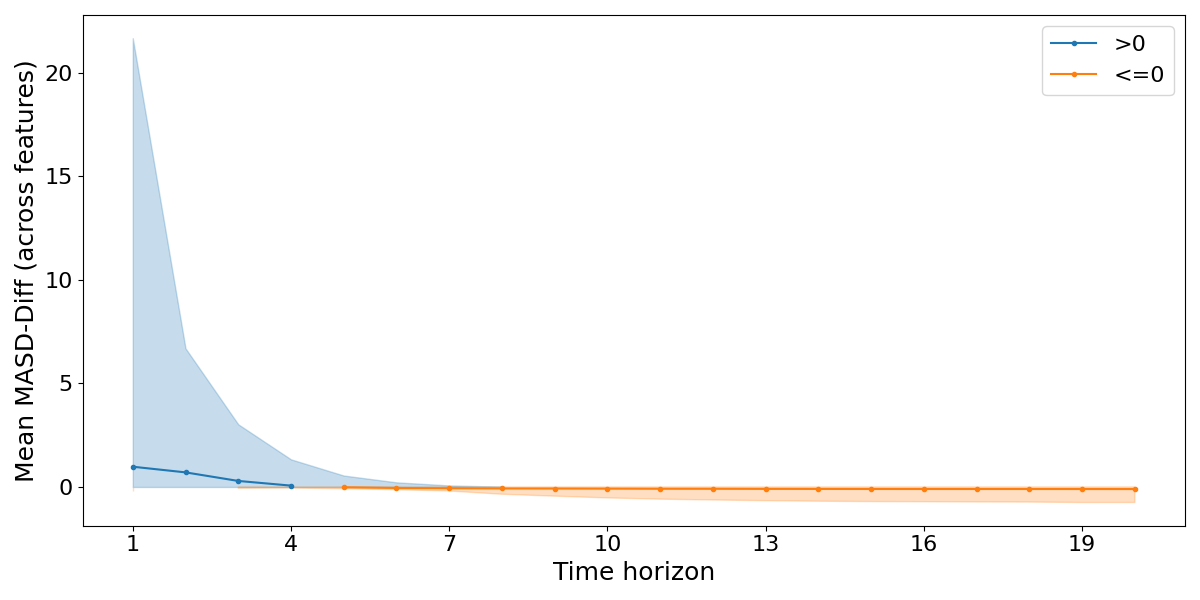}
		\caption{Testing dataset}
	\end{subfigure}
	\caption[The figure describes the median (across features) mean MASD-Diff values for the sepsis-fluid dataset, under the vanilla estimator. The shaded region describes the min and max mean MASD-Diff values by timepoint.]{The figure describes the median (across features) $\masdtdiffmean{}$ values for the sepsis-fluid dataset, under the vanilla estimator. The shaded region describes the min and max $\masdtdiffmean{}$ values by timepoint.}
	\label{fig:mimic_sepsis_fluid_xgboost_1_w_masd_avg_across_feat}
\end{figure}

\begin{figure}[!h]
	\centering
	\begin{subfigure}{0.49\linewidth}
		\includegraphics[width=\linewidth]{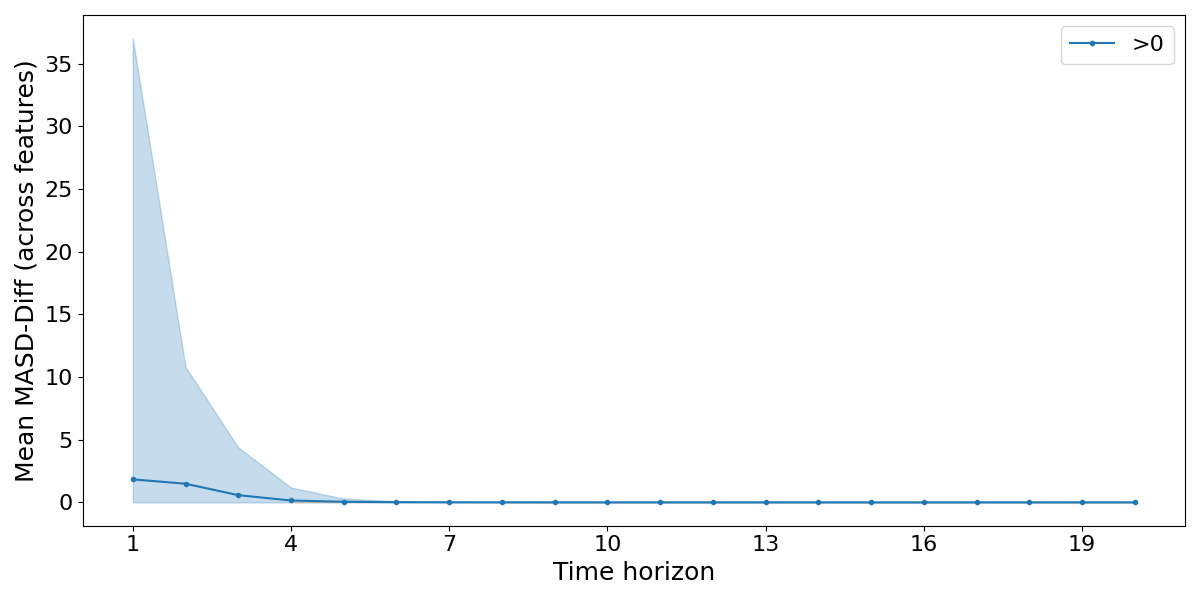}
		\caption{Training dataset}	
	\end{subfigure}
	\begin{subfigure}{0.49\linewidth}
		\includegraphics[width=\linewidth]{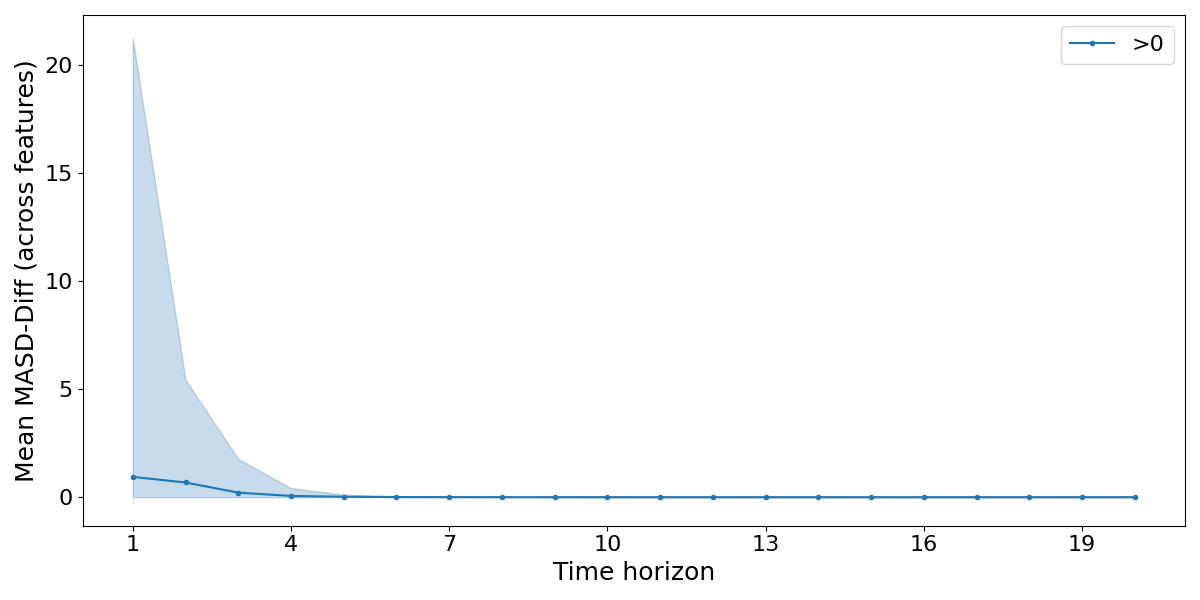}
		\caption{Testing dataset}
	\end{subfigure}
	\caption[The figure describes the median (across features) mean MASD-Diff values for the sepsis-fluid dataset, under the clipped estimator. The shaded region describes the min and max mean MASD-Diff values by timepoint.]{The figure describes the median (across features) $\masdtdiffmean{}$ values for the sepsis-fluid dataset, under the clipped estimator. The shaded region describes the min and max $\masdtdiffmean{}$ values by timepoint.}
	\label{fig:mimic_sepsis_fluid_xgboost_1_w_masd_avg_across_feat_clip}
\end{figure}

\begin{figure}[!h]
	\centering
	\begin{subfigure}{0.49\linewidth}
		\includegraphics[width=\linewidth]{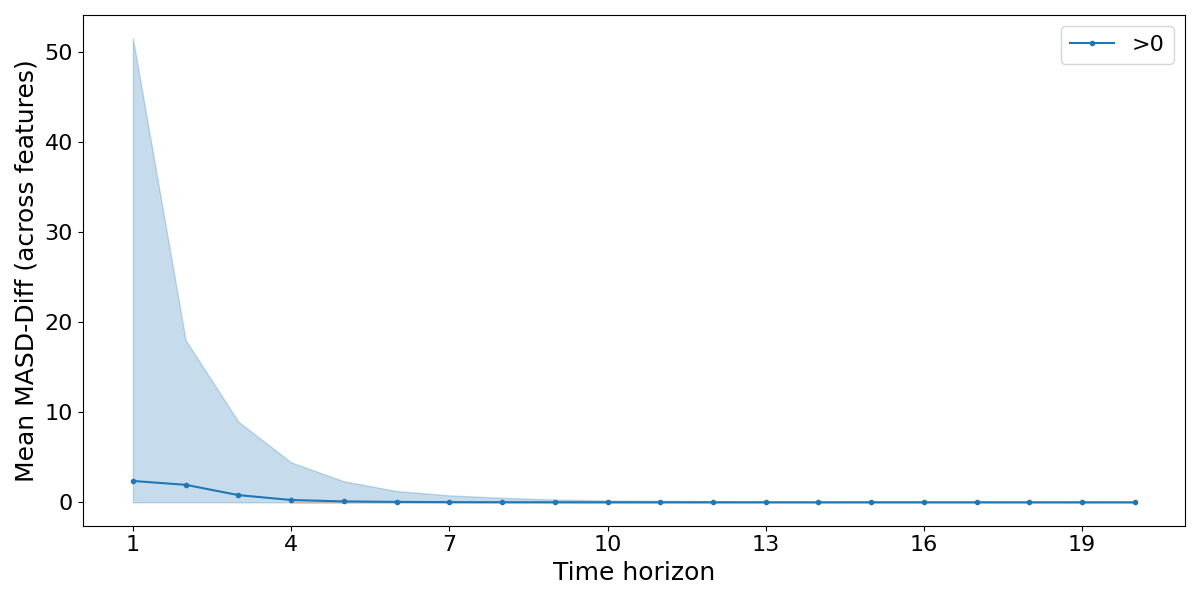}
		\caption{Training dataset}	
	\end{subfigure}
	\begin{subfigure}{0.49\linewidth}
		\includegraphics[width=\linewidth]{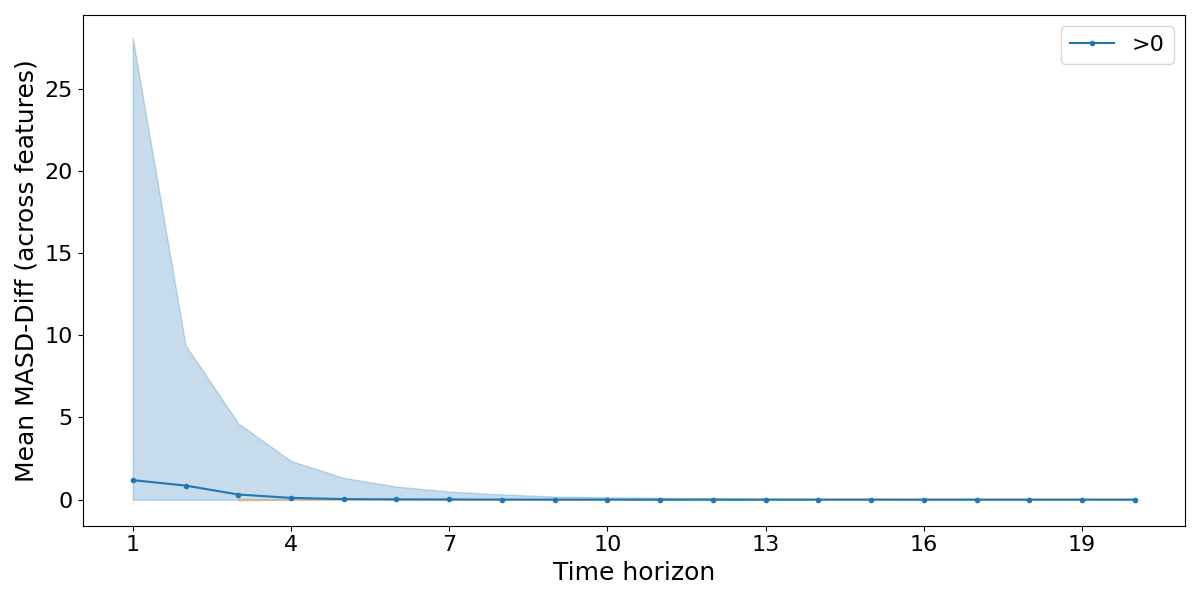}
		\caption{Testing dataset}
	\end{subfigure}
	\caption[The figure describes the median (across features) mean MASD-Diff values for the sepsis-fluid dataset, under the Hajek estimator. The shaded region describes the min and max mean MASD-Diff values by timepoint.]{The figure describes the median (across features) $\masdtdiffmean{}$ values for the sepsis-fluid dataset, under the Hajek estimator. The shaded region describes the min and max $\masdtdiffmean{}$ values by timepoint.}
	\label{fig:mimic_sepsis_fluid_xgboost_1_w_masd_avg_across_feat_hajek}
\end{figure}

\begin{landscape}
\begin{figure}[!h]
	\centering
	\begin{subfigure}{0.2\linewidth}
		\includegraphics[width=\linewidth]{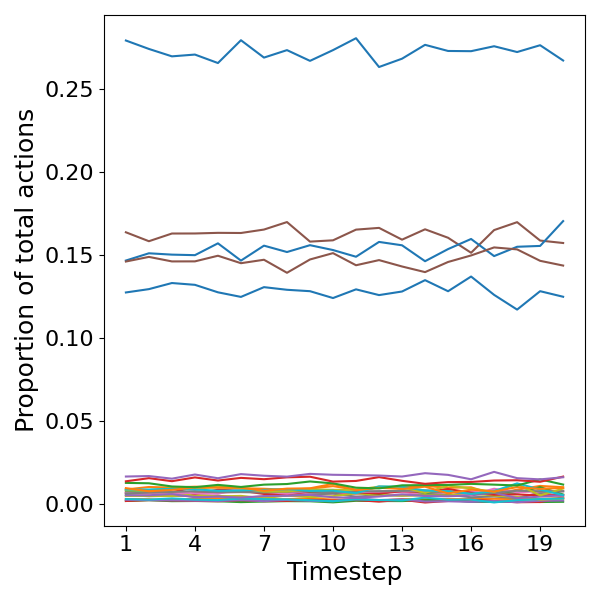}
		\caption{Training dataset (seed 0)}
	\end{subfigure}
	\begin{subfigure}{0.2\linewidth}
		\includegraphics[width=\linewidth]{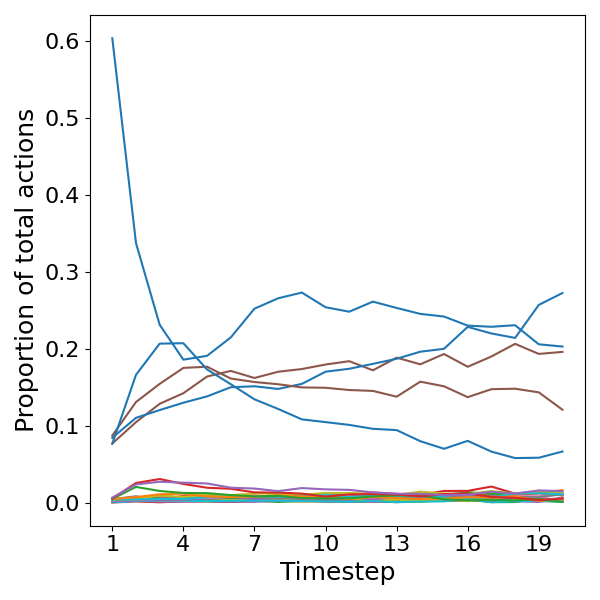}
		\caption{Testing dataset (seed 0)}
	\end{subfigure}
	\begin{subfigure}{0.2\linewidth}
		\includegraphics[width=\linewidth]{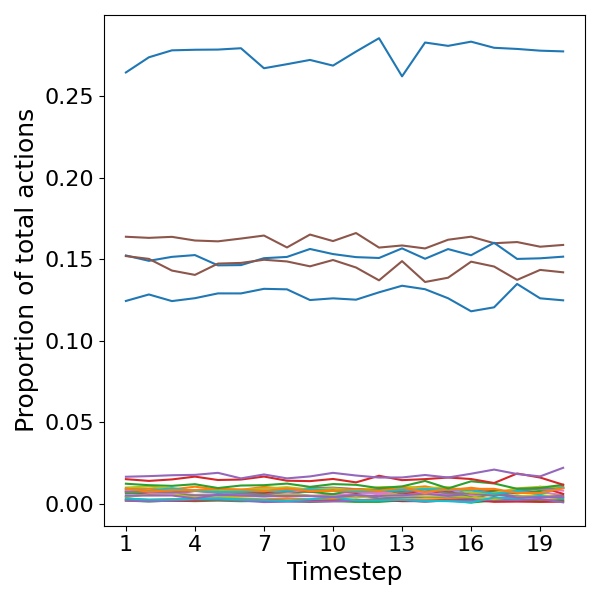}
		\caption{Training dataset (seed 1)}
	\end{subfigure}
	\begin{subfigure}{0.2\linewidth}
		\includegraphics[width=\linewidth]{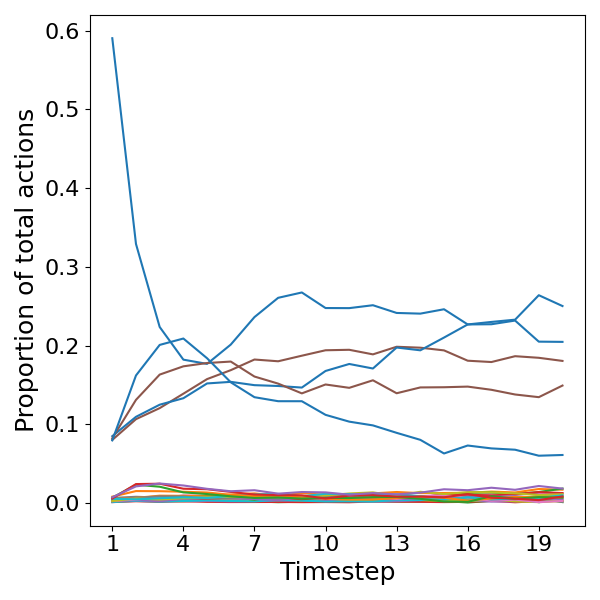}
		\caption{Testing dataset (seed 1)}
	\end{subfigure}
\caption{The figure describes the proportion of actions per timestep, across the training and testing sets, and split by seeds for the sepsis dataset.}
	\label{fig:prop_actions_ovr_time_sepsis}
\end{figure}
\begin{figure}[!h]
	\centering
	\begin{subfigure}{0.2\linewidth}
		\includegraphics[width=\linewidth]{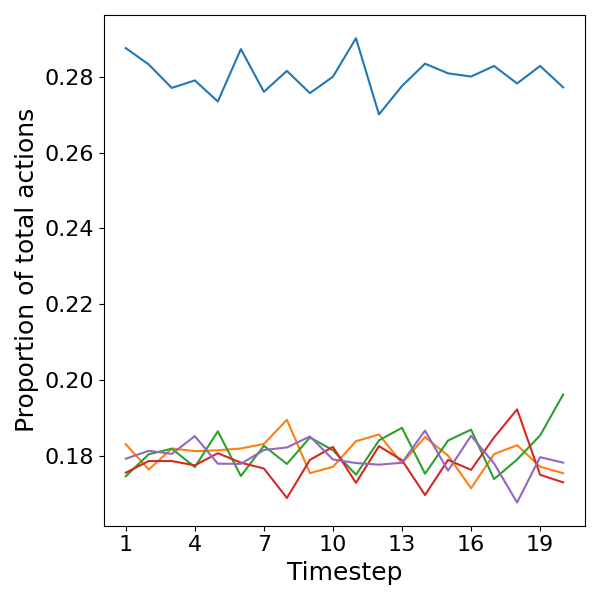}
		\caption{Training dataset (seed 0)}
	\end{subfigure}
	\begin{subfigure}{0.2\linewidth}
		\includegraphics[width=\linewidth]{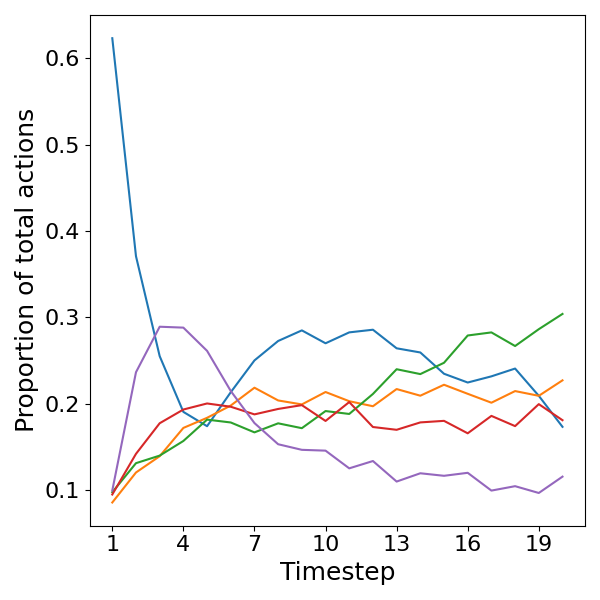}
		\caption{Testing dataset (seed 0)}
	\end{subfigure}
	\begin{subfigure}{0.2\linewidth}
		\includegraphics[width=\linewidth]{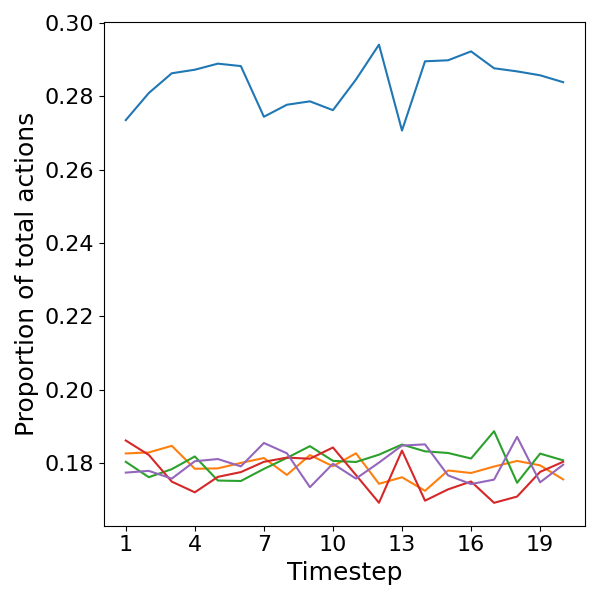}
		\caption{Training dataset (seed 1)}
	\end{subfigure}
	\begin{subfigure}{0.2\linewidth}
		\includegraphics[width=\linewidth]{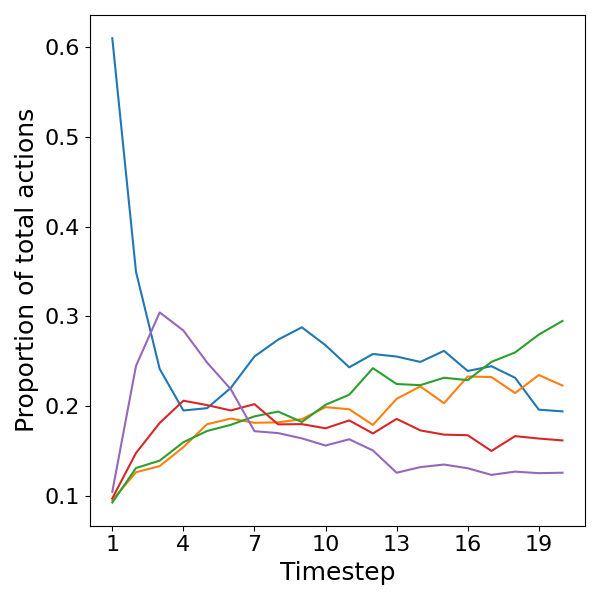}
		\caption{Testing dataset (seed 1)}
	\end{subfigure}
	\caption{The figure describes the proportion of actions per timestep, across the training and testing sets, and split by seeds for the sepsis-fluid dataset.}
	\label{fig:prop_actions_ovr_time_fluid}
\end{figure}
	
\end{landscape}

\FloatBarrier

\subsubsection{Reduced time horizon}\label{sec:reduced_time_horizon}
The results presented in sections \ref{sec:full_action_space} and \ref{sec:fluid_action_space} have suggested that covariate balance might be present over a shorter time horizon. Similar to constraining the action space, as discussed in section \ref{sec:fluid_action_space}, constraining the planning horizon might provide a reasonable approach to ensuring covariate balance is satisfied. This hypothesis was assessed by evaluating covariate balance for propensity models trained on horizons of length 2, 5, 8 and 13 timesteps. Horizons 5 and 13 were chosen since these defined the 5th and 50th percentile of trajectory lengths, calculated on the training set under seed 0. Horizon 2 was chosen as this defines the minimum multistep horizon. Horizon 8 was subsequently included based on the results of 5 and 13. Table \ref{tbl:trunc_traj_results} displays the number of timesteps where the median $\masdtdiffmean{}>0$ and appendix section \ref{sec:reduced_time_horizon_appendix} contains the complete median (across features) $\masdtwmean{}$ and $\masdtdiffmean{}$ plots.\\

%

\begin{table}[!h]
	\centering
	\begin{tabular}{cc|cc}
		\toprule
		Dataset & Horizon & \multicolumn{2}{c}{\# median $\masdtdiffmean{}>0$} \\
		 & & Training & Testing \\
		\midrule
		Sepsis & 2 & 2 & 2 \\
		Sepsis & 5 & 3 & 2 \\
		Sepsis & 8 & 3 & \textbf{3} \\
		Sepsis & 13 & 3 & 2 \\
		\midrule
		Sepsis-fluid & 2 & 2 & 2 \\
		Sepsis-fluid & 5 & 5 & \textbf{5} \\
		Sepsis-fluid & 8 & 5 & 4 \\
		Sepsis-fluid & 13 & 5 & 4 \\
		\bottomrule
	\end{tabular}
	\caption[The table displays the number of timesteps where the median mean MASD-Diff (across features) was greater than 0 for the truncated trajectories across, the sepsis and sepsis-fluid datasets.]{The table displays the number of timesteps where the median $\masdtdiffmean{}$ (across features) was greater than 0 for the truncated trajectories across, the sepsis and sepsis-fluid datasets.}
	\label{tbl:trunc_traj_results}
\end{table}

The results in table \ref{tbl:trunc_traj_results} demonstrate that, across the sepsis and sepsis-fluid datasets, reducing the time horizon \emph{can} improve the resulting covariate balance, albeit this is not uniformly the case. Of the restricted trajectory lengths assessed, covariate balance never worsened and improved on the sepsis dataset at horizon 8 and on the sepsis fluid dataset at horizon 5.\\

For all restricted time horizons and across both action spaces, the action dimension was effectively the same as under the original time horizon. There was a risk that certain actions were only present in later timesteps of the trajectories and thus by considering truncated trajectories, the action space would also effectively be reduced. Since it had already been demonstrated that a reduced action space dimension improved covariate balance (section \ref{sec:fluid_action_space}), reducing the action space along with the horizon risked confounding the analysis. However, as stated the action space dimension was identical except for the first timestep under the sepsis dataset where the effective action dimension was 24 rather than 25. Since the action dimension was effectively retained, any improvement in covariate balance could reasonably be attributed to truncating the trajectories.\\ 

Since covariate balance did not uniformly improve across all truncated horizons (or uniformly across all horizons below a given length), further assessment is required, to ensure these results were not statistical anomalies. It is perhaps counterintuitive that restricting the maximum time horizon would increase the horizon over which covariate balance is present. It is hypothesised however, that truncating the horizon improves the stability of training the propensity model. This is discussed further in section \ref{sec:prop_score_arch}.\\

\FloatBarrier

\section{Conclusion}

The core implication of the analysis is that existing applications of offline reinforcement learning for learning treatment recommendations are too high dimensional (both in time and action space) to conclude, using existing approaches for assessing diagnostic balance that, at least heuristically, conditional exchangeability holds and/or the propensity model is correctly specified. With respect to the sepsis domain assessed, figures \ref{fig:xgboost_4_w_masd_avg_across_feat} and \ref{fig:mimic_sepsis_fluid_xgboost_1_w_masd_avg_across_feat} combined suggest that covariate balance can be achieved for the first 2 timesteps and this can be improved by restricting the action space and potentially by truncating trajectories.\\

A second implication of the analysis is that, traditional approaches to reducing variance in importance sampling estimators are not valid for assessing covariate balance. As the variance of the vanilla estimator (and thus uncertainty) increases, both the Hajek and clipped estimator will predictably suggest that covariate balance has been achieved. Arguably, it is unsurprising that the clipped and Hajek estimators are insufficient for assessing covariate balance. Where sequential exchangeability does not hold/the propensity model is miss-specified, the resulting causal estimates are biased. As a result, the MASD metrics are a diagnostic check for bias. On the surface, checking for bias using biased estimators (without proper control) feels counterintuitive. An impactful line of future work would consider the point at which the bias under the clipped or Hajek estimators overwhelms the estimates. If this could be characterised, then far greater confidence could be placed in the presence of covariate balance at timesteps larger than 2, for the sepsis dataset assessed in this analysis (based on figures \ref{fig:xgboost_4_w_masd_avg_across_feat_clip} and \ref{fig:xgboost_4_w_masd_avg_across_feat_hajek}).\\

\subsection{Training/validation set}\label{sec:training_validation_set}

As discussed in section \ref{sec:experimental_protocol}, it is intuitive to consider the covariate balance over the training set to be a biased estimate of the expected covariate balance. On the surface, the general degradation in balance between the testing and training set supports this. Figure \ref{fig:xgboost_4_w_masd_avg_across_feat_subsamp} describes the median (across features) $\masdtdiffmean{}$ values, calculated on seed 0 of the training data but subsampled to be the size of the testing set, for the sepsis and sepsis-fluid datasets, respectively. Notably, in comparison to figure \ref{fig:xgboost_4_w_masd_avg_across_feat} for the sepsis dataset and \ref{fig:mimic_sepsis_fluid_xgboost_1_w_masd_avg_across_feat} for the sepsis-fluid datasets, the number of timesteps with a positive median covariate balance is equivalent to that observed on the testing sets. It is thus quite possible that the degradation in balance on the testing set is a function of the reduced sample size, rather than the training set metric being inherently biased/the bias being substantial.\\

\begin{figure}[!h]
	\centering
	\begin{subfigure}{0.49\linewidth}
			\includegraphics[width=\linewidth]{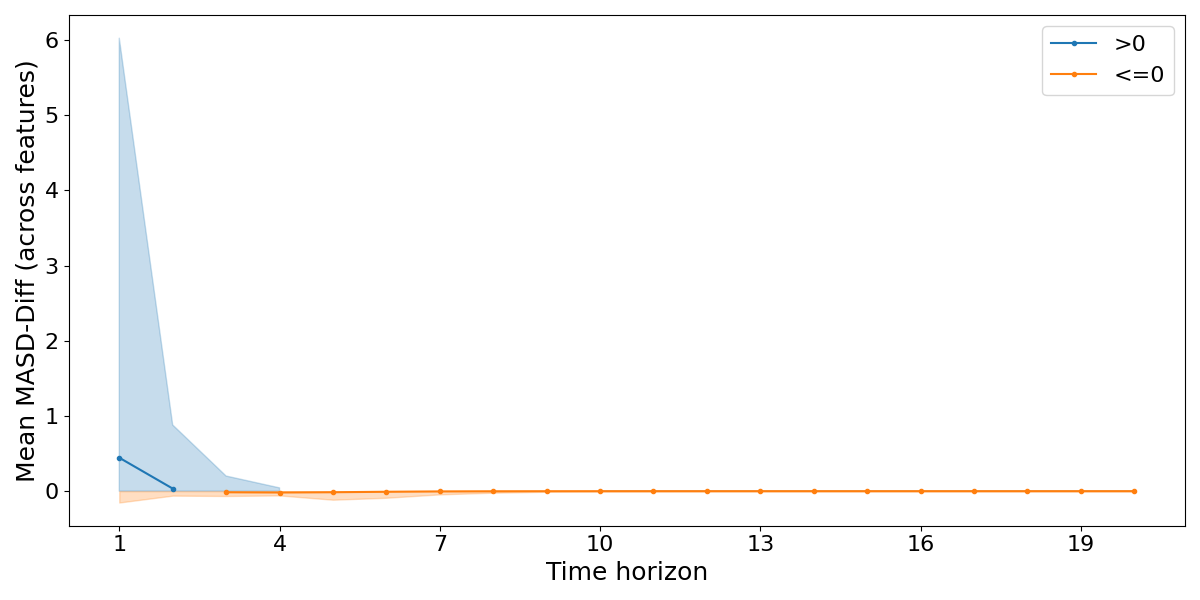}
			\caption{Sepsis dataset}	
	\end{subfigure}
	\begin{subfigure}{0.49\linewidth}
			\includegraphics[width=\linewidth]{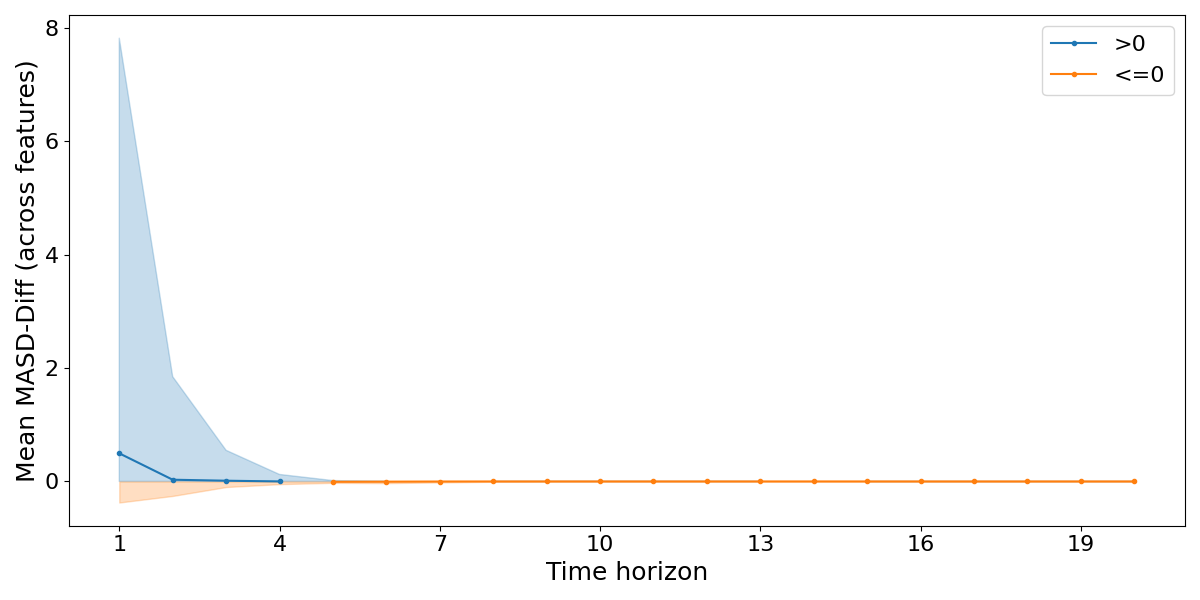}
			\caption{Sepsis-fluid dataset}	
	\end{subfigure}
	\caption[The figure describes the median (across features) mean MASD-Diff values for seed 0 of the training data, subsampled to the size of the testing set. The metric values are derived under the vanilla importance sampling estimator. The shaded region describes the min and max mean MASD-Diff values by timepoint.]{The figure describes the median (across features) $\masdtdiffmean{}$ values for seed 0 of the training data, subsampled to the size of the testing set. The metric values are derived under the vanilla importance sampling estimator. The shaded region describes the min and max $\masdtdiffmean{}$ values by timepoint.}
	\label{fig:xgboost_4_w_masd_avg_across_feat_subsamp}
\end{figure}

Without considering the use of covariate balance for model training (in the case of \cite{Imai2014}), it is quite possible that the bias (if any) from assessing covariate balance on the training set (or even better the joint training and testing set) is sufficiently small such that the increased sample size provides a more accurate view of the covariate balance achieved. Given this observation, an important line of future work would be to establish the behaviour of covariate balance to determine whether the training set can be used to assess balance.\\

\FloatBarrier

\subsection{Propensity score architecture}\label{sec:prop_score_arch}
The derivation of the weighted MASD metrics described in section \ref{sec:covariate_balance} utilised the independence structure of the MDP to simplify the time $t$ propensity estimate by equating $p(A_{0:t}|S_{0:t})=\prod_{t'=0}^{t}p(A_{t'}|S_{t'})$. This autoregressive approach is common for sequence modelling. An alternate modelling approach would be to jointly learn the density of the trajectory. The results presented in \ref{sec:reduced_time_horizon} potentially support the hypothesis that this joint modelling approach would benefit the development of propensity models. The results showed that by truncating the trajectory length, the number of timesteps where covariate balance was present (relative to the unweighted distribution) improved, in some scenarios. It is hypothesised that this was a result of the global structure of the trajectory being easier to capture.\\

\subsection{Additional future directions}
It was alluded to in section \ref{sec:background} and \ref{sec:cov_bal_implications} that: an alternate measure of model miss-specification could be considered through stabilised weights and; if the propensity model is estimated non-parametrically, then covariate balance could be viewed as a measure of hidden confounding. Neither of these claims were investigated in the analysis but might provide fruitful avenues of future research. In particular, hidden confounding is a primary source of bias within causal inference studies and thus any estimator directly measuring this assumption would be hugely impactful.\\


\acks{Joshua Spear is funded by Great Ormond Street Hospital Charity.}


\appendix

\section{Covariate balance}\label{sec:cov_bal_appendix}

\begin{theorem}[Theorem 3 of \cite{rosenbaum_central_1983}]\label{ther:rosenbaum_central} If treatment assignment is strongly ignorable, then it is strongly ignorable given any balancing score $b(x)$; that is $\forall x$,
\begin{align*}
	(R_{1},R_{0}) \indep A|S \textrm{ and } 0<P(A=1|S)<1
\end{align*}
implies, $\forall b(x)$:
\begin{align*}
	(R_{1},R_{0}) \indep A|b(S) \textrm{ and } 0<P(A=1|b(S))<1
\end{align*}
\end{theorem}

\begin{corollary}[Corollary \ref{corol:cov_bal_as_nec_cond} restated] If the conditions of theorem \ref{ther:rosenbaum_central} hold, then $\forall b(s)$:
	\begin{align*}
		(R_{1},R_{0}) \indep A|b(S) \textrm{ and } 0<P(A=1|b(S))<1 \implies A \indep S | b(S)
	\end{align*}
\end{corollary}
\begin{proof} Begin with Theorem \ref{ther:rosenbaum_central}, then for $R^{\pearldo{}(A)}\indep A|b(S)$, $b(S)$ must be a balancing score (by assumption of the theorem). 
\end{proof}

\begin{corollary}[Corollary \ref{corol:weight_cov_bal_as_nec_cond} restated] If the conditions of theorem \ref{ther:rosenbaum_central} hold, and:
\begin{equation}
	\mathbb{E}\Bigg[\frac{\mathbbm{1}(A=1)S}{\hat{e}(S)}-\frac{\mathbbm{1}(A=0)S}{1-\hat{e}(S)}\Bigg] = 0
\end{equation}
then, $\forall \hat{e}(s)$:
\begin{align*}
	(R_{1},R_{0}) \indep A|\hat{e}(S) \textrm{ and } 0<P(A=1|\hat{e}(S))<1
\end{align*}
\end{corollary} 
\begin{proof} Firstly, by theorem \ref{ther:rosenbaum_central}, $\forall s, e(s)$:
\begin{align*}
	(R_{1},R_{0}) \indep A|S & \textrm{ and } 0<P(A=1|S)<1\\ 
	&\implies (R_{1},R_{0}) \indep A|e(S) \textrm{ and } 0<P(A=1|e(S))<1
\end{align*}
Secondly, observe that, for the true propensity score, $e(S)$:
\begin{align}
	\mathbb{E}\Bigg[&\frac{\mathbbm{1}(A=1)S}{e(S)}-\frac{\mathbbm{1}(A=0)S}{1-e(S)}\Bigg] = \mathbb{E}\Bigg[\frac{\mathbbm{1}(A=1)S}{p(A|S)}-\frac{\mathbbm{1}(A=0)S}{p(A|S)}\Bigg] \\
	=&\int\frac{\mathbbm{1}(A=1)S}{p(A|S)}p(S,A)dsa -\int\frac{\mathbbm{1}(A=0)S}{p(A|S)}p(S,A)dsa \nonumber \\
	=&\int\frac{S}{p(A=1|S)}p(A=1|S)p(S)dsa -\int\frac{S}{p(A=0|S)}p(A=0|S)p(S)dsa \nonumber \\
	=&\mathbb{E}[S] - \mathbb{E}[S] = 0 \nonumber 
\end{align}
As such:
\begin{align*}
	\hat{e}(S) = e(S) \iff \mathbb{E}\Bigg[&\frac{\mathbbm{1}(A=1)S}{\hat{e}(S)}-\frac{\mathbbm{1}(A=0)S}{1-\hat{e}(S)}\Bigg] = 0
\end{align*}
Combining the two statements, if theorem \ref{ther:rosenbaum_central} and equation \ref{equ:weighted_cov_bal} hold, $\implies \forall \hat{e}(s)$:
\begin{align*}
	(R_{1},R_{0}) \indep A|\hat{e}(S) \textrm{ and } 0<P(A=1|\hat{e}(S))<1
\end{align*}
\end{proof}

\section{Stabilised weights}\label{sec:stabilised_weights}
Within the causal inference literature, the use of stabilised weights of the form $\frac{g(A)}{p(A|L)}$ (in the time independent setting) are popular (\cite{hernan_causal_2023}). In particular, the function $g$ is often defined to be the marginal density of treatment, $p(A)$. The proceeding demonstrates the balancing quality of these weights under non-dynamic policies (section \ref{sec:stabilised_non_dynamic}) and the asymptotic bias of stabilised weights under dynamic policies (section \ref{sec:stabilised_dynamic}).\\

\subsection{Balancing quality of stabilised weights for non-dynamic policies}\label{sec:stabilised_non_dynamic}
When using stabilised weights, the validity of the propensity models used for the analysis is asserted through the observation that:
\begin{align}
	\mathbb{E}_{S,A}\Bigg[\frac{p(A)}{p(A|S)}\Bigg] =& \int_{\mathcal{A}\times\mathcal{S}} \frac{p(A)}{p(A|S)} dP(A,S) \label{equ:stablised_time_indep_balance_start}\\
	=& \int_{\mathcal{A}\times\mathcal{S}} \frac{p(A)}{p(A|S)} p(A|S)p(S)d\mu(A,S) \\
	=& \int_{\mathcal{A}}p(A)\Bigg(\int_{\mathcal{S}}p(S)\mu(S)\Bigg)\mu(A)\\
	=& \int_{\mathcal{A}}p(A)\mu(A) = 1 \label{equ:stablised_time_indep_balance_end}
\end{align}
where $\mathbb{E}_{L,A}[p(A)p(A|L)^{-1}]\approx \frac{1}{n}\sum_{i=1}^{n}p(A)p(A|L)^{-1}$ and the second to last line follows from Fubini's theorem as $\mu(A,L)$ is the product measure over Lebesgue measures, $\mu(A)$ and $\mu(L)$. Assuming $S_{t}\rightarrow A_{t}$ and arbitrary time dependent confounding, as demonstrated for two timesteps in figure \ref{fig:arb_confound}, the same balancing quality is demonstrated for the time dependent setting in equations \ref{equ:stablised_time_dep_balance_start} to \ref{equ:stablised_time_dep_balance_end}

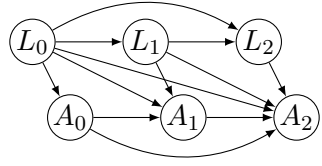
\begin{figure}[!h]

    \centering
    \begin{tikzpicture}[
        roundnode/.style={circle, draw=black,  minimum size=6mm, inner sep=0},
        shaderoundnode/.style={circle, draw=black,  minimum size=6mm, inner sep=0, fill=black!20},
        squarenode/.style={square, draw=black,  minimum size=6mm, inner sep=0},
    ]
		\node[roundnode, align=center] (l0) at (0,0) {$L_{0}$};
		\node[roundnode, align=center] (l1) at (1.5,0) {$L_{1}$};
		\node[roundnode, align=center] (l2) at (3,0) {$L_{2}$};
		
		\node[roundnode, align=center] (a0) at (0.5,-1) {$A_{0}$};
		\node[roundnode, align=center] (a1) at (2,-1) {$A_{1}$};
		\node[roundnode, align=center] (a2) at (3.5,-1) {$A_{2}$};
        

		\draw[-latex] (l0) -- (l1);
		\draw[-latex] (l0) to[bend left=30] (l2);
		\draw[-latex] (l0) -- (a0);
		\draw[-latex] (l0) -- (a1);
		\draw[-latex] (l0) -- (a2);
		\draw[-latex] (l1) -- (l2);
		\draw[-latex] (l1) -- (a1);
		\draw[-latex] (l1) -- (a2);
		\draw[-latex] (l2) -- (a2);
		
		\draw[-latex] (a0) -- (a1);
		\draw[-latex] (a0) to[bend right=30] (a2);
		\draw[-latex] (a1) -- (a2);
        
    \end{tikzpicture}
    \caption{DAG with arbitrary confounding }
    \label{fig:arb_confound}
\end{figure}

\begin{align}
	&\mathbb{E}_{L,A}\Bigg[\prod_{t=0}^{T-1}\frac{p(A_{t}|\bar{A}_{t-1})}{p(A_{t}|\bar{A}_{t-1},\bar{L}_{t})}\Bigg] = \int_{\bar{\mathcal{A}}\times\bar{\mathcal{L}}} \prod_{t=0}^{T-1}\frac{p(A_{t}|\bar{A}_{t-1})}{p(A_{t}|\bar{A}_{t-1},\bar{L}_{t})} dP(\bar{A}_{T-1},\bar{L}_{T-1})\label{equ:stablised_time_dep_balance_start}\\
	&=\int_{\bar{\mathcal{A}}\times\bar{\mathcal{L}}} \prod_{t=0}^{T-1}p(A_{t}|\bar{A}_{t-1})p(L_{0})\prod_{t=1}^{T-1}p(L_{t}|\bar{A}_{t-1},\bar{L}_{t-1})d\mu(\bar{A}_{T-1},\bar{L}_{T-1}) \\
	&=\int_{\bar{\mathcal{A}}\times\bar{\mathcal{L}}} p(L_{0})p(A_{0})\prod_{t=1}^{T-1}p(A_{t}|\bar{A}_{t-1})p(L_{t}|\bar{A}_{t-1},\bar{L}_{t-1})d\mu(\bar{A}_{T-1},\bar{L}_{T-1})\\
	&=\int_{\mathcal{A}_{0}\times\mathcal{L}_{0}} p(L_{0})p(A_{0})\Bigg(\prod_{t=1}^{T-1}\int_{\mathcal{A}_{t}\times\mathcal{L}_{t}|\bar{\mathcal{A}}_{t-1},\bar{\mathcal{L}}_{t-1}}p(A_{t}|\bar{A}_{t-1})p(L_{t}|\bar{A}_{t-1},\bar{L}_{t-1})d\mu(A_{t},L_{t})\Bigg)d\mu(A_{0},L_{0})\label{equ:stablised_time_dep_balance_end} \\
	&=1
\end{align}
noting that:
\begin{align}
	&\int_{\mathcal{A}_{t}\times\mathcal{L}_{t}|\bar{\mathcal{A}}_{t-1},\bar{\mathcal{L}}_{t-1}}p(A_{t}|\bar{A}_{t-1})p(L_{t}|\bar{A}_{t-1},\bar{L}_{t-1})d\mu(A_{t},L_{t})\\
	&=\int_{\mathcal{A}_{t}|\bar{\mathcal{A}}_{t-1},\bar{\mathcal{L}}_{t-1}}p(A_{t}|\bar{A}_{t-1})\Bigg(\int_{\mathcal{L}_{t}|\bar{\mathcal{A}}_{t-1},\bar{\mathcal{L}}_{t-1}}p(L_{t}|\bar{A}_{t-1},\bar{L}_{t-1})d\mu(L_{t})\Bigg)d\mu(A_{t})\label{equ:int_fubini}
\end{align}
For equations \ref{equ:stablised_time_dep_balance_end} and \ref{equ:int_fubini}, Fubini's theorem has been applied where, $\mathcal{A}_{t}|\bar{\mathcal{A}}_{t-1},\bar{\mathcal{L}}_{t-1}$ is the sub-$\sigma$-algebra induced by the natural filtration on $\bar{\mathcal{A}}_{t-1},\bar{\mathcal{L}}_{t-1}$. However, $\mathcal{A}_{t}|\bar{\mathcal{A}}_{t-1},\bar{\mathcal{L}}_{t-1} = \mathcal{A}_{0} = \mathcal{A}$ is always assumed.\\

\subsection{Asymptotic bias of stabilised weights under dynamic policies}\label{sec:stabilised_dynamic}
Stabilised weights have not been considered in the analysis presented since, when using dynamic policies, similar estimators are not yet proven to be unbiased. This is most intuitively seen in the time dependent setting. Using the notation $\bar{B}_{t-1} = \{B_{0}, ..., B_{t-2}\}$, and assuming time $t$ confounders, $L_{t}$ always cause the time $t$ actions, $A_{t}$ but allowing for arbitrary time dependent confounding, as shown in figure \ref{fig:arb_confound} (for $t=2$), consider the proceeding estimator for the dynamic policy, $\pi_{e}: \mathcal{S} \rightarrow \mathcal{A}$:
\begin{equation}\label{equ:time_dep_stab_weight_dynamic}
    \frac{1}{n}\sum_{i=1}^{n}y_{i}\mathbbm{1}\big(\bar{a}_{i,T-1}=\pi_{e}(\bar{s}_{i,T-1})\big)\prod_{t=0}^{T-1}\frac{p(a_{i,t}|\bar{a}_{i,t-1})}{p(a_{i,t}|\bar{a}_{i,t-1},\bar{s}_{i,t})}
\end{equation}
Equation \ref{equ:time_dep_stab_weight_dynamic} defines the standard time dependent stablised weight estimator (\cite{hernan_causal_2023}) but conditioning on the set of actions defined by the deterministic policy, $\pi_{e}$. Naively, this is intended to be an estimator for $\mathbb{E}[Y|\pearldo{}(\bar{A}=\pi_{e}(\bar{S}))]$. Examining the bias however:
\begin{align*}
    &\mathbb{E}\Bigg[\frac{1}{n}\sum_{i=1}^{n}y_{i}\mathbbm{1}\big(\bar{a}_{i,T-1}=\pi_{e}(\bar{s}_{i,T-1})\big)\prod_{t=0}^{T-1}\frac{p(a_{i,t}|\bar{a}_{i,t-1})}{p(a_{i,t}|\bar{a}_{i,t-1},\bar{s}_{i,t})}\Bigg] \\
    &= \mathbb{E}\Bigg[Y\mathbbm{1}\big(\bar{A}_{T-1}=\pi_{e}(\bar{S}_{T-1})\big)\prod_{t=0}^{T-1}\frac{p(A_{t}|\bar{A}_{t-1})}{p(A_{t}|\bar{A}_{t-1},\bar{S}_{t})}\Bigg]\\
    &=\int_{\mathcal{Y}\times\bar{\mathcal{A}}\times\bar{\mathcal{S}}} Y\mathbbm{1}\big(\bar{A}_{T-1}=\pi_{e}(\bar{S}_{T-1})\big)\prod_{t=0}^{T-1}\frac{p(A_{t}|\bar{A}_{t-1})}{p(A_{t}|\bar{A}_{t-1},\bar{S}_{t})} dP(Y,\bar{A},\bar{S})\\
    &=\int_{\mathcal{Y}\times\bar{\mathcal{A}}\times\bar{\mathcal{S}}} Y\mathbbm{1}\big(\bar{A}_{T-1}=\pi_{e}(\bar{S}_{T-1})\big)\Bigg(\prod_{t=0}^{T-1}p(A_{t}|\bar{A}_{t-1})\Bigg)\Bigg(p(S_{0})\prod_{t=1}^{T-1}p(S_{t+1}|\bar{S}_{t},\bar{A}_{t})\Bigg)d\mu(Y,\bar{A},\bar{S})\\
    &=\int_{\mathcal{Y}\times\bar{\mathcal{S}}} Y\Bigg(\prod_{t=0}^{T-1}p\big(\pi_{e}(S_{t})|\pi_{e}(\bar{S}_{t-1})\big)\Bigg)\Bigg(p(S_{0})\prod_{t=1}^{T-1}p\big(S_{t+1}|\bar{S}_{t},\pi_{e}(\bar{S}_{t})\big)\Bigg)d\mu(Y,\bar{S})\\
    &\neq\mathbb{E}[Y|\pi_{e}(\bar{S})]
\end{align*}
\sloppy 
Since the estimate is biased by the term, $\prod_{t=0}^{T-1}p(\pi_{e}(S_{t})|\pi_{e}(\bar{S}_{t-1}))$. For non-dynamic policies, this bias is side-stepped by introducing the multiplicative control variate $\mathbb{E}[\mathbbm{1}[\bar{A}=\bar{a}]\prod_{t=0}^{T-1}p(A_{t}|\bar{A}_{t-1})p(A_{t}|\bar{A}_{t-1},\bar{S}_{t})^{-1}]$ and observing that this equals $1\cdot\prod_{t=0}^{T-1}p(A_{t}|\bar{A}_{t-1})$. However, for the dynamic case with $\bar{A} = \pi_{e}(\bar{S})$, the term $\prod_{t=0}^{T-1}p(\pi_{e}(S_{t})|\pi_{e}(\bar{S}_{t-1}))$ still depends on the measure of integration and thus cannot be considered constant and extracted from the expectation. Thus, introducing a similar control variate would not un-bias the estimate. Note, in the derivation, the $\bar{A}$ can be dropped from the domain of integration because the actions are fixed to $\pi_{e}$. Additionally, the measure is restricted to $\mu(Y,\bar{S})$ (i.e., the product measure) since the sub-$\sigma$-algebra generated by $\bar{A}=\pi_{e}(\bar{S})$ is still the product space $\mathcal{Y}\times\bar{S}$. Furthermore, the derivation uses the identities:
\begin{align*}
	p(A_{t},\bar{A}_{t-1},\bar{L}_{t}) =& p(L_{0})p(A_{0}|L_{0})\prod_{t'=1}^{t}p(A_{t'}|\bar{A}_{t'-1},\bar{L}_{t'})p(L_{t'}|\bar{A}_{t'-1},\bar{L}_{t'-1})\\
	p(A_{0}|\bar{A}_{-1},\bar{L}_{0}) =& p(A_{0}|L_{0}) \\
	p(A_{0}|\bar{A}_{-1}) =& p(A_{0})
\end{align*}

\vskip 0.2in
\bibliography{bib}

\end{document}